%% file: main.tex
\documentclass{article}
\usepackage{preprint,times}

\usepackage{amsmath,amssymb,amsthm,mathtools}
\usepackage{graphicx}
\usepackage{float}
\usepackage{booktabs}
\usepackage{xcolor}
\usepackage{tikz}
\usetikzlibrary{arrows,arrows.meta,positioning,calc}
\usetikzlibrary{cd}      %
\usepackage{subcaption}  %
\usepackage{hyperref}
\usepackage{url}

\graphicspath{{figs/}}

\theoremstyle{plain}
\newtheorem{theorem}{Theorem}[section]
\newtheorem{proposition}[theorem]{Proposition}
\newtheorem{corollary}[theorem]{Corollary}
\newtheorem{lemma}[theorem]{Lemma}
\theoremstyle{definition}
\newtheorem{definition}[theorem]{Definition}
\newtheorem{remark}[theorem]{Remark}
\newtheorem{example}[theorem]{Example}

\tikzset{
  zzval/.style  = {very thick, blue!60!black},
  zzbar/.style  = {line width=2.4pt, blue!60!black, line cap=round},
  zzdrop/.style = {densely dotted, gray!70},
  zzthr/.style  = {dashed, gray!80},
  zzmark/.style = {circle, draw=red!75!black, fill=white, line width=.7pt, inner sep=1.3pt},
  zztxt/.style  = {font=\scriptsize},
  zzred/.style  = {font=\scriptsize, red!75!black},
}

\definecolor{ForestGreen}{rgb}{0.13, 0.55, 0.13}
\definecolor{airforceblue}{rgb}{0.36, 0.54, 0.66}
\definecolor{orange}{rgb}{1.0, 0.5, 0.0}
\definecolor{amethyst}{rgb}{0.6, 0.4, 0.8}
\definecolor{awesome}{rgb}{1.0, 0.13, 0.32}
\definecolor{chromeyellow}{rgb}{1.0, 0.65, 0.0}

\title{A differentiability framework for zigzag persistent homology via linear interpolation}

\author{Enrico Maria Ferrari$^{1,2}$, Clemens Bannwart$^{1}$, Matteo Biagetti$^{1}$\\
$^{1}$Area Science Park, Trieste, Italy\\
$^{2}$Politecnico di Torino, Torino, Italy}

\begin{document}
\maketitle

\begin{abstract}
Persistent homology can be differentiated and incorporated into learning pipelines, but no analogous framework exists for zigzag persistence, which is needed when the underlying topological structure evolves non-monotonically over time.
We develop such a framework for sequences of simplicial complexes obtained by thresholding time-dependent filtering values on a fixed complex.
By assigning persistence diagram endpoints the real-valued times at which linearly interpolated filtering values cross the threshold, we transfer the continuity of the filtering values to the diagram points.
This yields smooth local lifts of the resulting persistence-diagram-valued map, from which we derive differentials almost everywhere under mild regularity conditions on the parametrization of the filtering values.
We prove local Lipschitz continuity outside an explicit measure-zero exclusion set; standard stochastic subgradient convergence guarantees therefore do not apply directly.
We argue that, even without such guarantees, this exclusion set is small enough in practice to allow effective optimization. 
We test this empirically in two experiments: sensor network coverage optimization and dynamic graph classification.

\end{abstract}

\input{sections/intro}

\section{Background}
\label{sec:bg}

We review essential background in this section. Section~\ref{sec:zigzag} introduces zigzag persistent homology, while Section~\ref{sec:diff} recalls the differentiability framework for maps to and from the space of persistence diagrams. In this paper, we combine both tools by studying differentiability for persistence diagrams generated from zigzag filtrations.

\subsection{Zigzag persistent homology}
\label{sec:zigzag}

We summarize the basic concepts of zigzag persistent homology on simplicial complexes \cite{carlsson2010zigzag}.

\paragraph{Zigzag filtration.} A finite \emph{simplicial complex} $K$ over a vertex set $V$ is a collection of non-empty subsets, called \emph{simplices}, closed under taking subsets. A \emph{zigzag filtration} $F$ is a finite sequence of simplicial complexes connected by bidirectional inclusions:
$$F: K_0 \leftrightarrow K_1 \leftrightarrow \dots \leftrightarrow K_n,$$
where each $K_j \leftrightarrow K_{j+1}$ represents either a forward inclusion $K_j \hookrightarrow K_{j+1}$ or a backward inclusion $K_j \hookleftarrow K_{j+1}$. The filtration is \emph{simplex-wise} if adjacent complexes differ by exactly one simplex.

\paragraph{Zigzag persistent homology.} Fix a field $\mathbb{F}$ (e.g., $\mathbb{Z}_2$) and $p \in \mathbb{N}$. The $p$-th simplicial homology functor $H_p(\cdot)$ maps simplicial complexes to vector spaces and inclusions to linear maps. Applying $H_p(\cdot)$ to $F$ yields a sequence of vector spaces and linear maps called a \emph{zigzag persistence module}:
$$H_p(F): H_p(K_0) \leftrightarrow H_p(K_1) \leftrightarrow \dots \leftrightarrow H_p(K_n).$$
By Gabriel's theorem for $A_{n+1}$ quivers, $H_p(F)$ decomposes uniquely up to isomorphism into a direct sum of interval modules $\bigoplus_{k=1}^m \mathbb{I}_{[b_k, d_k]}$, where each interval $[b_k, d_k]$ represents a homological feature born at index $b_k$ and persisting up to $d_k$ inclusive. Each interval yields a birth-death index pair $(b_k, d_k +1) \in \{0,\dots,n\} \times \{1,\dots,n, +\infty\}$, with the convention $d_k+1:=+\infty$ if $d_k=n$ (infinite-lived features). The multiset of these pairs forms the \emph{$p$-th persistence diagram} of the filtration.
Since in this paper birth-death indices will take real, rather than integer, values, we generalize the notion:

\begin{definition}[Space of Persistence Diagrams]
\label{def:pd_space}
The \emph{space of persistence diagrams}, denoted by $\mathcal{D}$, is the set of all finite multisets of points in $\mathbb{R} \times (\mathbb{R} \cup \{+\infty\})$ counted with multiplicity.
\end{definition}

\paragraph{Sequence of complexes.} In this paper, fixing a complex $K$, we consider a sequence of subcomplexes $\mathcal K = \{K_{T_i}\}_{i=0}^{N+1}$ of $K$, indexed by real time stamps $T_0=0 < \dots < T_{N+1}$ starting from $0$, such that the first and last complexes are empty, $K_{T_0}=K_{T_{N+1}}=\emptyset$, and $K_{T_i}\neq K_{T_{i+1}}$ for every $i$. Consecutive complexes need not be nested, so $\mathcal K$ is in general not a zigzag filtration.

\begin{definition}
\label{def:seq_compl}
We call $\mathrm{Seq}(K)$ the set of all such sequences of subcomplexes of $K$, with empty first and last complexes and distinct consecutive complexes. The trivial sequence consisting of the empty complex alone is added.
\end{definition}

\begin{remark}
In the TDA literature, a sequence of this kind is typically extended to a zigzag filtration by
inserting intermediate unions or intersections, to guarantee inclusion on each side. We instead
unroll the sequence directly into a simplex-wise zigzag filtration, as detailed in
Section~\ref{sec:diff_pd}; since this unrolling is not unique, Lemma~\ref{lem:tie} shows that the resulting persistence diagrams, converted back to the original sequence indices, agree regardless of the choice made — validating the sequence, rather than the filtration, as the primary object.
We favor this approach for two reasons: zigzag persistence lacks a general stability theorem from filtrations to diagrams, so no index-level correspondence needs to be preserved to invoke it; and omitting intermediate complexes removes artificial discontinuities in the persistence computation, shrinking the non-locally-Lipschitz region that hampers optimization.
\end{remark}

\subsection{Differentiability of maps to and from the space of persistence diagrams}
\label{sec:diff}

We recall the differentiability framework for maps to and from the space of persistence diagrams introduced in \cite{leygonie2022}.

\begin{definition}[Quotient Map {\cite[Def.~3.1]{leygonie2022}}]\label{def:quotient}
For $m, n \in \mathbb{N}$, the space of ordered persistence diagrams with $m$ paired points and $n$ unpaired points is $\mathbb{R}^{2m} \times \mathbb{R}^n$. The map $Q_{m,n}: \mathbb{R}^{2m} \times \mathbb{R}^n \to \mathcal{D}$ quotients this space by permutations $S_m \times S_n$:
$$Q_{m,n}((b_1, d_1, \dots, b_m, d_m), (v_1, \dots, v_n)) := \{(b_i, d_i)\}_{i=1}^m \cup \{(v_j, +\infty)\}_{j=1}^n \in \mathcal{D}.$$
When $n=0$, we set $Q_m:=Q_{m,0}$ in the notation.
\end{definition}

\begin{definition}[Differentiability {\cite[Def.~3.3, 3.6, 3.10, 3.13]{leygonie2022}}]\label{def:differentiability}
Let $M, N$ be smooth manifolds, $\theta \in M$, $\alpha \in \mathcal{D}$, and $r \in \mathbb{N} \cup \{\infty\}$.
\begin{enumerate}
    \item A map $B: M \to \mathcal{D}$ is \emph{$r$-differentiable at $\theta$} if there exist an open neighborhood $U \subseteq M$ of $\theta$, integers $m, n \in \mathbb{N}$, and a $C^r$ map $\widetilde{B}: U \to \mathbb{R}^{2m} \times \mathbb{R}^n$ such that:
    $B|_{U} = Q_{m,n} \circ \widetilde{B}$.
    $\widetilde{B}$ is called a \emph{local lift} of $B$ at $\theta$.
    If $r \ge 1$, the \emph{differential of $B$ at $\theta$ relative to $\widetilde{B}$} is 
    $\mathrm{d}_{\theta, \widetilde{B}}B := \mathrm{d}_\theta \widetilde{B}: T_\theta M \to \mathbb{R}^{2m} \times \mathbb{R}^n$.
    \item A map $L: \mathcal{D} \to N$ is \emph{$r$-differentiable at $\alpha$} if $L \circ Q_{m,n}$ is $C^r$ on an open neighborhood of every $\widetilde{\alpha} \in Q_{m,n}^{-1}(\alpha)$.
    If $r \ge 1$, the \emph{differential of $L$ at $\alpha$ relative to the pre-image $\widetilde{\alpha}$} is $\mathrm{d}_{\alpha, \widetilde{\alpha}}L := \mathrm{d}_{\widetilde{\alpha}}(L \circ Q_{m,n}): \mathbb{R}^{2m} \times \mathbb{R}^n \to T_{L(\alpha)} N$.
\end{enumerate}
\end{definition}

\begin{proposition}[Chain Rule {\cite[Prop.~3.14]{leygonie2022}}]
\label{prop:chain}
Let $B: M \to \mathcal{D}$ be $r$-differentiable at $\theta \in M$ and $L: \mathcal{D} \to N$ be $r$-differentiable at $B(\theta) \in \mathcal{D}$. Then $\mathcal{L}:=L \circ B: M \to N$ is $C^r$ at $\theta$, and for $r \ge 1$, its differential is independent of choice of lift and satisfies:
$$\mathrm{d}_\theta \mathcal{L} = \mathrm{d}_{B(\theta), \widetilde{B}(\theta)}L \circ \mathrm{d}_{\theta, \widetilde{B}}B: T_\theta M \to T_{\mathcal{L}(\theta)} N.$$
\end{proposition}

Practically, this allows treating intermediate persistence diagrams locally as Euclidean vectors in $\mathbb{R}^{2m} \times \mathbb{R}^n$, computing standard differentials, and composing them to obtain the overall one on $M$.

\section{Theoretical contributions}

\subsection{Persistence diagram-valued maps from sequences of complexes}
\label{sec:diff_pd}

In this section, we illustrate how persistence diagrams can be computed from sequences of simplicial complexes passing through intermediate simplex-wise zigzag filtrations. We then establish sufficient conditions for the differentiability of this pipeline via Corollary~\ref{cor:condition1}.

Let $K$ be a finite simplicial complex. Consider $\mathcal K =\{K_{T_i}\}_{i=0}^{N+1}\in \mathrm{Seq}(K)$ (see Definition \ref{def:seq_compl}) a finite sequence of subcomplexes with empty first and last complexes and distinct consecutive ones.

Each transition $K_{T_{i-1}} \to K_{T_i}$ may involve several simplices entering or leaving at once; to unroll the sequence into a well-defined simplex-wise filtration, we must decide, for such simultaneous changes, an order in which to introduce or remove them one at a time. For each transition, record an \emph{event} for every simplex that changes status: an \emph{insertion} $+\sigma$ for $\sigma \in K_{T_i} \setminus K_{T_{i-1}}$, and a \emph{deletion} $-\sigma$ for $\sigma \in K_{T_{i-1}} \setminus K_{T_i}$, both recorded at time $T_i$. Let $E := \{\pm\sigma : \sigma \in K\}$ be the set of all possible events.

Since several events may share the same time stamp $T_i$, we need a rule fixing their relative order, so that unrolling them one at a time always passes through valid simplicial complexes. To do so, define a partial order on $E$ by $+\tau \le +\sigma$ and $-\sigma \le -\tau$ whenever $\tau \subseteq \sigma$ (faces must enter before cofaces, and cofaces must be removed before faces), and extend it to a total order $(E,<)$ making face-unrelated events comparable.
With $(E,<)$ fixed, list the events by increasing time stamp $T_i$, breaking ties within a stamp by $<$, obtaining a unique sequence $\mathbf e=(e_1,\dots,e_n)\in E^n$. Applying $e_1,\dots,e_n$ one at a time from $\widetilde{K}_0:=\emptyset$ via
\begin{equation}
\label{eq:filtration_constr}
    \widetilde{K}_i := \begin{cases} \widetilde{K}_{i-1} \cup \{\sigma\} & \text{if } e_i = +\sigma, \\ \widetilde{K}_{i-1} \setminus \{\sigma\} & \text{if } e_i = -\sigma, \end{cases}
\end{equation}
yields, by construction of $(E,<)$, a valid simplicial complex at every step: this is a simplex-wise zigzag filtration $\widetilde{K}_0 \leftrightarrow \dots \leftrightarrow \widetilde{K}_n$, with $\widetilde{K}_0=\widetilde{K}_n=\emptyset$.

Thus, under $(E,<)$, the sequence $\mathcal K$ is uniquely encoded as a triple $f=(n,\mathbf t,\mathbf e)$: $n\in\mathbb N$ is the number of events, $\mathbf t=(t_1\le\dots\le t_n)\in\mathbb R^n$ records each event's time stamp $T_i$, and $\mathbf e=(e_1,\dots,e_n)\in E^n$ is the event sequence built above, satisfying $t_i=t_{i+1}\implies e_i<e_{i+1}$ and inducing the filtration $\widetilde{K}_0\leftrightarrow\dots\leftrightarrow \widetilde{K}_n$. We denote the space of all such sequence representations by
$\mathrm{SR}(K;<) \subseteq \bigsqcup_{n} \{n\} \times \mathbb{R}^n \times E^n$,
where $(n,\mathbf e)$ is the \emph{combinatorial structure} (the induced filtration) and $\mathbf t$ its \emph{temporal structure}, i.e.\ when each combinatorial event occurs. 
From now on, fix $K$ and $(E,<)$: this identifies $\mathrm{SR}(K;<)$ with $\mathrm{Seq}(K)$, via the encoding just described, which is a bijection.

For each sequence, we consider the persistence diagram in homological dimension $p$ of its induced filtration translated onto the temporal indices $\mathbf{t}$, realized by the map $\mathrm{PH}_p: \mathrm{SR}(K;<) \to \mathcal{D}$.
Crucially, as a direct consequence of Lemma \ref{lem:tie}, this diagram does not depend on the choice of the total order $(E,<)$ inducing the filtration, up to diagonal points.
Then, assume $M$ is a smooth manifold acting as the parameter space, and $S: M \to \mathrm{SR}(K;<)$ is a parametrized family of sequences. We consider the following general persistence diagram-valued map:
$$\mathcal{P}_p: M \xrightarrow{S} \mathrm{SR}(K;<) \xrightarrow{\mathrm{PH}_p} \mathcal{D}$$
for which we want to find sufficient conditions to guarantee its $r$-differentiability in the sense of Def.~\ref{def:differentiability}.
To do so, we exploit the following partition of $\mathrm{SR}(K;<)$ into cells.

\begin{definition}[Cell]
Given $f = (n, \mathbf{t}, \mathbf{e}) \in \mathrm{SR}(K;<)$, its \emph{cell} is defined as the set of sequences sharing the same combinatorial structure:
$\mathrm{cell}(f) := \{g = (n_g, \mathbf{t}_g, \mathbf{e}_g) \in \mathrm{SR}(K;<) : n_g = n, \, \mathbf{e}_g = \mathbf{e}\}$.
We define the map extracting the number of events in $ \mathrm{SR}(K;<)$ as $\#(g) := n_g \in \mathbb{N}$, and the map extracting the temporal structure of sequences in $\mathrm{cell}(f)$ as $\mathrm{temp}_{f}(g) := \mathbf{t}_g \in \mathbb{R}^{n}$.
\end{definition}

Since $g\in\mathrm{cell}(f)$ if and only if $f\in\mathrm{cell}(g)$, such cells form a partition of $\mathrm{SR}(K;<)$; furthermore, if $f,g$ are within the same cell, we have $\mathrm{temp}_{f}=\mathrm{temp}_{g}$.

The evaluation of $\mathrm{PH}_p(f)$ decomposes into a pairing step and a temporal evaluation.

\begin{enumerate}
    \item \textbf{Combinatorial part.} Consider the simplex-wise zigzag filtration induced by $\mathbf e$ over $\{0,\dots,n\}$ (see Equation \ref{eq:filtration_constr}). Since it starts and ends with the empty complex, every homological feature is born at time $\ge 1$ and is destroyed at time $\le n$, so the diagram has no infinity points; running the persistence pairing algorithm on this filtration produces a set $P^f_p$ of pairs $(i,j)$, with $i,j\in\{1,\dots,n\}$ distinct, identifying which events $e_i,e_j$ jointly create and destroy each feature.

    \item \textbf{Temporal part.} Once $P^f_p$ is obtained, $\mathrm{PH}_p(f)$ is evaluated by mapping index pairs to the corresponding time instants:
    $\mathrm{PH}_p(f) = \{(t_i, t_j)\}_{(i,j) \in P^f_p}$.
\end{enumerate}

Within the same cell, $\mathrm{PH}_p$ behaves as a permuted coordinate projection.

\begin{proposition}
\label{prop:permutation}
Let $p \in \mathbb{N}$ and $f = (n, \mathbf{t}, \mathbf{e}) \in \mathrm{SR}(K;<)$. For every $g \in \mathrm{cell}(f)$ we have $P^g_p = P^f_p$. Fixing an arbitrary order $P^f_p = ((a_1, a_2), \dots, (a_{2b-1}, a_{2b}))$ on its pairs, where $b = |P^f_p|$, the permuted coordinate projection $\widetilde{\mathrm{PH}}_{p,f}: \mathbb{R}^{n} \to \mathbb{R}^{2b}$, $\widetilde{\mathrm{PH}}_{p,f}(\mathbf{t}) := (t_{a_1}, \dots, t_{a_{2b}})$, satisfies 
$Q_{b}\bigl(\widetilde{\mathrm{PH}}_{p,f}(\mathrm{temp}_f(g))\bigr) = \mathrm{PH}_p(g)$
for every $g \in \mathrm{cell}(f)$, where $Q$ is the quotient map onto the space of persistence diagrams (Def.~\ref{def:quotient}).
\end{proposition}

Note that for every $g\in \mathrm{cell}(f)$ we have $\widetilde{\mathrm{PH}}_{p,f}=\widetilde{\mathrm{PH}}_{p,g}$.
Proposition \ref{prop:permutation} allows us to find sufficient conditions for $r$-differentiability (Definition \ref{def:differentiability}) of $\mathcal{P}_p$ by constructing suitable local lifts.

\begin{corollary}
\label{cor:condition1}
Let $M$ be a smooth manifold, $\theta \in M$, and $S: M \to \mathrm{SR}(K;<)$. If there exists an open neighborhood $U \subseteq M$ of $\theta$ such that for all $\theta' \in U$, $S(\theta') \in \mathrm{cell}(S(\theta))$, then the map $\widetilde{B}: U \to \mathbb{R}^{2\vert{}P^{S(\theta)}_p\vert{}}$ defined by:
$$\widetilde{B}(\theta') := \widetilde{\mathrm{PH}}_{p,S(\theta)}(\mathrm{temp}_{S(\theta)}(S(\theta')))$$
for all $\theta' \in U$, is a local lift of $\mathcal{P}_p = \mathrm{PH}_p \circ S$ at $\theta$.
Consequently, if in addition $\widetilde{S} := \mathrm{temp}_{S(\theta)} \circ S\vert{}_U$ is $C^r$, then $\mathcal{P}_p$ is $r$-differentiable at $\theta$, and for $r \ge 1$, the differential of $\mathcal{P}_p$ at $\theta$ relative to the local lift $\widetilde{B}$ is:
$\mathrm{d}_{\theta, \widetilde{B}}\mathcal{P}_p := \mathrm{d}_{\widetilde{S}(\theta)}\widetilde{\mathrm{PH}}_{p,S(\theta)} \circ \mathrm{d}_{\theta}\widetilde{S}$.
\end{corollary}

Note that the condition $\widetilde{S} = \mathrm{temp}_{S(\theta)} \circ S|_U$ is $C^r$ cannot be verified by checking the two component maps individually: we have not endowed $\mathrm{SR}(K;<)$ with a topology, let alone a manifold structure, so no differential is defined on it, and the chain rule cannot be applied through the sequence space.

\subsection{Temporal continuity via linear interpolation}
\label{sec:lin_int}

The space $\mathrm{SR}(K;<)$ places no constraint on how event times vary, but in practice sequences of complexes typically arise from thresholding continuous filtering values assigned to simplices at fixed integer time steps $\{1,\dots,N\}$. To differentiate our pipeline, we need this fixed-time data to vary continuously as the underlying filtering values vary continuously; we now bridge the two by linear interpolation (Eq. \ref{eq1}), producing a canonical map into $\mathrm{SR}(K;<)$.

Fix $N \in \mathbb{N}$, a simplicial complex $K$, and a threshold $\epsilon > 0$. We define the space of \emph{filtering functions on $K$ with $N$ time steps} as
\begin{equation}
\label{eq:ff_N}
FF_N(K) := \left\{ v: K \to \{0\}\times \mathbb{R}^N \times\{0\}\;\middle|\; v(\sigma)_i \ge v(\tau)_i, \, \forall i \in \{1, \dots, N\}, \, \forall \sigma \subsetneq \tau \right\},
\end{equation}
where the boundary values $v(\cdot)_0=v(\cdot)_{N+1}:=0$ are fixed by convention, and $\epsilon>0$ ensures the empty complex at both ends. A filtering function $v \in FF_N(K)$ induces a \emph{fixed-time sequence} $\{K_i\}_{i=0}^{N+1}$, with $K_0=K_{N+1}=\emptyset$, by thresholding: $\sigma \in K_i \iff v(\sigma)_i > \epsilon$. This yields a valid sequence of simplicial complexes because $v(\sigma)_i \ge v(\tau)_i$ for $\sigma\subsetneq\tau$ preserves face inclusion.

On the other hand, we obtain a \emph{continuous-time} sequence by linearly interpolating $v$ between consecutive integer steps.

\begin{definition}[Crossing]\label{def:crossing}
Let $v\in FF_N(K)$, $\sigma\in K$, and $k \in \{0, \dots, N\}$. We say $\sigma$ \emph{crosses} $\epsilon$ at $[k,k+1]$, and call $(\sigma,k)$ a \emph{crossing}, if $v(\sigma)_k \le \epsilon < v(\sigma)_{k+1}$ (an \emph{entry crossing}) or $v(\sigma)_k > \epsilon \ge v(\sigma)_{k+1}$ (an \emph{exit crossing}). In either case the denominator below never vanishes, so the \emph{crossing time}
\begin{equation}\label{eq1}
t_\sigma^{(k)}(v) := k + \frac{\epsilon - v(\sigma)_k}{v(\sigma)_{k+1} - v(\sigma)_k}
\end{equation}
is always well-defined.
\end{definition}

Inserting $\sigma$ at each entry crossing and removing it at each exit crossing, in the order given by the crossing times $t_\sigma^{(k)}(v)$ --- except that an exit and a re-entry of the same simplex at the same time cancel (the \emph{phantom crossings} of Definition~\ref{def:phantom}) --- defines a sequence in $\mathrm{Seq}(K)$, as formalized and proven in Lemma~\ref{lem:S-valid}. Denote by $S: FF_N(K) \to \mathrm{Seq}(K)\cong\mathrm{SR}(K;<)$ the map associating $v$ with this sequence (overloading the notation $S$ of Section \ref{sec:diff_pd}), omitting the dependence on $\epsilon$ from the notation.
For brevity, we define $(\mathbb{R}^{N|K|})^*:=\{0\}^{|K|}\times\mathbb{R}^{N|K|}\times\{0\}^{|K|}$ and see $FF_N(K)$ as a closed subset of it.

To express $S$ analytically, we partition $FF_N(K)$ into cells. For $v \in FF_N(K)$, define
$$\mathrm{cell}(v) := \left\{ w \in FF_N(K) \;\middle|\;
\begin{array}{l}
\forall i \in \{1, \dots, N\}, \, \forall \sigma \in K: \mathrm{sgn}(w(\sigma)_i - \epsilon) = \mathrm{sgn}(v(\sigma)_i - \epsilon);\\
\forall k \in \{0, \dots, N\}, \, \forall \sigma, \tau \in K \text{ crossing } \epsilon \text{ in } [k, k+1]: \\
t_\sigma^{(k)}(w) < t_\tau^{(k)}(w) \iff t_\sigma^{(k)}(v) < t_\tau^{(k)}(v), \text{ and} \\
t_\sigma^{(k)}(w) = t_\tau^{(k)}(w) \iff t_\sigma^{(k)}(v) = t_\tau^{(k)}(v),
\end{array}
\right\}$$
with $\mathrm{sgn}(x):=1$ if $x>0$, $0$ if $x=0$, $-1$ if $x<0$; we omit the dependence of $\mathrm{cell}(v)$ on $\epsilon$.
These cells form a finite semialgebraic partition of $FF_N(K)$ (see Proposition \ref{prop:cell-basic}), but are not smooth manifolds in general.

Write $|\cdot|:E\to K$ for the map forgetting the sign of an event, $|+\sigma|=|-\sigma|=\sigma$.

\begin{definition}[Indexing function]\label{def:indexing}
Let $v\in FF_N(K)$ with $S(v)=(n,\mathbf t,\mathbf e)$. A function $I:\{1,\dots,n\}\to\{0,\dots,N\}$ such that $\mathbf t_i = t^{I(i)}_{|e_i|}(v)$ for every $i$ is called an \emph{indexing function} of $v$. Such an $I$ always exists (a direct consequence of $S$ construction), though not necessarily uniquely.
\end{definition}

\begin{proposition}\label{prop:same_comb}
Let $v\in FF_N(K)$ with indexing function $I$. For every $w\in\mathrm{cell}(v)$, $S(w)\in\mathrm{cell}(S(v))$, and $I$ is also an indexing function of $w$.
\end{proposition}

If $S(v)=(n,\mathbf t,\mathbf e)$ and $v$ has indexing function $I$, defining the open set $W_v:=\big\{w\in(\mathbb{R^{N|K|}})^*:\ w(|e_i|)_{I(i)+1}\neq w(|e_i|)_{I(i)}\ \text{for}\ i=1,\dots,n\big\}$, by Equation \ref{eq1}, the map
\begin{equation}
\label{eq:S_tilde}
    \widetilde{S}_{v} : W_v \to \mathbb{R}^n, \qquad \widetilde{S}_v(w) = \big(t^{I(1)}_{|e_1|}(w), \dots, t^{I(n)}_{|e_n|}(w)\big),
\end{equation}
is well-defined, since it is a rational map with nowhere-vanishing denominators on $W_v$, hence $C^\infty$. In particular, since $W_v$ contains $\mathrm{cell}(v)$, Proposition \ref{prop:same_comb} implies $\widetilde{S}_{v}|_{\mathrm{cell}(v)} = \mathrm{temp}_{S(v)} \circ S|_{\mathrm{cell}(v)}$.

Specializing our general persistence diagram-valued map along this construction,
$$\mathcal{P}_p: M \xrightarrow{G} FF_N(K) \xrightarrow{S}\mathrm{SR}(K;<) \xrightarrow{\mathrm{PH}_p} \mathcal{D},$$
Corollary~\ref{cor:condition1} yields the following sufficient condition for $r$-differentiability of $\mathcal{P}_p$.

\begin{corollary}
\label{cor:condition2}
    Let $\theta \in M$. If there exists an open neighborhood $U$ of $\theta$ such that for all $\theta' \in U$, $G(\theta') \in \mathrm{cell}(G(\theta))$, and $G|_U$ is of class $C^r$, then the map $\widetilde{B}: U \to \mathbb{R}^{2|P^{S(G(\theta))}_p|}$ defined by:
    $$\widetilde{B}(\theta') := \widetilde{\mathrm{PH}}_{p,S(G(\theta))}(\widetilde{S}_{G(\theta)}(G(\theta')))$$
    for all $\theta' \in U$, is a $C^r$ local lift of $\mathcal{P}_p$ at $\theta$.
    Thus, $\mathcal{P}_p$ is $r$-differentiable at $\theta$, and the differential of $\mathcal{P}_p$ at $\theta$ relative to the local lift $\widetilde{B}$ is:
$\mathrm{d}_{\theta, \widetilde{B}}\mathcal{P}_p = \mathrm{d}_{\widetilde{S}_{G(\theta)}(G(\theta))}\widetilde{\mathrm{PH}}_{p,S(G(\theta))} \circ \mathrm{d}_{G(\theta)}\widetilde{S}_{G(\theta)} \circ \mathrm{d}_{\theta}G|_{U}$.
\end{corollary}

\begin{figure}[t]
    \centering

    \begin{subfigure}[t]{0.35\textwidth}
        \centering
        \begin{equation*}
        \begin{tikzcd}[column sep=18pt]
        U \arrow[r, dashed, "S|_U"] \arrow[dr, "\widetilde{S}" description]
        & \mathrm{cell}(S(\theta)) \arrow[d, dashed, "\mathrm{temp}_{S(\theta)}"]
        \arrow[r, dashed, "\mathrm{PH}_p|_{\mathrm{cell}}"]
        & \mathcal{D} \\
        & \mathbb{R}^{\#(S(\theta))}
        \arrow[r, "\widetilde{\mathrm{PH}}_{p,S(\theta)}"']
        & \mathbb{R}^{2|P^{S(\theta)}_p|}
        \arrow[u, dashed, "Q_{P^{S(\theta)}_p}"']
        \end{tikzcd}
        \end{equation*}
        \vspace{-0.5em}
        \caption{}
        \label{fig:commutative-diagram}
    \end{subfigure}
    \hfill
    \begin{subfigure}[t]{0.6\textwidth}
        \centering
        \begin{equation*}
        \begin{tikzcd}[column sep=17pt]
        U \arrow[r, "G|_U"]
        &
        \mathrm{cell}(G(\theta)) 
        \arrow[dr, "\widetilde{S}_{G(\theta)}|_{\mathrm{cell}}" description]
        \arrow[r, dashed, "S|_{\mathrm{cell}}"]
        &
        \mathrm{cell}(S(G(\theta)))
        \arrow[d, dashed, "\mathrm{temp}_{S(G(\theta))}"]
        \arrow[r, dashed, "\mathrm{PH}_p|_{\mathrm{cell}}"]
        &
        \mathcal{D}
        \\
        &
        &
        \mathbb{R}^{\#(S(G(\theta)))}
        \arrow[r, "\widetilde{\mathrm{PH}}_{p,S(G(\theta))}"']
        &
        \mathbb{R}^{2|P^{S(G(\theta))}_p|}
        \arrow[u, dashed, "Q_{P^{S(\theta)}_p}"']
        \end{tikzcd}
        \end{equation*}
        \vspace{-0.5em}
        \caption{}
        \label{fig:commutative-diagram-FG}
    \end{subfigure}

    \vspace{-0.5em}
    \caption{Commutative diagrams of the constructions of Corollaries
    \ref{cor:condition1} (\subref{fig:commutative-diagram}) and \ref{cor:condition2} (\subref{fig:commutative-diagram-FG}), representing the function
    $\mathcal{P}_p|_U$. The continuous line denotes the maps whose differentials we
    compute and compose to obtain the differential of the local lift
    $\widetilde{B}$.}
    \label{fig:commutative-diagrams}
    \vspace{-0.8em}
\end{figure}

\subsection{Objective optimization}
\label{sec:stability}

Let $\mathcal{P}_p: M \xrightarrow{G} FF_N(K) \xrightarrow{S}\mathrm{Seq}(K) \xrightarrow{\mathrm{PH}_p} \mathcal{D}$ be our persistence diagram-valued map.
Assume $M=\mathbb{R}^d$, and let $L:\mathcal D\rightarrow \mathbb{R}$ be a real-valued map defined on persistence diagrams\footnote{We implicitly assume $L$ to be invariant under the addition of diagonal points to a diagram, so that, combined with Lemma~\ref{lem:tie}, the composite $\mathcal{L}$ is independent of the choice of event ordering $<$.}. We are interested in minimizing the loss function $\mathcal{L}:=L\circ \mathcal{P}_p$.

Under standard conditions on $G$ and $L$, the following proposition shows that $\mathcal L$ is differentiable outside a Lebesgue-null set, and provides an explicit formula for its differential that makes use of  Corollary~\ref{cor:condition2} and the chain rule of Proposition~\ref{prop:chain}.

\begin{proposition}\label{prop:ae-differentiable}
Let $G:\mathbb R^d\to FF_N(K)$ be
definable in $\mathbb R_{an,exp}$ (see Definition \ref{def:o_minimal}). Then there is a $0$-measure set $Z\subseteq\mathbb R^d$, such that every $\theta\in\mathbb R^d\setminus Z$ has an open neighbourhood $U$ with
$G(U)\subseteq\mathrm{cell}(G(\theta))$ and $G|_U$ analytic. Consequently, for every
$\theta\in\mathbb R^d\setminus Z$, the map $\mathcal{P}_p=\mathrm{PH}_p\circ S\circ G$ is $\infty$-differentiable at $\theta$,
with differential given by Corollary~\ref{cor:condition2}; if furthermore $L$ is
$r$-differentiable (Definition~\ref{def:differentiability}), the composite
$\mathcal L=L\circ \mathcal{P}_p$ is $C^r$ on $\mathbb R^d\setminus Z$, with differential
$$\mathrm d_\theta\mathcal L=\mathrm d_{\mathcal{P}_p(\theta),\widetilde B(\theta)}L\circ\mathrm d_{\theta,\widetilde B}\mathcal{P}_p$$
obtained by composing the two differentials via the chain rule of Proposition~\ref{prop:chain}.
\end{proposition}

In particular, the Clarke subgradient of $\mathcal{L}$ (Definition~\ref{def:clark_sub}) is well-defined, and is non-empty almost everywhere. We then minimize $\mathcal L$ via stochastic subgradient descent (SSD) (Definition~\ref{def:clark}), with the subgradient implementation described in Appendix \ref{sec:subgrad_impl}.
However, convergence guarantees such as those established in Proposition~\ref{prop:clark} fail\footnote{Proposition~\ref{prop:clark} guarantees convergence to critical points rather than global optimality, as standard in non-convex optimization. Consequently, solution optimality is not discussed in the experiments of Section~\ref{sec:experiments}.}, since $\mathrm{PH}_p\circ S$, valued in $(\mathcal D,d_B)$ with $d_B$ the bottleneck distance (Definition~\ref{def:bottleneck}), is not locally Lipschitz. The following theorem identifies a null measure set $\mathcal N_p\subseteq FF_N(K)$ outside of which local Lipschitzness is guaranteed to hold; its size, shape, and examples of the resulting discontinuities are discussed in Appendix~\ref{app:region_nonloc}.

\begin{theorem}\label{thm:lipschitz_breve}
Let $p \in \mathbb{N}$, and $\mathcal{N}_p$ be the set defined in Definition~\ref{def:critical_set}.
Then the map $\mathrm{PH}_p\circ S:(FF_N(K),\|\cdot\|_\infty)\to(\mathcal D,d_B)$ is locally Lipschitz on $FF_N(K)\setminus\mathcal N_p$.
\end{theorem}

In particular, if $G$ and $L$ are locally Lipschitz, so is $\mathcal L$ on $G^{-1}(FF_N(K)\setminus\mathcal N_p)$. Although $FF_N(K)\setminus\mathcal N_p$ has full Lebesgue measure — which would heuristically favor optimization — the genericity and complexity of $G$ (e.g., a neural network) prevent us from guaranteeing any lower bound on the size of its preimage.
Independently of this limitation, we observe that the lack of local Lipschitz continuity on $\mathcal{N}_p$, although precluding theoretical convergence guarantees, does not appear to hinder convergence in practice, as demonstrated by the optimization examples presented in Section~\ref{sec:experiments}.

\subsection{Applications}
\label{sec:applications}

Providing a map $G:M\to FF_N(K)$ directly is generally infeasible, since it would require
enforcing monotonicity of face inclusions (Equation \ref{eq:ff_N}) by hand. Instead, one fixes an \emph{extension map} $\mathrm{Ext}:\mathbb R^c\to FF_N(K)$ producing a valid filtering function from an arbitrary real vector, and lets $G$ be the composition of $\mathrm{Ext}$ with a parametric model, which, overloading notation, we still call $G:M\to\mathbb R^c$. We report two standard extension maps.
Throughout, $\nu\in V$ denotes a vertex of $K$, with $V$ being the vertex set of $K$, and $\sigma^0$ the vertex set of a simplex $\sigma$.

\noindent {\bfseries Upper-star extension.} Suppose the filtering function is induced by scalar
values on the vertices of $K$: for each time $i=1,\dots,N$ we assign $x(\nu)_i\in\mathbb R$ to every vertex $\nu\in V$, i.e. $x\in\mathbb R^{N|V|}$, and extend it to all of $K$ via $\mathrm{US}:\mathbb R^{N|V|}\to FF_N(K)$, defined as $\mathrm{US}(x)(\sigma)_i:=\min_{\nu\in\sigma^0}x(\nu)_i$.
A minimum over a larger vertex set is smaller, so $\mathrm{US}(x)(\sigma)_i\ge\mathrm{US}(x)(\tau)_i$ for $\sigma\subsetneq\tau$, and $\mathrm{US}(x)\in FF_N(K)$ with the
boundary convention $v(\cdot)_0=v(\cdot)_{N+1}=0$.

\noindent {\bfseries Vietoris--Rips extension.} Let a point cloud of $b$ points in $\mathbb R^c$ move over $N$ time steps, and let $K$ be the full simplex on these $b$ points. Write $x\in\mathbb R^{cbN}$ for a trajectory, with $x(\nu)_i\in\mathbb R^c$ the position of point $\nu$ at time $i$, and
$\mathrm{diam}(A):=\max_{y,y'\in A}\|y-y'\|$ for a finite $A\subseteq\mathbb R^c$, so that $\mathrm{diam}=0$
on singletons. We define $\mathrm{VR}:\mathbb R^{cbN}\to FF_N(K)$ as $\mathrm{VR}(x)(\sigma)_i:=2\epsilon-\mathrm{diam}(\{x(\nu)_i:\nu\in\sigma^0\})$.
Since the diameter is monotone under inclusion of vertex sets,
$\mathrm{VR}(x)(\sigma)_i\ge\mathrm{VR}(x)(\tau)_i$ for $\sigma\subsetneq\tau$, and $\mathrm{VR}(x)\in FF_N(K)$ with $v(\cdot)_0=v(\cdot)_{N+1}=0$. With the
selection rule of Section~\ref{sec:lin_int}, $\sigma$ is present at time $i$ if and only if $\mathrm{diam}(\{x(\nu)_i:\nu\in\sigma^0\})<\epsilon$: the
thresholded fixed-time sequence is the Vietoris--Rips complex of the point cloud at scale
$\epsilon$ at each index.

In both cases the extension map is semialgebraic, so Proposition~\ref{prop:ae-differentiable}
applies and the differential of $\mathcal L$ can be computed almost everywhere by the classical
chain rule.

\input{sections/exp_coverage}

\input{sections/exp_mooc}

\input{sections/conclusions}

\bibliographystyle{preprint}
\bibliography{references}

\appendix

\input{sections/related}

\section{Further background}

\subsection{Whitney stratifications and o-minimal geometry}

Let $r\in\mathbb{N}\cup\{\infty,\omega\}$. A subset $M\subseteq\mathbb R^d$ is a \emph{$C^r$ submanifold}
of dimension $c$ if every $x\in M$ has an open neighbourhood $U$ in $\mathbb R^d$ and a $C^r$ map
$F:U\to\mathbb R^{d-c}$ with $\mathrm dF(x)$ of full rank and $M\cap U=\{y\in U:F(y)=0\}$; its tangent and
normal spaces at $x$ are $T_xM:=\ker\mathrm dF(x)$ and $N_xM:=(T_xM)^\perp$. A submanifold of
dimension $d$ is an open subset of $\mathbb R^d$.

\begin{definition}[{Whitney stratification, \cite[Def.~5.6]{davis2020stochastic}}]\label{def:whitney}
A \emph{Whitney $C^r$-stratification} of $Q\subseteq\mathbb R^d$ is a finite partition $\mathcal S$ of $Q$
into non-empty $C^r$ submanifolds of $\mathbb R^d$, called \emph{strata}, such that:
\begin{description}
\item[(Frontier)] for strata $L\ne M$, if $L\cap\overline M\neq\emptyset$ then $L\subseteq\overline M$
(closure in $\mathbb R^d$);
\item[(Whitney (a))] for strata $L,M$, if $z_k\in M$ converge to $\bar z\in L$ and normal vectors
$v_k\in N_{z_k}M$ converge to $v$, then $v\in N_{\bar z}L$.
\end{description}
A stratum of dimension $d$ is called \emph{top-dimensional}. A function $f:\mathbb R^d\to\mathbb R$ is
\emph{Whitney $C^r$-stratifiable} if its graph admits a Whitney $C^r$-stratification.
\end{definition}

\begin{definition}[Weakly stratified map]\label{def:weakly-stratified}
Let $\mathcal S$ be a Whitney stratification of $\mathbb R^d$ and let $N\subseteq\mathbb R^m$ be endowed with a
partition $\mathcal C_N$. A map $f:\mathbb R^d\to N$ is \emph{weakly stratified} with respect to
$(\mathcal S,\mathcal C_N)$ if for every $C\in\mathcal C_N$ the preimage $f^{-1}(C)$ is a union of strata of
$\mathcal S$; equivalently, every stratum is mapped into a single piece of $\mathcal C_N$.
\end{definition}

\begin{definition}[{o-minimal structure, \cite[Def.~5.10]{davis2020stochastic}}]\label{def:o_minimal}
An \emph{o-minimal structure} on the real field $(\mathbb{R}, +, \cdot, <)$ is a sequence $\mathcal{O} = (\mathcal{O}_n)_{n \in \mathbb{N}}$, where each $\mathcal{O}_n$ is a collection of subsets of $\mathbb{R}^n$ satisfying the following axioms:
\begin{enumerate}
    \item $\mathcal{O}_1$ consists exactly of all finite unions of points and open intervals in $\mathbb{R}$;
    \item for each $n \in \mathbb{N}$, all algebraic subsets of $\mathbb{R}^n$ belong to $\mathcal{O}_n$;
    \item each $\mathcal{O}_n$ is closed under finite unions, intersections, and complementation in $\mathbb{R}^n$;
    \item if $A \in \mathcal{O}_n$ and $B \in \mathcal{O}_m$, then $A \times B \in \mathcal{O}_{n+m}$;
    \item if $\pi: \mathbb{R}^{n+1} \to \mathbb{R}^n$ is the canonical linear projection onto the first $n$ coordinates and $A \in \mathcal{O}_{n+1}$, then $\pi(A) \in \mathcal{O}_n$.
\end{enumerate}
A set $E \subseteq \mathbb{R}^n$ is called a \emph{definable set} if $E \in \mathcal{O}_n$. A map $\mathcal{L}: E \to \mathbb{R}^m$ is called a \emph{definable map} if its graph $\mathrm{graph}(\mathcal{L}) := \{(x, \mathcal{L}(x)) \in \mathbb{R}^{n+m} \mid x \in E\}$ is a definable set in $\mathbb{R}^{n+m}$.

An important result by \cite{wilkie1996model} established that the structure $\mathbb R_{exp}$, which includes the exponential function $x \mapsto \exp(x)$, is o-minimal. Formally, $\mathbb R_{exp}$ is the smallest o-minimal structure containing the graphs of all polynomials and the exponential function. When further expanded to include restricted analytic functions (such as sine or cosine on compact intervals), the resulting structure, denoted $\mathbb R_{an,exp}$, is o-minimal as well \cite{van1994real}. These frameworks allow us to treat functions that combine semi-algebraic components with transcendental ones, such as sigmoids, logarithms, or Gaussian kernels, as definable.
\end{definition}

\begin{remark}
     Functions definable over an o-minimal structure are Whitney $C^r$-stratifiable for any $r\in\mathbb{N}$ \cite{vandendries1996geometric}.
\end{remark}

\subsection{The stochastic subgradient method}

We introduce the stochastic subgradient method for optimizing almost everywhere differentiable functions $\mathcal{L}: \mathbb{R}^d \to \mathbb{R}$, together with its asymptotic convergence guarantees developed in \cite{davis2020stochastic}.

\begin{definition}[Clarke subgradient]
\label{def:clark_sub}
Let $\mathcal{L}: \mathbb{R}^d \to \mathbb{R}$ be almost everywhere differentiable. The \emph{Clarke subgradient} of $\mathcal{L}$ at $\theta \in \mathbb{R}^d$ is defined as:
$$\partial \mathcal{L}(\theta) := \mathrm{conv} \left\{ v \in \mathbb{R}^d \;\middle|\; \exists\, (\theta_i)_{i \in \mathbb{N}} \subset \mathrm{Dom}(\nabla \mathcal{L}) \text{ s.t. } \lim_{i \to \infty} \theta_i = \theta \text{ and } \lim_{i \to \infty} \nabla \mathcal{L}(\theta_i) = v \right\},$$
where $\mathrm{Dom}(\nabla \mathcal{L}) \subseteq \mathbb{R}^d$ denotes the dense set of points where $\mathcal{L}$ is Fréchet differentiable, and $\mathrm{conv}(\cdot)$ denotes the convex hull. A point $\theta^* \in \mathbb{R}^d$ is called a \emph{critical point} of $\mathcal{L}$ if $0 \in \partial \mathcal{L}(\theta^*)$.
\end{definition}

The stochastic subgradient method optimizes $\mathcal{L}$ by iteratively taking steps opposite to noisy estimates of a Clarke subgradient.

\begin{definition}
\label{def:clark}
Given an initial parameter $\theta_0 \in \mathbb{R}^d$, the \emph{stochastic subgradient method} updates the parameter estimate $\theta_k \in \mathbb{R}^d$ at step $k \in \mathbb{N}$ according to:
$$\theta_{k+1} = \theta_k - \gamma_k (g_k + \zeta_k), \quad \text{with } g_k \in \partial \mathcal{L}(\theta_k),$$
where $(\gamma_k)_{k \in \mathbb{N}}$ is a sequence of positive step sizes (learning rates), $g_k \in \partial \mathcal{L}(\theta_k)$ is a Clarke subgradient, and $(\zeta_k)_{k \in \mathbb{N}}$ is a sequence of $\mathbb{R}^d$-valued random variables representing noise in the subgradient estimates.
\end{definition}

Note that, by Rademacher's theorem, a locally Lipschitz function is differentiable almost everywhere.
The main convergence result for stochastic subgradient descent is stated below.

\begin{proposition}[{\cite[Cor.~5.9]{davis2020stochastic}}]
\label{prop:clark}
Let $\mathcal{L}: \mathbb{R}^d \to \mathbb{R}$ be a locally Lipschitz Whitney $C^r$-stratifiable function. Consider the sequence of iterates $(\theta_k)_{k \in \mathbb{N}}$ generated by the stochastic subgradient method under the following standard assumptions:
\begin{enumerate}
    \item \textbf{Step-size schedule}: The step sizes $\gamma_k \ge 0$ satisfy $\sum_{k=0}^{\infty} \gamma_k = +\infty$ and $\sum_{k=0}^{\infty} \gamma_k^2 < +\infty$.
    \item \textbf{Boundedness of iterates}: Almost surely, the sequence $(\theta_k)_{k \in \mathbb{N}}$ is bounded, i.e., $\sup_{k \ge 0} \|\theta_k\| < +\infty$.
    \item \textbf{Unbiased noise with bounded variance}: Denoting by $\mathcal{F}_k = \sigma(\{\theta_0, \dots, \theta_k, g_0, \dots, g_k, \zeta_0, \dots, \zeta_{k-1}\})$ the increasing $\sigma$-algebra generated by the history up to step $k$, there exists a function $p: \mathbb{R}^d \to \mathbb{R}$ bounded on bounded sets such that almost surely, for every $k \ge 0$:
    $\mathbb{E}[\zeta_k \mid \mathcal{F}_k] = 0$ and $\mathbb{E}[\|\zeta_k\|^2 \mid \mathcal{F}_k] \le p(\theta_k)$.
\end{enumerate}
Then, almost surely:
\begin{enumerate}
    \item Every limit point $\theta^*$ of the sequence $(\theta_k)_{k \in \mathbb{N}}$ is a critical point of $\mathcal{L}$, that is, $0 \in \partial \mathcal{L}(\theta^*)$.
    \item The sequence of objective values $(\mathcal{L}(\theta_k))_{k \in \mathbb{N}}$ converges to a limit value $\mathcal{L}^*$, which coincides with $\mathcal{L}(\theta^*)$.
\end{enumerate}
\end{proposition}

\subsection{Bottleneck distance}

A persistence diagram $D\in\mathcal D$ is a finite multiset of points $(b,d)\in\mathbb R\times(\mathbb R\cup\{+\infty\})$, as in Definition~\ref{def:pd_space}. Here we restrict attention to diagrams whose points satisfy $b\le d$ (birth before death); this is not an actual restriction for our purposes, since $\mathrm{PH}_p$ always produces such points by construction. We call a point \emph{essential} if $d=+\infty$, and \emph{finite} otherwise.

Let $\Delta:=\{(x,x):x\in\mathbb R\}$ denote the diagonal, each of its points counted with countably infinite multiplicity. For $D\in\mathcal D$, write $D^+:=D\cup\Delta$. Adding $\Delta$ with infinite multiplicity allows a diagram with few off-diagonal points to be matched, point by point, against one with many: unmatched finite points are simply paired with a (freely available) copy of their own projection onto $\Delta$. Essential points, having infinite distance to every point of $\Delta$, can never be matched to $\Delta$; they must be matched to essential points of the other diagram, as made precise below.

\begin{definition}[Bottleneck distance]\label{def:bottleneck}
Let $D_1,D_2\in\mathcal D$, and let $D_1^+,D_2^+$ be as above. A \emph{matching} between $D_1$ and $D_2$ is a bijection $\gamma:D_1^+\to D_2^+$ such that $\gamma$ restricts to a bijection between the essential points of $D_1^+$ and those of $D_2^+$ (equivalently, $D_1$ and $D_2$ have the same multiset of essential points). If no such $\gamma$ exists, set $d_B(D_1,D_2):=+\infty$. Otherwise, the \emph{bottleneck distance} between $D_1$ and $D_2$ is
$$
d_B(D_1,D_2):=\inf_{\gamma}\ \sup_{x\in D_1^+}\ \|x-\gamma(x)\|_\infty,
$$
where the infimum is over all matchings $\gamma$ as above, $\|(x_1,x_2)-(y_1,y_2)\|_\infty:=\max(|x_1-y_1|,|x_2-y_2|)$, and the convention $\infty-\infty=0$ applies when matching two essential points.
\end{definition}

Equivalently, for finite $d_B(D_1,D_2)$, this is the smallest $\delta\ge0$ for which there exists a partial matching between the finite off-diagonal points of $D_1$ and those of $D_2$ (with the essential points matched exactly against each other), leaving unmatched finite points within $\ell^\infty$-distance $\delta$ of $\Delta$, such that every matched pair is within $\ell^\infty$-distance $\delta$ of each other. The bottleneck distance defines an extended metric on $\mathcal D$ (up to the identification of diagrams differing only on $\Delta$).

\section{Theoretical details}
\label{app:theory}

\subsection{Sequence representations}\label{app:rep}

Throughout, $K$ is a fixed finite simplicial complex, $E=\{\pm\sigma:\sigma\in K\}$, and $(E,\le)$ is the partial order of Section~\ref{sec:diff_pd}. For an event sequence $\mathbf e=(e_1,\dots,e_n)\in E^n$ we write $\widetilde K_0:=\emptyset$ and $\widetilde K_i:=\widetilde K_{i-1}\cup\{\sigma\}$ if $e_i=+\sigma$, $\widetilde K_i:=\widetilde K_{i-1}\setminus\{\sigma\}$ if $e_i=-\sigma$. We say that $\mathbf e$ is \emph{valid} if for every $i$: $e_i=+\sigma$ implies $\sigma\notin\widetilde K_{i-1}$ and every facet of $\sigma$ lies in $\widetilde K_{i-1}$; $e_i=-\sigma$ implies $\sigma\in\widetilde K_{i-1}$ and no coface of $\sigma$ lies in $\widetilde K_{i-1}$. Equivalently, every $\widetilde K_i$ is a simplicial complex and each step changes exactly one simplex, so that a valid $\mathbf e$ induces the simplex-wise zigzag filtration $\widetilde K_0\leftrightarrow\dots\leftrightarrow\widetilde K_n$.
Since every event toggles the membership of its simplex, the complex $\widetilde K_i$ reached after the first $i$ events of a valid sequence is the set of simplices toggled an odd number of times by $e_1,\dots,e_i$: it depends only on the multiset $\{e_1,\dots,e_i\}$, not on the order.

In this appendix only, it is convenient to allow sequences of subcomplexes (see Definition \ref{def:seq_compl}) in which consecutive complexes may coincide. Let $\mathrm{Seq}^*(K)\supseteq\mathrm{Seq}(K)$ be the set of sequences $\mathcal K=\{K_{T_i}\}_{i=0}^{N+1}$ with $T_0=0<\dots<T_{N+1}$, $K_{T_0}=K_{T_{N+1}}=\emptyset$ and no condition on consecutive complexes; a stamp $T_i$, $i\ge1$, with $K_{T_i}=K_{T_{i-1}}$ is called \emph{idle}, and $\mathrm{Seq}(K)$ is the subset of sequences without idle stamps. Sequences with idle stamps never arise from the constructions of the paper; they only host the auxiliary sequences with detours used in the proof of Theorem~\ref{thm:lipschitz_breve}.

\begin{definition}[Sequence representations]\label{def:seq-rep}
\begin{enumerate}
\item[(a)] Let $\mathcal K=\{K_{T_i}\}_{i=0}^{N+1}\in\mathrm{Seq}^*(K)$. An \emph{extended representation} of $\mathcal K$ is a triple $f=(n,\mathbf t,\mathbf e)$ with $n\in\mathbb N$, $\mathbf t=(t_1\le\dots\le t_n)\in\{T_1,\dots,T_{N+1}\}^n$ non-decreasing and $\mathbf e\in E^n$ valid, such that for every $i\in\{1,\dots,N+1\}$ the complex obtained after applying all the events with time $\le T_i$ is $K_{T_i}$, i.e.\ $\widetilde K_{\ell_i}=K_{T_i}$ with $\ell_i:=\max\{j:t_j\le T_i\}$ (and $\ell_i:=0$ if no such $j$ exists). We write $\mathrm{Rep}^*(\mathcal K)$ for the set of extended representations of $\mathcal K$.
\item[(b)] Let $\mathcal K\in\mathrm{Seq}(K)$. A \emph{representation} of $\mathcal K$ is an extended representation $f$ such that, for every $i$, the events of $f$ with time $T_i$ are exactly the insertions $+\sigma$, for $\sigma\in K_{T_i}\setminus K_{T_{i-1}}$, and the deletions $-\sigma$, for $\sigma\in K_{T_{i-1}}\setminus K_{T_i}$, each occurring once. We write $\mathrm{Rep}(\mathcal K)\subseteq\mathrm{Rep}^*(\mathcal K)$ for the set of representations of $\mathcal K$.
\item[(c)] Given a total order $(E,<)$ extending $(E,\le)$, a representation $f$ of $\mathcal K\in\mathrm{Seq}(K)$ is \emph{ordered by $<$} if $t_i=t_{i+1}$ implies $e_i<e_{i+1}$. We write $\mathrm{Rep}(\mathcal K;<)\subseteq\mathrm{Rep}(\mathcal K)$ for the set of representations of $\mathcal K$ ordered by $<$.
\end{enumerate}
We set:
\begin{align*}
\mathrm{SR}(K;<)&:=\bigcup_{\mathcal K\in\mathrm{Seq}(K)}\mathrm{Rep}(\mathcal K;<),\\
\mathrm{SR}(K)&:=\bigcup_{\mathcal K\in\mathrm{Seq}(K)}\mathrm{Rep}(\mathcal K),\\
\mathrm{SR}^*(K)&:=\bigcup_{\mathcal K\in\mathrm{Seq}^*(K)}\mathrm{Rep}^*(\mathcal K).
\end{align*}
\end{definition}

The set $\mathrm{SR}(K;<)$ coincides with the space of encodings constructed in Section~\ref{sec:diff_pd}. Once a linear extension $<$ of $(E,\le)$ is fixed, each sequence $\mathcal K$ has exactly one representation that is ordered by $<$ (Remark~\ref{rem:rep}(i) and (iii)).

\begin{definition}[Tie classes and detours]\label{def:tie-detour}
Let $f=(n,\mathbf t,\mathbf e)\in\mathrm{SR}^*(K)$. For a time $s\in\{t_1,\dots,t_n\}$, the \emph{tie class} of $f$ at $s$ is the set of positions $\{i:t_i=s\}$. A \emph{detour} of $f$ is a pair of positions $a<b$ in the same tie class acting on the same simplex $\sigma$ with no event of $\sigma$ strictly between them; by validity, $\{e_a,e_b\}=\{+\sigma,-\sigma\}$.
\end{definition}

\begin{remark}\label{rem:rep}
Let $\mathcal K=\{K_{T_i}\}_{i=0}^{N+1}\in\mathrm{Seq}^*(K)$ and, for $i\in\{1,\dots,N+1\}$, let $E_i:=\{+\sigma:\sigma\in K_{T_i}\setminus K_{T_{i-1}}\}\cup\{-\sigma:\sigma\in K_{T_{i-1}}\setminus K_{T_i}\}\subseteq E$ be the set of events of the transition $K_{T_{i-1}}\to K_{T_i}$; $E_i=\emptyset$ exactly when $T_i$ is idle, so $E_i\neq\emptyset$ for every $i$ when $\mathcal K\in\mathrm{Seq}(K)$.

(i) By definition $\mathrm{Rep}(\mathcal K;<)\subseteq\mathrm{Rep}(\mathcal K)\subseteq\mathrm{Rep}^*(\mathcal K)$ for $\mathcal K\in\mathrm{Seq}(K)$, hence $\mathrm{SR}(K;<)\subseteq\mathrm{SR}(K)\subseteq\mathrm{SR}^*(K)$. Every triple $f=(n,\mathbf t,\mathbf e)$ with $\mathbf t$ non-decreasing, $\mathbf e$ valid and $\widetilde K_n=\emptyset$ belongs to $\mathrm{SR}^*(K)$: it is an extended representation of the sequence $\mathcal K_f\in\mathrm{Seq}^*(K)$ whose stamps are the distinct entries of $\mathbf t$ and whose complexes are those reached after the last event of each stamp. If $f\in\mathrm{Rep}^*(\mathcal K)$ with $\mathcal K\in\mathrm{Seq}(K)$, then every stamp $T_i$, $i\ge1$, carries at least one event of $f$, because $K_{T_i}\neq K_{T_{i-1}}$; hence $\mathcal K$ is determined by $f$ (its stamps are the distinct entries of $\mathbf t$, its complexes the $\widetilde K_{\ell_i}$), the unions defining $\mathrm{SR}(K;<)$ and $\mathrm{SR}(K)$ are disjoint, and, by (iii) below, the encoding of Section~\ref{sec:diff_pd}, $\mathcal K\mapsto$ the unique element of $\mathrm{Rep}(\mathcal K;<)$, is a bijection $\mathrm{Seq}(K)\to\mathrm{SR}(K;<)$.

(ii) Let $f\in\mathrm{Rep}^*(\mathcal K)$. Then $f$ has no detour if and only if $f\in\mathrm{SR}(K)$, in which case $f\in\mathrm{Rep}(\mathcal K')$ for the sequence $\mathcal K'\in\mathrm{Seq}(K)$ obtained from $\mathcal K$ by deleting its idle stamps. Indeed, in a tie class of a representation each simplex occurs at most once. Conversely, if each simplex occurs at most once in the tie class at $T_i$, the events at $T_i$ toggle distinct simplices and turn $K_{T_{i-1}}$ into $K_{T_i}$; by validity a simplex is inserted only if absent and deleted only if present, so the events at $T_i$ are exactly the insertions of $K_{T_i}\setminus K_{T_{i-1}}$ and the deletions of $K_{T_{i-1}}\setminus K_{T_i}$, i.e.\ the elements of $E_i$; in particular an idle stamp carries no event, and deleting the idle stamps yields $\mathcal K'\in\mathrm{Seq}(K)$ with $f\in\mathrm{Rep}(\mathcal K')$. Consequently two representations of the same $\mathcal K\in\mathrm{Seq}(K)$ have the same time vector and the same set $E_i$ of events on every tie class, and differ only in the order chosen inside tie classes.

(iii) Let $\mathcal K\in\mathrm{Seq}(K)$. An insertion of $E_i$ is never a face or a coface of a deletion of $E_i$: if $\sigma\in K_{T_i}\setminus K_{T_{i-1}}$ and $\tau\in K_{T_{i-1}}\setminus K_{T_i}$, then $\tau\subseteq\sigma$ would give $\tau\in K_{T_i}$ and $\sigma\subseteq\tau$ would give $\sigma\in K_{T_{i-1}}$. Moreover, the orderings of $E_i$ whose sequential application from $K_{T_{i-1}}$ is valid are exactly the linear extensions of $(E_i,\le)$: faces before cofaces among insertions, cofaces before faces among deletions, insertions and deletions interleaved arbitrarily. Indeed, in a linear extension an insertion $+\tau$ finds every facet of $\tau$ present, since a facet lies either in $K_{T_{i-1}}\cap K_{T_i}$, and is never touched at $T_i$, or in $K_{T_i}\setminus K_{T_{i-1}}$, and has been inserted before $\tau$; a deletion $-\sigma$ finds no coface present, since the cofaces of $\sigma$ in $K_{T_{i-1}}$ are not in $K_{T_i}$ and have been deleted before $\sigma$, and no coface of $\sigma$ is inserted at $T_i$. Conversely, an ordering placing $+\sigma$ before $+\tau$ with $\tau\subsetneq\sigma$, or $-\tau$ before $-\sigma$ with $\tau\subsetneq\sigma$, is invalid, since $\tau$ is absent when $\sigma$ is inserted, respectively $\sigma$ is present when $\tau$ is deleted. As a consequence: $\mathrm{Rep}(\mathcal K)\neq\emptyset$; $\mathrm{Rep}(\mathcal K;<)$ consists of exactly one element, the encoding of Section~\ref{sec:diff_pd}, so that the construction of Section~\ref{sec:diff_pd} passes through valid complexes as claimed there; and two events of a tie class of a representation that are incomparable in $(E,\le)$ act on distinct simplices, neither of which is a face of the other (they are either of the same sign and not nested, or of opposite signs).

(iv) The persistence pairing $P^f_p$, the diagram $\mathrm{PH}_p(f)$, the permuted coordinate projection $\widetilde{\mathrm{PH}}_{p,f}$, the cell $\mathrm{cell}(f)$, the map $\mathrm{temp}_f$ and Proposition~\ref{prop:permutation} are defined, and hold verbatim, for every $f\in\mathrm{SR}^*(K)$, with $\mathrm{cell}(f):=\{g\in\mathrm{SR}^*(K):n_g=n,\ \mathbf e_g=\mathbf e\}$: they only use the filtration induced by the valid event sequence $\mathbf e$ and the time vector $\mathbf t$.
\end{remark}

We first prove that exchanging two simultaneous events acting on face-unrelated simplices does not change the persistence diagram off the diagonal; Lemma~\ref{lem:tie} then follows by connectivity of linear extensions. Recall the indexing convention: the filtration induced by $\mathbf e$ is $\widetilde K_0\leftrightarrow\widetilde K_1\leftrightarrow\dots\leftrightarrow\widetilde K_n$, the transition $\widetilde K_{k-1}\leftrightarrow\widetilde K_k$ corresponding to the event $e_k$ with time stamp $t_k$. An interval summand $\mathbb{I}_{[b,d]}$ of $H_p$ of this filtration, supported on the nodes $b,\dots,d$, is born by the event $e_b$ (time $t_b$) and dies by the event $e_{d+1}$ (time $t_{d+1}$); since $\widetilde K_0=\widetilde K_n=\emptyset$ one has $1\le b\le d\le n-1$, and $\mathbb{I}_{[b,d]}$ evaluates to the point $(t_b,t_{d+1})\in\mathbb R^2$ of $\mathrm{PH}_p$.

\begin{lemma}[Transposition step]\label{lem:transp}
Let $\mathcal K\in\mathrm{Seq}^*(K)$, let $f=(n,\mathbf t,\mathbf e)\in\mathrm{Rep}^*(\mathcal K)$ and let $m$ be such that $t_m=t_{m+1}$ and $e_m,e_{m+1}$ act on distinct simplices, neither of which is a face of the other. Let $f'=(n,\mathbf t,\mathbf e')$ be obtained from $f$ by exchanging $e_m$ and $e_{m+1}$. Then $f'\in\mathrm{Rep}^*(\mathcal K)$ and $\mathrm{PH}_p(f)$, $\mathrm{PH}_p(f')$ coincide as multisets after removing the points on the diagonal. If $\mathcal K\in\mathrm{Seq}(K)$ and $f\in\mathrm{Rep}(\mathcal K)$, then $f'\in\mathrm{Rep}(\mathcal K)$.
\end{lemma}

\begin{proof}
\emph{Validity and membership.} Write $e_m=\pm\sigma$, $e_{m+1}=\pm\tau$, and let $\widetilde K_j$ be the complexes induced by $\mathbf e$. Applying $e_{m+1}$ to $\widetilde K_{m-1}$ is valid: if $e_{m+1}=+\tau$, then $\tau\notin\widetilde K_{m-1}$ (as $\tau\notin\widetilde K_m$ and $\tau\ne\sigma$) and the facets of $\tau$ lie in $\widetilde K_m\subseteq\widetilde K_{m-1}\cup\{\sigma\}$, none of them being $\sigma$; if $e_{m+1}=-\tau$, then $\tau\in\widetilde K_{m-1}$ (as $\tau\in\widetilde K_m$ and $\tau\ne\sigma$) and a coface of $\tau$ in $\widetilde K_{m-1}\subseteq\widetilde K_m\cup\{\sigma\}$ would lie in $\widetilde K_m$, which contains none, or equal $\sigma$, which is not a coface. The same argument applies to $e_m$ on either $\widetilde K_{m-1}\cup\{\tau\}$ or $\widetilde K_{m-1}\setminus\{\tau\}$, and the complex reached after both events is $\widetilde K_{m+1}$ in either order. Hence $\mathbf e'$ is valid, and since $\mathbf t$ is unchanged, the events with time $t_m$ are the same, and the complexes reached after each stamp are unchanged, $f'\in\mathrm{Rep}^*(\mathcal K)$; if $f\in\mathrm{Rep}(\mathcal K)$ the set of events at every stamp is that of $f$, so $f'\in\mathrm{Rep}(\mathcal K)$.

\emph{Diagrams.} The modules $H_p(f)$ and $H_p(f')$ differ only at the node $m$. We distinguish two cases.

\paragraph{Case 1: Same direction events ($+\sigma, +\tau$ or $-\sigma, -\tau$).}
Assume without loss of generality $e_m = +\sigma$ and $e_{m+1} = +\tau$, so that $\widetilde K_{m+1} = \widetilde K_{m-1} \cup \{\sigma, \tau\}$ (the case $-\sigma,-\tau$ is identical with the arrows reversed). The two sequences yield the modules
$$H_p(f): \dots \leftrightarrow H_p(\widetilde K_{m-1}) \xrightarrow{i_1^*} H_p(\widetilde K_{m-1} \cup \{\sigma\}) \xrightarrow{i_2^*} H_p(\widetilde K_{m+1}) \leftrightarrow \dots$$
$$H_p(f'): \dots \leftrightarrow H_p(\widetilde K_{m-1}) \xrightarrow{j_1^*} H_p(\widetilde K_{m-1} \cup \{\tau\}) \xrightarrow{j_2^*} H_p(\widetilde K_{m+1}) \leftrightarrow \dots$$
By functoriality, $i_2^* \circ i_1^* = j_2^* \circ j_1^* = \iota^*$, where $\iota: \widetilde K_{m-1} \hookrightarrow \widetilde K_{m+1}$. Restricting both modules to the index set $\{0,\dots,n\}\setminus\{m\}$, with the composite $\iota^*$ as the map from node $m-1$ to node $m+1$, yields the same module. Restriction of an interval summand $\mathbb{I}_{[b,d]}$ is the interval summand on $[b,d]\setminus\{m\}$, or zero if $[b,d]=[m,m]$, and restriction commutes with direct sums; hence the multisets of restricted intervals of $H_p(f)$ and $H_p(f')$ coincide. The value pair of a summand is determined by its restriction: a right endpoint $m-1$ or $m$ means death by the event $e_m$ or $e_{m+1}$, both of time $t_m=t_{m+1}$; a left endpoint $m$ or $m+1$ means birth by the event $e_m$ or $e_{m+1}$, of the same time; every other endpoint is unchanged by restriction; and the summands $\mathbb{I}_{[m,m]}$, which restrict to zero, evaluate to $(t_m,t_{m+1})=(t_m,t_m)$, a point on the diagonal. Hence the value diagrams coincide after removing diagonal points. 

\paragraph{Case 2: Mixed direction events ($+\sigma,-\tau$).}
As before, we may assume without loss of generality $e_m=+\sigma$, $e_{m+1}=-\tau$. Put $K_\cup:=\widetilde K_{m-1}\cup\{\sigma\}$ and $K_\cap:=\widetilde K_{m-1}\setminus\{\tau\}$; these are the intermediate complexes of the two orders, and $K_\cup=\widetilde K_{m-1}\cup \widetilde K_{m+1}$, $K_\cap=\widetilde K_{m-1}\cap \widetilde K_{m+1}$. The two modules differ only at node $m$:
$$H_p(f):\ \dots\leftrightarrow H_p(\widetilde K_{m-1})\rightarrow H_p(K_\cup)\leftarrow H_p(\widetilde K_{m+1})\leftrightarrow\dots$$
$$H_p(f'):\ \dots\leftrightarrow H_p(\widetilde K_{m-1})\leftarrow H_p(K_\cap)\rightarrow H_p(\widetilde K_{m+1})\leftrightarrow\dots$$
The strong diamond principle \cite[Thm.~5.8]{carlsson2010zigzag} gives a bijection between the persistence diagrams of $H_*(f)$ and $H_*(f')$, under which left endpoints $m\leftrightarrow m+1$ and right endpoints $m-1\leftrightarrow m$ are exchanged, all other endpoints being fixed. Summands of the form $\mathbb{I}_{[m,m]}$ are matched between $H_p(f)$ and $H_{p-1}(f')$ and between $H_p(f')$ and $H_{p+1}(f)$, all other matchings preserve the homological dimension. A left endpoint $m$ or $m+1$ is a birth by the event $e_m$ or $e_{m+1}$ and a right endpoint $m-1$ or $m$ is a death by the event $e_m$ or $e_{m+1}$; since $t_m=t_{m+1}$, the bijection preserves value pairs. The summands $\mathbb{I}_{[m,m]}$ evaluate to $(t_m,t_{m+1})=(t_m,t_m)$, a point on the diagonal. Hence the two diagrams coincide after removing diagonal points.
\end{proof}

\begin{lemma}\label{lem:tie}
Let $f,f'\in\mathrm{Rep}(\mathcal K)$ be two representations of the same sequence $\mathcal K\in\mathrm{Seq}(K)$ (for instance $f\in\mathrm{Rep}(\mathcal K;<_1)$ and $f'\in\mathrm{Rep}(\mathcal K;<_2)$ for two total orders $<_1,<_2$ extending $(E,\le)$). Then $\mathrm{PH}_p(f)$ and $\mathrm{PH}_p(f')$ coincide as multisets after removing the points on the diagonal.
\end{lemma}

\begin{proof}
By Remark~\ref{rem:rep}(ii), $f$ and $f'$ have the same time vector and the same set $E_i$ of events on the tie class at every stamp $T_i$, and by Remark~\ref{rem:rep}(iii) the orders they induce on $E_i$ are two linear extensions of $(E_i,\le)$. Linear extensions of a finite poset are connected by adjacent transpositions of incomparable elements, and every intermediate linear extension is again a valid order of $E_i$ (Remark~\ref{rem:rep}(iii)). Hence $\mathbf e$ is transformed into $\mathbf e'$ by finitely many adjacent transpositions of events $e_m,e_{m+1}$ with $t_m=t_{m+1}$ incomparable in $(E,\le)$, every intermediate triple $(n,\mathbf t,\mathbf e)$ being an element of $\mathrm{Rep}(\mathcal K)$. By Remark~\ref{rem:rep}(iii), each transposed pair acts on distinct simplices neither of which is a face of the other, so Lemma~\ref{lem:transp} applies to every transposition and the claim follows by transitivity.
\end{proof}

\begin{remark}\label{rem:tie-sharp}
The hypothesis that $f,f'$ are representations cannot be weakened to extended representations reaching the same complexes at their common stamps. For $K=\{\sigma\}$ a vertex, $p=0$ and stamps $T_0=0<T_1<T_2<T_3$, the sequence $(\emptyset,\{\sigma\},\emptyset)$ with stamps $T_0<T_1<T_3$ has the representation $f'=(2,(T_1,T_3),(+\sigma,-\sigma))$, with $\mathrm{PH}_0(f')=\{(T_1,T_3)\}$, while the extended representation $f=(4,(T_1,T_2,T_2,T_3),(+\sigma,-\sigma,+\sigma,-\sigma))$ of $(\emptyset,\{\sigma\},\{\sigma\},\emptyset)\in\mathrm{Seq}^*(K)$, which contains the detour $(2,3)$ at the idle stamp $T_2$, has $\mathrm{PH}_0(f)=\{(T_1,T_2),(T_2,T_3)\}$. Detours of the form $-\tau\ldots+\tau$ with $\dim\tau=p+1$, or $+\sigma\ldots-\sigma$ with $\dim\sigma=p$, are instead invisible off the diagonal (Lemma~\ref{lem:detour}); this asymmetry is the source of the dimension restrictions in Theorem~\ref{thm:lipschitz_breve}.
\end{remark}

\begin{lemma}[Dimension reduction]\label{lem:dimred}
Let $K^{(p)}:=\{\sigma\in K:\dim\sigma\in\{p,p+1\}\}$. For $f\in\mathrm{SR}^*(K)$ let $f^{(p)}$ be the sub-sequence of its events acting on simplices of $K^{(p)}$, with their times. If $f,g\in\mathrm{SR}^*(K)$ satisfy $f^{(p)}=g^{(p)}$, then $\mathrm{PH}_p(f)=\mathrm{PH}_p(g)$.
\end{lemma}
\begin{proof}
For a subcomplex $L\subseteq K$, $H_p(L)=\ker(\partial_p|_{C_p(L)})/\partial_{p+1}(C_{p+1}(L))$ depends only on $L\cap K^{(p)}$, and for $L\subseteq L'$ with $L\cap K^{(p)}=L'\cap K^{(p)}$ the map induced by the inclusion is the identity of this space. Hence every event of $f$ acting on a simplex outside $K^{(p)}$ induces an isomorphism in $H_p(f)$. If the arrow between nodes $a-1$ and $a$ of a zigzag module is an isomorphism, no interval summand has an endpoint there (the summands $\mathbb{I}_{[b,a-1]}$ and $\mathbb{I}_{[a,d]}$ would contribute to its kernel or cokernel), so deleting node $a$ and composing the two adjacent arrows (one of them, or its inverse, being an isomorphism) yields a module whose interval summands correspond bijectively to those of the original, with the same birth and death events. Deleting in this way all nodes that follow an event outside $K^{(p)}$ produces a module whose nodes are the positions of the events of $f^{(p)}$.
The same holds for $g$, so if $f^{(p)}=g^{(p)}$, then $\mathrm{PH}_p(f)=\mathrm{PH}_p(g)$.
\end{proof}

\begin{lemma}[Detour]\label{lem:detour}
Let $\mathcal K\in\mathrm{Seq}^*(K)$, $g\in\mathrm{Rep}^*(\mathcal K)$, and let $(a,b)$ be a detour of $g$ such that either $e_a=-\tau$, $e_b=+\tau$ with $\dim\tau=p+1$, or $e_a=+\sigma$, $e_b=-\sigma$ with $\dim\sigma=p$, and such that $e_{a+1},\dots,e_{b-1}$ act on simplices distinct from and face-unrelated to $\tau$ (resp.\ $\sigma$). Let $g^\circ$ be $g$ with the positions $a$ and $b$ removed. Then $g^\circ\in\mathrm{Rep}^*(\mathcal K)$ and $\mathrm{PH}_p(g)$, $\mathrm{PH}_p(g^\circ)$ coincide after removing the points on the diagonal.
\end{lemma}
\begin{proof}
Let $\widetilde K_j$ be the complexes induced by $g$. In $g^\circ$ the complexes at positions $j \in \{a,\dots,b-1\}$ are $\widetilde K_j\cup\{\tau\}$ (resp.\ $\widetilde K_j\setminus\{\sigma\}$), and those before $a$ and after $b$ are unchanged, since $\widetilde K_{a-1}=\widetilde K_a\cup\{\tau\}$ and $\widetilde K_{b-1}\cup\{\tau\}=\widetilde K_b$ (resp.\ $\widetilde K_{a-1}=\widetilde K_a\setminus\{\sigma\}$, $\widetilde K_{b-1}\setminus\{\sigma\}=\widetilde K_b$). Each such set is a complex and each step is valid, because no face of $\tau$ is deleted and no coface of $\tau$ inserted (resp.\ no coface of $\sigma$ inserted) by $e_{a+1},\dots,e_{b-1}$. Since the two removed events lie in the same tie class and cancel, the time vector of $g^\circ$ is that of $g$ without two copies of $t_a$ and the complex reached after every stamp is unchanged: $g^\circ\in\mathrm{Rep}^*(\mathcal K)$ (if the detour was the only pair of events at its stamp, that stamp is idle for $\mathcal K$ and now carries no event, which Definition~\ref{def:seq-rep}(a) allows).

By $b-a-1$ applications of Lemma~\ref{lem:transp} we move $e_a$ next to $e_b$ without leaving $\mathrm{Rep}^*(\mathcal K)$ and without changing $\mathrm{PH}_p$ off the diagonal (each transposition exchanges $e_a$ with an event of equal time acting on a distinct, face-unrelated simplex), and the sequence with the two adjacent events removed is again $g^\circ$; so we may assume $b=a+1$. Let $A:=\widetilde K_{a-1}=\widetilde K_{a+1}$. In the first case the module reads $H_p(A)\xleftarrow{\,j\,}H_p(A\setminus\{\tau\})\xrightarrow{\,j\,}H_p(A)$ with $j$ induced by the inclusion; as $\dim\tau=p+1$ we have $Z_p(A\setminus\{\tau\})=Z_p(A)$ and $B_p(A\setminus\{\tau\})\subseteq B_p(A)$, so $j$ is surjective. Choosing a complement $S$ of $\ker j$ in $H_p(A\setminus\{\tau\})$, the module is the direct sum of $\ker j\otimes\mathbb{I}_{[a,a]}$ and of the module obtained from $H_p(g^\circ)$ by inserting two nodes carrying $S\cong H_p(A)$ and $H_p(A)$ joined by isomorphisms; as in the proof of Lemma~\ref{lem:dimred}, the latter has the same interval summands as $H_p(g^\circ)$, with the same birth and death events, while $\mathbb{I}_{[a,a]}$ evaluates to $(t_a,t_{a+1})$, a point on the diagonal. In the second case the module reads $H_p(A)\xrightarrow{\,j\,}H_p(A\cup\{\sigma\})\xleftarrow{\,j\,}H_p(A)$ with $j$ injective (as $\dim\sigma=p$, $B_p(A\cup\{\sigma\})=B_p(A)$ and $Z_p(A)\subseteq Z_p(A\cup\{\sigma\})$), a complement $C$ of $\operatorname{im}j$ in $H_p(A\cup\{\sigma\})$ gives the summand $C\otimes\mathbb{I}_{[a,a]}$, and the argument is the same.
\end{proof}

\subsection{Proofs of Section \ref{sec:diff_pd}}

\begin{proof}[Proof of Proposition \ref{prop:permutation}]
Let $g=(n',\mathbf{t}',\mathbf{e}')\in\mathrm{cell}(f)$. 
Since $g\in\mathrm{cell}(f)$ means the simplex-wise zigzag filtrations $(n',\mathbf{e}')$, $(n,\mathbf{e})$ are the same, it follows automatically $P_p^g = P_p^f =: P$.
Since $P$ does not depend on the choice of $g \in \mathrm{cell}(f)$, the map $\widetilde{\mathrm{PH}}_{p,g}$ is the same for every $g\in\mathrm{cell}(f)$ and is by definition a permuted coordinate projection.

By the continuous evaluation step in the computation of $\mathrm{PH}_{p}(g)$, together with $P_p^g = P$:
$$\mathrm{PH}_p(g) = \{(t'_{i},t'_{j})\}_{(i,j)\in P}.$$
On the other hand, applying $Q_{|P|}$ to $\widetilde{\mathrm{PH}}_{p,f}(\mathbf{t}') = (t'_{a_1},\dots,t'_{a_{2|P|}})$ and recalling $P=((a_1,a_2),\dots,(a_{2|P|-1},a_{2|P|}))$, Definition \ref{def:quotient} gives exactly
$$Q_{|P|}\big(\widetilde{\mathrm{PH}}_{p,f}(\mathbf{t}')\big) = \{(t'_{a_{2\ell-1}},t'_{a_{2\ell}})\}_{\ell=1}^{|P|} = \mathrm{PH}_p(g).$$
\end{proof}

\begin{proof}[Proof of Corollary \ref{cor:condition1}]
By hypothesis, $S(\theta') \in \mathrm{cell}(S(\theta))$ for every $\theta' \in U$, so Proposition \ref{prop:permutation} applies with $f = S(\theta)$: writing $m := |P_p^{S(\theta)}|$, for every $\theta' \in U$ we have
$$Q_{m}\big(\widetilde{\mathrm{PH}}_{p,S(\theta)}(\mathrm{temp}_{S(\theta)}(S(\theta')))\big) = \mathrm{PH}_p(S(\theta')) = \mathcal{P}_p(\theta').$$
The left-hand side equals $Q_{m}(\widetilde{B}(\theta'))$ with $\widetilde{B} := \widetilde{\mathrm{PH}}_{p,S(\theta)} \circ \mathrm{temp}_{S(\theta)} \circ S|_U$, thus $\mathcal{P}_p|_U = Q_{m} \circ \widetilde{B}$, i.e.\ $\widetilde{B}$ is a local lift of $\mathcal{P}_p$ at $\theta$ in the sense of Definition \ref{def:differentiability}.

Suppose in addition $\widetilde{S} := \mathrm{temp}_{S(\theta)} \circ S|_U$ is $C^r$. Since $\widetilde{\mathrm{PH}}_{p,S(\theta)}$ is a linear (hence $C^\infty$) map, the composite $\widetilde{B} = \widetilde{\mathrm{PH}}_{p,S(\theta)} \circ \widetilde{S}$ is $C^r$ on $U$, so $\mathcal{P}_p$ is $r$-differentiable at $\theta$ by Definition \ref{def:differentiability}. If moreover $r \ge 1$, the classical chain rule applied to $\widetilde{B} = \widetilde{\mathrm{PH}}_{p,S(\theta)} \circ \widetilde{S}$ gives
$$\mathrm{d}_\theta \widetilde{B} = \mathrm{d}_{\widetilde{S}(\theta)}\widetilde{\mathrm{PH}}_{p,S(\theta)} \circ \mathrm{d}_\theta \widetilde{S},$$
which is precisely $\mathrm{d}_{\theta,\widetilde{B}}\mathcal{P}_p$ by definition of the differential relative to the lift $\widetilde{B}$ (Definition \ref{def:differentiability}). \qedhere
\end{proof}

\subsection{Proofs of Section \ref{sec:lin_int}}

We formalize the map $S:FF_N(K)\rightarrow\mathrm{Seq}(K)$ introduced in Section~\ref{sec:lin_int} and prove that it is well defined in Lemma~\ref{lem:S-valid}. Not every crossing of Definition~\ref{def:crossing} produces an event: the following configuration is excluded.

\begin{definition}[Phantom crossings]\label{def:phantom}
Let $v\in FF_N(K)$, $\sigma\in K$ and $k\in\{2,\dots,N-1\}$ be such that $v(\sigma)_k=\epsilon$ and $v(\sigma)_{k-1},v(\sigma)_{k+1}>\epsilon$. Then $(\sigma,k-1)$ is an exit crossing and $(\sigma,k)$ an entry crossing, with $t^{(k-1)}_\sigma(v)=t^{(k)}_\sigma(v)=k$ by \eqref{eq1}; we call them \emph{phantom crossings}. Every other crossing of $v$ is called \emph{genuine}. 
\end{definition}

\begin{lemma}[Well-definedness of $S$]\label{lem:S-valid}
Let $v\in FF_N(K)$ and, for $\sigma\in K$, let $\hat v(\sigma):[0,\infty)\to\mathbb R$ be the piecewise-linear interpolation of $(v(\sigma)_0,\dots,v(\sigma)_{N+1})$, extended by $0$ after $N+1$. For $t\in[0,N+1]$ put
$$K_t(v):=\{\sigma\in K:\ \hat v(\sigma)>\epsilon\ \text{on}\ (t,t+\eta)\ \text{for some}\ \eta>0\}.$$
Then:
\begin{enumerate}
    \item[(i)] every $K_t(v)$ is a subcomplex of $K$;
    \item[(ii)] $t\mapsto K_t(v)$ is right-continuous and piecewise constant, with finitely many jumps $0<T_1<\dots<T_m<N+1$, and $K_t(v)=\emptyset$ for $t\in[0,T_1)\cup[T_m,N+1]$ (if there are no jumps, $m=0$ and $K_t(v)=\emptyset$ for every $t$);
    \item[(iii)] the jumps are exactly the genuine crossings of $v$: $\sigma$ enters $K_t(v)$ at $t=t^{(k)}_\sigma(v)$ for every genuine entry crossing $(\sigma,k)$, leaves it at $t=t^{(k)}_\sigma(v)$ for every genuine exit crossing $(\sigma,k)$, and its status does not change at any other time.
\end{enumerate}
Consequently, setting $T_0:=0$, the map $S:FF_N(K)\rightarrow\mathrm{Seq}(K)$ defined by
$$S(v):=\{K_{T_i}(v)\}_{i=0}^{m}$$
is well-defined.
\end{lemma}

\begin{proof}
(i) For $\sigma\subsetneq\tau$ we have $v(\sigma)_i\ge v(\tau)_i$ for all $i$, hence $\hat v(\sigma)\ge\hat v(\tau)$ pointwise, so $\tau\in K_t(v)$ implies $\sigma\in K_t(v)$.

(ii) For every $\sigma$, $\{t:\hat v(\sigma)(t)>\epsilon\}$ is a finite union of open intervals, and $\sigma\in K_t(v)$ iff $t$ is an interior point or the left endpoint of one of them; so $t\mapsto K_t(v)$ is right-continuous and changes only at the finitely many endpoints of these intervals. Let $T_1<\dots<T_m$ be the times at which $K_t(v)$ changes. $K_0(v)=\emptyset$ since $\hat v(\sigma)(0)=0<\epsilon$ and $\hat v$ is continuous, so $K_t(v)=\emptyset$ on $[0,T_1)$; $K_{N+1}(v)=\emptyset$ by the extension, and since $K_t(v)$ is constant on $[T_m,N+1]$ it is empty there.

(iii) On $(k,k+1)$ the function $\hat v(\sigma)$ is affine, so $\sigma$ changes status at an interior point $t$ iff $v(\sigma)_k$ and $v(\sigma)_{k+1}$ lie strictly on opposite sides of $\epsilon$, and then $t=t^{(k)}_\sigma(v)$ with the corresponding type; such a crossing is genuine, as both its values differ from $\epsilon$. At an integer $k\in\{1,\dots,N\}$ with $v(\sigma)_k\ne\epsilon$ nothing changes. If $v(\sigma)_k=\epsilon$, then $\sigma\in K_k(v)$ iff $v(\sigma)_{k+1}>\epsilon$, while $\sigma$ is present just before $k$ iff $v(\sigma)_{k-1}>\epsilon$. Hence $\sigma$ enters at $k$ iff $v(\sigma)_{k-1}\le\epsilon<v(\sigma)_{k+1}$, which is the genuine entry crossing $(\sigma,k)$ with $t^{(k)}_\sigma(v)=k$ (and $(\sigma,k-1)$ is not a crossing); $\sigma$ leaves at $k$ iff $v(\sigma)_{k-1}>\epsilon\ge v(\sigma)_{k+1}$, which is the genuine exit crossing $(\sigma,k-1)$ with $t^{(k-1)}_\sigma(v)=k$ (and $(\sigma,k)$ is not a crossing); and $\sigma$ does not change status at $k$ iff both neighbours are $>\epsilon$, in which case $(\sigma,k-1),(\sigma,k)$ are the phantom crossings of Definition~\ref{def:phantom}, or both are $\le\epsilon$, in which case neither is a crossing. This accounts for all crossings and proves (iii).

By (i) the $K_{T_i}(v)$ are subcomplexes; by (ii) $K_{T_0}(v)=K_{T_m}(v)=\emptyset$, and consecutive complexes differ by definition of the jumps; hence $S(v)\in\mathrm{Seq}(K)$ (the trivial sequence if $m=0$).
\end{proof}

\begin{remark}[Phantom crossings and events]\label{obs:phantom}
By Lemma~\ref{lem:S-valid}(iii), the crossings of $v$ that produce no event of $S(v)$ are exactly the phantom crossings: the exit at time $k$ and the re-entry at time $k$ of Definition~\ref{def:phantom} cancel, so $\sigma$ contributes no event on $(k-1,k+1)$ and the pair $(\sigma,k-1),(\sigma,k)$ is excluded when listing the $n$ events of $S(v)=(n,\mathbf t,\mathbf e)$.

Moreover, two genuine crossings of the same simplex never have the same time: consecutive crossings of $\sigma$ lie in distinct intervals $[k,k+1]$, $[k',k'+1]$, an entry crossing never occurs at the right endpoint of its interval and an exit crossing never at the left endpoint, so a coincidence forces $k'=k+1$ with both times equal to $k+1$, i.e.\ an exit at $k+1$ followed by an entry at $k+1$: the phantom configuration. In particular a tie class of $S(v)$ contains at most one event of each simplex, and $S(v)$ has no detours.
\end{remark}

\begin{remark}\label{rem:phantom}
If $v$ has phantom crossings at $[k-1,k]$ and $[k,k+1]$ for $\sigma$, the same holds for every $w\in\mathrm{cell}(v)$: $\mathrm{sgn}(w(\sigma)_k-\epsilon)=\mathrm{sgn}(v(\sigma)_k-\epsilon)=0$ and $\mathrm{sgn}(w(\sigma)_{k\pm1}-\epsilon)=\mathrm{sgn}(v(\sigma)_{k\pm1}-\epsilon)=1$ for every such $w$.
\end{remark}

Since $FF_N(K)$ is specified by the finite set of non-strict linear inequalities \eqref{eq:ff_N}, it is a closed convex polyhedral cone in $(\mathbb R^{N|K|})^*$ with non-empty interior.

\begin{proposition}\label{prop:cell-basic}
Let $v\in FF_N(K)$. $\mathrm{cell}(v)$ is a semialgebraic subset of $\mathbb R^{N|K|}$, and the partition $\{\mathrm{cell}(v)\}_{v\in FF_N(K)}$ of $FF_N(K)$ is finite.
\end{proposition}

\begin{proof}
The sign conditions $\mathrm{sgn}(w(\sigma)_i-\epsilon)=\mathrm{sgn}(v(\sigma)_i-\epsilon)$ are linear
equalities or strict linear inequalities. On the set where they hold, the crossings of $w$ and
their types are those of $v$, each denominator $w(\sigma)_{k+1}-w(\sigma)_k$ of a crossing has a
constant sign, and multiplying an order relation $t^{(k)}_\sigma(w)\lessgtr t^{(k)}_\tau(w)$, or an
equality, by the product of the two denominators turns it into a polynomial inequality or
equality of degree at most two. Together with the linear inequalities defining $FF_N(K)$ this
exhibits $\mathrm{cell}(v)$ as a semialgebraic set. A cell is determined by the sign pattern
$(\mathrm{sgn}(w(\sigma)_i-\epsilon))_{\sigma,i}\in\{-1,0,1\}^{N|K|}$ and, for each $k\in\{0,\dots,N\}$, by a
total preorder on the finite set of simplices crossing $\epsilon$ on $[k,k+1]$; there are finitely
many such data.
\end{proof}

\begin{proof}[Proof of Proposition \ref{prop:same_comb}]
Let $S(v)=\left( n, \, \mathbf{t}, \, \mathbf{e}\right)$ and $S(w)=\left( n', \, \mathbf{t}', \, \mathbf{e}'\right)$.
Fix $\sigma \in K$ crossing $\epsilon$, say on $[k,k+1]$ for $v$. Since $w\in\mathrm{cell}(v)$, $\mathrm{sgn}(w(\sigma)_i - \epsilon) = \mathrm{sgn}(v(\sigma)_i-\epsilon)$ for every $i$; as a consequence, $\sigma$ crosses $\epsilon$ on $[k,k+1]$ for $w$ as well, with the same entry/exit type. Hence the set $C_v:=\{(\sigma,k,type) : \sigma \text{ crosses } \epsilon \text{ on } [k,k+1] \text{ with type } type \text{ for } v\}$ coincides with $C_w$.

By Remark~\ref{rem:phantom}, a pair $(\sigma,k-1,\mathrm{exit}),(\sigma,k,\mathrm{entry})\in C_v$ is phantom precisely when $v(\sigma)_{k-1},v(\sigma)_{k+1}>\epsilon$ and $v(\sigma)_k=\epsilon$, a condition entirely determined by signs preserved on $\mathrm{cell}(v)$; hence the phantom subset $\mathrm{Phantom}(v)\subseteq C_v$ coincides with $\mathrm{Phantom}(w)\subseteq C_w$ as a set of triples. Writing $C_v^\circ := C_v\setminus\mathrm{Phantom}(v)$ for the non-phantom crossings — which by definition of $S$ are exactly the $n$ events of $S(v)$ — we thus have $C_v^\circ = C_w^\circ$, and in particular $n'=n$. Also by definition of $S$, $\sigma$ enters/exits from the filtrations $S(v), S(w)$, for each genuine crossing, at times $t_\sigma^{(k)}(v)$ and $t_\sigma^{(k)}(w)$, respectively.

Since $w\in\mathrm{cell}(v)$, for every $\sigma, \tau \in K$ crossing $\epsilon$ on $[k,k+1]$
we have $t_\sigma^{(k)}(w) < t_\tau^{(k)}(w) \iff t_\sigma^{(k)}(v) < t_\tau^{(k)}(v)$ and $
t_\sigma^{(k)}(w) = t_\tau^{(k)}(w) \iff t_\sigma^{(k)}(v) = t_\tau^{(k)}(v)$.
Paired with the fact that the simplices crossing $\epsilon$ on $[k,k+1]$ for $v$ and $w$ are the same, this implies that the order of the events of $S(v)$ and $S(w)$ happening in $[k,k+1]$ is the same.

It remains to compare events with same interpolation time and different crossing intervals. Let $\sigma$ cross $\epsilon$ on $[k,k+1]$ and $\tau$ cross $\epsilon$ on $[k+1,k+2]$, both giving rise to genuine (non-phantom) events. By definition, $t^{(k)}_\sigma(v) = k+1$ if and only if $v(\sigma)_{k+1}=\epsilon$, i.e.\ $\mathrm{sgn}(v(\sigma)_{k+1}-\epsilon)=0$; since $w\in\mathrm{cell}(v)$ preserves this sign, $t^{(k)}_\sigma(w)=k+1$ if and only if $t^{(k)}_\sigma(v)=k+1$. The same holds for $\tau$ on the shared boundary: $t^{(k+1)}_\tau(v)=k+1 \iff \mathrm{sgn}(v(\tau)_{k+1}-\epsilon)=0 \iff t^{(k+1)}_\tau(w)=k+1$. Hence $t^{(k)}_\sigma(w)=t^{(k+1)}_\tau(w) \iff t^{(k)}_\sigma(v)=t^{(k+1)}_\tau(v)$, so ties across adjacent intervals are preserved from $v$ to $w$. Since events crossing on non-adjacent intervals necessarily have distinct interpolated times (their intervals are disjoint and share no endpoint), no further coincidences can occur.

Combining this with the within-interval order preservation above, the global order and tie pattern of all $n$ genuine events coincide for $v$ and $w$. Since ties in $\mathrm{SR}(K;<)$ are systematically resolved by the fixed total order $(E, <)$, this implies that $\mathbf e' = \mathbf e$ and $t'_i = t^{I(i)}_{|e_i|}(w)$ for $i=1,\dots,n$. 
Consequently, $I$ is an indexing function of $w$.
\end{proof}

\begin{proof}[Proof of Corollary \ref{cor:condition2}]
Write $v:=G(\theta)$. By hypothesis $G(\theta')\in\mathrm{cell}(v)$ for all $\theta'\in U$, hence
$S(G(\theta'))\in\mathrm{cell}(S(v))$ by Proposition~\ref{prop:same_comb}: $S\circ G$ satisfies the
hypothesis of Corollary~\ref{cor:condition1} at $\theta$, with the roles of $S$ there played by
$S\circ G$ here, and the local lift appearing in Corollary~\ref{cor:condition1} is
$$\mathrm{temp}_{S(v)}\circ(S\circ G)|_U=\mathrm{temp}_{S(v)}\circ S|_{\mathrm{cell}(v)}\circ G|_U=\widetilde S_v|_{\mathrm{cell}(v)}\circ G|_U.$$ Now $G|_U$ is a $C^r$ map from the open set $U$ into
$\{0\}^{|K|}\times\mathbb R^{N|K|}\times\{0\}^{|K|}$ with image in $\mathrm{cell}(v)\subseteq W_v$, and $\widetilde S_v$ is $C^\infty$ on the open set
$W_v$; by the classical chain rule for maps between open subsets of Euclidean spaces,
$\widetilde S_v\circ G|_U:U\to\mathbb R^n$ is of class $C^r$, with
$\mathrm d_\theta(\widetilde S_v\circ G|_U)=\mathrm d_{v}\widetilde S_v\circ\mathrm d_\theta G|_U$.
Applying Corollary~\ref{cor:condition1} to $S\circ G$ yields that
$\widetilde B=\widetilde{\mathrm{PH}}_{p,S(v)}\circ\widetilde S_v\circ G|_U$ is a $C^r$ local lift of $\mathcal{P}_p$ at $\theta$, that
$\mathcal{P}_p$ is $r$-differentiable at $\theta$, and that
$$\mathrm d_{\theta,\widetilde B}\mathcal{P}_p=\mathrm d_{\widetilde S_v(v)}\widetilde{\mathrm{PH}}_{p,S(v)}\circ\mathrm d_\theta(\widetilde S_v\circ G|_U)=\mathrm d_{\widetilde S_v(v)}\widetilde{\mathrm{PH}}_{p,S(v)}\circ\mathrm d_{v}\widetilde S_v\circ\mathrm d_\theta G|_U.$$
\end{proof}

\begin{remark}[Index set generalization]\label{rmk:indexes}
In the paper's construction, we fixed the convention where sequences of complexes are indiced by time stamps starting from $0$, with the first and last complex empty (Definition~\ref{def:seq_compl}), and a filtering function $f\in FF_N(K)$ is defined over the index set $\{0,\dots,N+1\}$, with the first and last value zero (Section~\ref{sec:lin_int}), so that $S(f)$ is indeed a sequence in this sense.
None of the results, however, use the integrality of these indices, only their order and the linear interpolation between consecutive ones; so, when the filtering function is more naturally indexed by fixed time stamps $a_1<\dots<a_N$, we may instead fix any $a_0<a_1$ and $a_{N+1}>a_N$, impose $f(\cdot)_{a_0}=f(\cdot)_{a_{N+1}}=0$, and consider filtering functions over $a_0<\dots<a_{N+1}$, and assume sequences of complexes start from $a_0$: the paper's construction then applies verbatim, provided the interpolation formula (Eq.~\ref{eq1}) is replaced, for $\sigma$ crossing $\epsilon$ between $a_k$ and $a_{k+1}$, by
$$t_\sigma^{(k)}(v) := a_k + (a_{k+1}-a_k)\,\frac{\epsilon - v(\sigma)_{a_k}}{v(\sigma)_{a_{k+1}} - v(\sigma)_{a_k}}.$$
We refer to this as adopting the index convention $\{a_0,\dots,a_{N+1}\}$ in place of $\{0,\dots,N+1\}$.
\end{remark}

\subsection{Proofs of Section \ref{sec:stability}}

We fix an o-minimal structure $\mathcal{O}$ on $(\mathbb R,+,\cdot,<)$ in the sense of
Definition~\ref{def:o_minimal}; ``definable'' means definable in $\mathcal O$. We use the following facts, for which \cite[\S4]{vandendries1996geometric} gives statements and references to \cite{vandendries1998tame} for proofs; $r$ denotes a positive integer.
\begin{itemize}
\item[(O1)] (\cite[2.1]{vandendries1996geometric}) Images and preimages of definable sets under
definable maps are definable; compositions of definable maps are definable; a map is definable if
and only if its coordinates are.
\item[(O2)] ($C^r$ cell decomposition, \cite[4.2]{vandendries1996geometric}) Given definable
$A_1,\dots,A_k\subseteq\mathbb R^d$ and a definable $f:\mathbb R^d\to\mathbb R$, there is a finite partition
$\mathcal P$ of $\mathbb R^d$ into definable $C^r$ cells, each a connected $C^r$ embedded submanifold
of $\mathbb R^d$, compatible with $\{A_1,\dots,A_k\}$ (each cell is contained in or disjoint from
each $A_j$) and such that $f|_P$ is $C^r$ for every $P\in\mathcal P$. Applying this successively to
the coordinates of a definable $G:\mathbb R^d\to\mathbb R^m$, refining at each step, the same holds
with $G|_P$ of class $C^r$.
\item[(O3)] (dimension, \cite[4.7]{vandendries1996geometric}) For definable $A\subseteq\mathbb R^d$,
$\dim(\overline A\setminus A)<\dim A$; $\dim A=d$ if and only if $A$ has non-empty interior; a
definable set of dimension $<d$ is a finite union of $C^1$ cells of dimension $<d$, hence has
Lebesgue measure zero.
\item[(O4)] (Whitney stratification of sets, \cite[4.8(1)]{vandendries1996geometric},
\cite{loi1998verdier}) Given definable $A_1,\dots,A_k\subseteq\mathbb R^d$, there is a finite definable
$C^r$ Whitney stratification of $\mathbb R^d$ compatible with $\{A_1,\dots,A_k\}$, with each stratum a
$C^r$ cell.
\end{itemize}

Concerning (O4): the published proof of the $C^r$ Whitney stratification theorem in
\cite{vandendries1996geometric} has a gap, acknowledged by the authors, which they consider
reparable for the stratification of \emph{sets} (part (1) of their 4.8) but not settled for the
stratification of \emph{maps} (part (2)); an independent proof of the stratification of sets,
through the stronger Verdier condition, is given in \cite{loi1998verdier}. We only use part (1),
together with (O2). Concerning the differentiability class: (O2) and (O4) hold for every finite
$r$ in every o-minimal structure, but $r=\infty$ cannot be taken in general
\cite{legal2009ominimal}. They do hold with $r=\omega$ (analytic cells and strata) when $\mathcal O$
has analytic decomposition, which is the case for the semialgebraic structure, for
$\mathbb R_{an}$ and for $\mathbb R_{an,exp}$ \cite[4.2, 5.1(3)]{vandendries1996geometric};
the last one contains every map built from polynomials, $\exp$, $\log$, the restricted analytic
functions and $\max$, hence all the network parametrizations $G$ used in this paper.

\begin{proposition}\label{prop:definable-stratified}
Let $r$ be a positive integer, or $r=\omega$ if $\mathcal O$ has %
analytic decomposition. Let $M\subseteq\mathbb R^m$ be definable, let $\mathcal C_M=\{C_1,\dots,C_s\}$ be a
finite partition of $M$ into definable sets, and let $G:\mathbb R^d\to M$ be definable. Then there is a
Whitney $C^r$-stratification $\mathcal S$ of $\mathbb R^d$ with definable strata
such that:
\begin{enumerate}
\item[(i)] $G$ is weakly stratified with respect to $(\mathcal S,\mathcal C_M)$;
\item[(ii)] $G|_S$ is of class $C^r$ (as a map into $\mathbb R^m$) for every stratum $S$;
\item[(iii)] the union $\Omega$ of the top-dimensional strata is an open dense definable subset
of $\mathbb R^d$, and $Z:=\mathbb R^d\setminus\Omega$ is a definable set of dimension $<d$, hence of
Lebesgue measure zero.
\end{enumerate}
\end{proposition}
\begin{proof}
The sets $G^{-1}(C_1),\dots,G^{-1}(C_s)$ are definable by (O1) and partition $\mathbb R^d$. By (O2)
there is a finite $C^r$ cell decomposition $\mathcal P$ of $\mathbb R^d$ compatible with them and such
that $G|_P$ is $C^r$ for every $P\in\mathcal P$. By (O4) there is a finite $C^r$ Whitney stratification
$\mathcal S$ of $\mathbb R^d$ compatible with the finitely many cells of $\mathcal P$. Every stratum $S$
is contained in a single cell $P$, hence in a single $G^{-1}(C_j)$, which gives (i), and $G|_S$ is
the restriction of the $C^r$ map $G|_P$ to the $C^r$ submanifold $S\subseteq P$, which gives (ii).
For (iii), a top-dimensional stratum is a $d$-dimensional $C^1$ submanifold of $\mathbb R^d$, hence
open; so $\Omega$ is open and definable, and $Z$ is the finite union of the strata of dimension
$<d$, a definable set of dimension $<d$ by (O3), hence Lebesgue-null; $Z$ has empty interior by
(O3), so $\Omega$ is dense.
\end{proof}

Consider the setup of Section \ref{sec:stability}, where we have the maps $\mathcal{P}_p: M \xrightarrow{G} FF_N(K) \xrightarrow{S}\mathrm{Seq}(K) \xrightarrow{\mathrm{PH}_p} \mathcal{D}$, $L:\mathcal D\rightarrow \mathbb{R}$, and their composition $\mathcal{L}=L\circ \mathcal{P}_p$, and use Proposition \ref{prop:definable-stratified} to prove Proposition~\ref{prop:ae-differentiable}

\begin{proof}[Proof of Proposition \ref{prop:ae-differentiable}]
Apply Proposition~\ref{prop:definable-stratified} to $M=FF_N(K)$, which is semialgebraic,
with $\mathcal C_M=\{\mathrm{cell}(v)\}_{v\in FF_N(K)}$, a finite partition into semialgebraic sets (Proposition~\ref{prop:cell-basic}), hence definable in every o-minimal
structure. For $\theta\notin Z$ let $U$ be the top-dimensional stratum containing $\theta$: it is
open, $G(U)$ is contained in a single cell, necessarily $\mathrm{cell}(G(\theta))$, and $G|_U$ is
analytical. The remaining assertions are Corollary~\ref{cor:condition2} with this $U$ and the chain rule
of Proposition~\ref{prop:chain}.
\end{proof}

In the following proposition, we show the geometric requirement of definability of $\mathcal{L}$ can be easily obtained from standard definability assumptions on $G$ and $L\circ Q_{m}$.

\begin{proposition}\label{prop:definable}
If for every $m\in\mathbb{N}$ the maps $L\circ Q_{m}:\mathbb{R}^{2m}\rightarrow \mathbb{R}$ are definable over a common o-minimal structure $\mathcal O$ (expanding the real field), and $G: \mathbb{R}^d \rightarrow FF_N(K)$ is definable over $\mathcal O$, then $\mathcal{L}$ is definable over $\mathcal O$.
\end{proposition}

\begin{proof}
Recall $FF_N(K)$ is partitioned into finitely many cells $\{\mathrm{cell}(v)\}_{v}$, each semialgebraic (Proposition~\ref{prop:cell-basic}). Fix $v\in FF_N(K)$, and let $I$ an indexing function.
On $\mathrm{cell}(v)$, $\widetilde S_v|_{\mathrm{cell}(v)}(w)=(t^{I(1)}_{|e_1|}(w),\dots,t^{I(n)}_{|e_n|}(w))$ is a tuple of ratios of affine functions with nowhere-vanishing denominators, hence semialgebraic. Since every o-minimal structure expanding the real field contains all semialgebraic sets, $\widetilde S_v|_{\mathrm{cell}(v)}$ is definable over $\mathcal O$ regardless of hypothesis. As $\widetilde{\mathrm{PH}}_{p,S(v)}$ is linear, it too is definable over $\mathcal O$. 
Let $m:=|P^{S(v)}_p|$.
By hypothesis, $L\circ Q_{m}$ is definable over $\mathcal O$; composing definable maps over the same structure yields
$$L\circ \mathrm{PH}_p\circ S|_{\mathrm{cell}(v)} = (L\circ Q_{m})\circ\widetilde{\mathrm{PH}}_{p,S(v)}\circ\widetilde S_v|_{\mathrm{cell}(v)}$$
definable over $\mathcal O$ (using that $Q_{m}(\widetilde{\mathrm{PH}}_{p,S(v)}(\widetilde S_v(w)))=\mathrm{PH}_p(S(w))$ for every $w\in\mathrm{cell}(v)$, by Proposition~\ref{prop:permutation} and $\widetilde{S}$ definition). Since $L\circ\mathrm{PH}_p\circ S$ restricted to each of the finitely many definable pieces $\mathrm{cell}(v)$ is definable over $\mathcal O$, and a finite union of definable graphs is definable, $L\circ\mathrm{PH}_p\circ S$ is definable over $\mathcal O$ on all of $FF_N(K)$. Composing with $G$ (definable over $\mathcal O$ by hypothesis) shows $\mathcal L = (L\circ\mathrm{PH}_p\circ S)\circ G$ is definable over $\mathcal O$.
\end{proof}

\subsection{Subgradient computation}
\label{sec:subgrad_impl}

We sketch how to compute, in practice, an element of the Clarke subgradient $\partial \mathcal L(\theta)$ (Definition~\ref{def:clark_sub}) of
$$\mathcal{L}: M \xrightarrow{G} FF_N(K) \xrightarrow{S}\mathrm{Seq}(K) \xrightarrow{\mathrm{PH}_p} \mathcal{D} \xrightarrow{L}\mathbb{R},$$
via the rule of Proposition~\ref{prop:ae-differentiable}, using standard backpropagation.

\paragraph{Two-phase evaluation.} Fix $\theta\in M$ and set $v:=G(\theta)\in FF_N(K)$ (a forward pass through the parametric model, tracked by autograd). Evaluating $\widetilde{B}(\theta)=\widetilde{\mathrm{PH}}_{p,S(v)}(\widetilde S_v(v))$ splits into two phases of different nature:
\begin{enumerate}
    \item \textit{Combinatorial phase (no gradient).} From the values of $v$ alone, a standard persistence pairing routine determines the crossings (Definition~\ref{def:crossing}), the induced event sequence $\mathbf e$, and the pairing set $P_p^{S(v)}$ together with an indexing function $I$ (Definition~\ref{def:indexing}). This step is purely combinatorial, so it is run with gradient tracking disabled.
    \item \textit{Temporal phase (differentiable).} Using the crossing pairs $(\sigma,k)=(|e_i|,I(i))$ found above, we \emph{recompute} the corresponding crossing times $t^{(k)}_\sigma(v)$ directly from $v$ via the rational formula of Equation~\ref{eq1} — this time with gradient tracking enabled, since $\widetilde S_v$ (Equation~\ref{eq:S_tilde}) is $C^\infty$ on $W_v\supseteq\mathrm{cell}(v)$. Gathering these values into pairs according to $P_p^{S(v)}$ gives $\widetilde{B}(\theta)=\widetilde{\mathrm{PH}}_{p,S(v)}(\widetilde S_v(v))$; applying $Q_{|P_p^{S(v)}|}$ and then $L$ yields $\mathcal L(\theta) = L\big(Q_{|P_p^{S(v)}|}(\widetilde B(\theta))\big)$.
\end{enumerate}
Backpropagating through phase 2 alone yields $\nabla_\theta \mathcal L(\theta)$, matching the chain rule of Proposition~\ref{prop:chain} exactly: differentiability is required, and used, only for the map $\widetilde S_v$ and the composition with $L$, never for the discrete pairing step. This mirrors standard practice in differentiable persistent homology, where the pairing is treated as a fixed (non-differentiable) index map and only the real-valued filtration values are backpropagated through.

\paragraph{Handling non-differentiability points via perturbation.} By Proposition~\ref{prop:ae-differentiable}, $\mathcal L$ fails to be differentiable at $\theta$ only on a Lebesgue-null set $Z\subseteq\mathbb R^d$. We exploit this by injecting a small Gaussian perturbation $\eta_k$ into every iterate $\theta_{k}$ (including the initial value), on top of the noise terms $\zeta_k$ already present in the scheme of Definition~\ref{def:clark}: since $Z$ is null and $\eta_k$ is absolutely continuous, $\theta_{k+1}=\theta_k - \gamma_k (g_k + \zeta_k+\eta_k)$ lands in $\mathbb R^d\setminus Z$ almost surely, regardless of the specific structure of the parametrization $G$. At such points $\mathcal L$ is differentiable, so the gradient computed via the two-phase evaluation above is an element of the Clarke subgradient.

In our experiments we combined this algorithmic strategy — backpropagation as above, together with parameter perturbation — to compute an element of the subgradient of $\mathcal{L}$.

\subsection{Theorem \ref{thm:lipschitz_breve}}
\label{sec:lip_theorem}

\begin{definition}\label{def:critical_set}
Let $p \in \mathbb{N}$, $N\in\mathbb N$, $K$ a finite simplicial complex, and define
$$
\mathcal N^{\epsilon+}_{p}:=\left\{w\in FF_N(K)\ \middle|\
\begin{array}{l}
\exists\,\sigma\in K,\ \dim\sigma=p,\ i\in\{2,\dots,N-1\}\ \text{such that}\\[2pt]
w(\sigma)_i=\epsilon\ \text{and}\ w(\sigma)_{i-1},\,w(\sigma)_{i+1}>\epsilon
\end{array}
\right\},
$$
$$
\mathcal N^{\epsilon-}_{p+1}:=\left\{w\in FF_N(K)\ \middle|\
\begin{array}{l}
\exists\,\sigma\in K,\ \dim\sigma=p+1,\ i\in\{1,\dots,N\}\ \text{such that}\\[2pt]
w(\sigma)_i=\epsilon\ \text{and}\ w(\sigma)_{i-1},\,w(\sigma)_{i+1}<\epsilon
\end{array}
\right\},
$$
$$
\mathcal N^{\epsilon,\epsilon}_{p,p+1}:=\left\{w\in FF_N(K)\ \middle|\
\begin{array}{l}
\exists\,\sigma\in K,\ \dim\sigma\in\{p,p+1\},\ i\in\{1,\dots,N-1\}\ \text{such that}\\[2pt]
w(\sigma)_i=w(\sigma)_{i+1}=\epsilon
\end{array}
\right\}.
$$
Let $\mathcal N_p:=\mathcal N^{\epsilon+}_{p}\cup\mathcal N^{\epsilon-}_{p+1}\cup\mathcal N^{\epsilon,\epsilon}_{p,p+1}$.
\end{definition}

\begin{proof}[Proof of Theorem \ref{thm:lipschitz_breve}]
\emph{Conventions.} $\|\cdot\|_\infty$ is the sup norm on the free coordinates $w(\sigma)_i$,
$\sigma\in K$, $1\le i\le N$; the sentinels $w(\sigma)_0=w(\sigma)_{N+1}=0$ are constant. The
bottleneck distance allows matching a point to the diagonal, so $d_B(D,D')=0$ whenever $D,D'$
differ only by points on the diagonal; $d_B$ is a pseudo-metric on $\mathcal D$, and for every
$m$ the map $Q_m:(\mathbb R^{2m},\|\cdot\|_\infty)\to(\mathcal D,d_B)$ is $1$-Lipschitz (match the
$\ell$-th point of $Q_m(x)$ with the $\ell$-th point of $Q_m(y)$). For $u\in FF_N(K)$ we identify the
sequence $S(u)\in\mathrm{Seq}(K)$ with its representation in $\mathrm{SR}(K;<)$ (Lemma~\ref{lem:S-valid}),
whose events are the genuine crossings of $u$ with their crossing times. Proposition~\ref{prop:permutation},
Lemma~\ref{lem:transp} and Lemma~\ref{lem:detour} are used on $\mathrm{SR}^*(K)$ (Remark~\ref{rem:rep}(iv)),
Lemma~\ref{lem:tie} on representations (Definition~\ref{def:seq-rep}). We write
$K^{(p)}:=\{\sigma\in K:\dim\sigma\in\{p,p+1\}\}$.

\medskip\noindent\textbf{Reduction.}
Fix $v\in FF_N(K)\setminus\mathcal N_p$. Put $M:=\epsilon+2+\max\{|v(\sigma)_i|:\sigma\in K^{(p)},\,1\le i\le N\}$
and define $F:FF_N(K)\cap B_\infty(v,1)\to\{0\}^{|K|}\times\mathbb R^{N|K|}\times\{0\}^{|K|}$ by
$F(u)(\sigma)_i:=u(\sigma)_i$ if $\sigma\in K^{(p)}$, $:=M$ if $\dim\sigma<p$, $:=-M$ if
$\dim\sigma>p+1$ ($1\le i\le N$; sentinels $0$). For $\|u-v\|_\infty<1$ one has $|u(\sigma)_i|<M$
for $\sigma\in K^{(p)}$, so $F(u)$ satisfies the face inequalities of \eqref{eq:ff_N} for every pair
$\mu\subsetneq\sigma$ (a face of dimension $<p$ has value $M$, larger than every other value; a
simplex of dimension $>p+1$ has value $-M$, smaller than every other value; for
$\mu,\sigma\in K^{(p)}$ the inequality is that of $u$): hence $F(u)\in FF_N(K)$, and $F$ is
$1$-Lipschitz. The events of $S(F(u))$ acting on $K^{(p)}$, with their times and order, coincide
with those of $S(u)$ (they are computed from the same coordinates and ties are broken by the
same $<$), so $\mathrm{PH}_p(S(F(u)))=\mathrm{PH}_p(S(u))$ by Lemma~\ref{lem:dimred}. If the Lipschitz
estimate $d_B(\mathrm{PH}_p(S(\bar w)),\mathrm{PH}_p(S(\bar w')))\le L_0\|\bar w-\bar w'\|_\infty$ holds on
$B_\infty(F(v),\delta)\cap FF_N(K)$, then for $w,w'\in B_\infty(v,\min\{\delta,1\})\cap FF_N(K)$ we
get $d_B(\mathrm{PH}_p(S(w)),\mathrm{PH}_p(S(w')))=d_B(\mathrm{PH}_p(S(F(w))),\mathrm{PH}_p(S(F(w'))))\le
L_0\|F(w)-F(w')\|_\infty\le L_0\|w-w'\|_\infty$. Since $F(v)$ has the same coordinates as $v$ on
$K^{(p)}$, $F(v)\notin\mathcal N_p$. Hence, renaming $F(v)$ as $v$, we may assume that
$v\in FF_N(K)\setminus\mathcal N_p$ satisfies:
\begin{itemize}
\item[(F1)] $v(\sigma)_i=\pm M\ne\epsilon$ for every $\sigma\notin K^{(p)}$ and $1\le i\le N$; in
particular every coordinate of $v$ equal to $\epsilon$ belongs to a simplex of $K^{(p)}$;
\item[(F2)] the events of $S(v)$ acting on simplices outside $K^{(p)}$ are the insertions of the
simplices of dimension $<p$ at time $\epsilon/M$ and their deletions at time $N+1-\epsilon/M$
(simplices of dimension $>p+1$ are never present), and no event of a simplex of $K^{(p)}$ has one
of these two times, since for $\sigma\in K^{(p)}$ the crossing times on $[0,1]$ and on $[N,N+1]$
are $\epsilon/v(\sigma)_1>\epsilon/M$ and $N+1-\epsilon/v(\sigma)_N<N+1-\epsilon/M$, and all other
crossing times lie in $[1,N]$.
\end{itemize}

\medskip\noindent\textbf{Step 0 (neighbourhood construction and tracked crossings).}
Let $\eta:=\min\{|v(\sigma)_j-\epsilon| : \sigma\in K,\ 0\le j\le N+1,\ v(\sigma)_j\ne\epsilon\}>0$,
and let $B:=B_\infty(v,\delta)\cap FF_N(K)$ with $\delta<\min\{\eta/2,1\}$. Consequently, for $u\in B$
and every coordinate with $v(\sigma)_j\ne\epsilon$ we have
$\mathrm{sgn}(u(\sigma)_j-\epsilon)=\mathrm{sgn}(v(\sigma)_j-\epsilon)$. Also, note that $B$ is convex.

\emph{(a) Ordinary crossings.} Let $(\sigma,k)$, $k\in\{0,\dots,N\}$, with
$v(\sigma)_k,v(\sigma)_{k+1}\ne\epsilon$. By sign preservation, $(\sigma,k)$ is a crossing of $u$ iff
it is a crossing of $v$, of the same type; it is never a phantom crossing (Definition~\ref{def:phantom}),
and its time $t^{(k)}_\sigma(u)$ lies in the open interval $(k,k+1)$. The denominator satisfies
$|u(\sigma)_{k+1}-u(\sigma)_k|\ge2\eta-2\delta>\eta$, so $t^{(k)}_\sigma$ is $C^\infty$ on $B$ with
derivatives bounded uniformly over the finitely many $(\sigma,k)$; let $L_0$ be a common
Lipschitz constant with respect to $\|\cdot\|_\infty$.

\emph{(b) $\epsilon$-crossings with opposite neighbours.} Let $v(\sigma)_i=\epsilon$ for some
$i\in\{1,\dots,N\}$; by (F1), $\sigma\in K^{(p)}$, and since $v\notin\mathcal N^{\epsilon,\epsilon}_{p,p+1}$,
$v(\sigma)_{i-1},v(\sigma)_{i+1}\ne\epsilon$. Assume in this item that they lie on opposite sides of
$\epsilon$. For $u\in B$ the values $u(\sigma)_{i\pm1}$ keep these signs. Hence, if
$v(\sigma)_{i-1}<\epsilon<v(\sigma)_{i+1}$ (entry case), $\sigma$ has exactly one crossing among
$(\sigma,i-1),(\sigma,i)$ at $u$: $(\sigma,i-1)$ if $u(\sigma)_i>\epsilon$ and $(\sigma,i)$ if
$u(\sigma)_i\le\epsilon$, both entries; if $v(\sigma)_{i-1}>\epsilon>v(\sigma)_{i+1}$ (exit case) it
is $(\sigma,i)$ if $u(\sigma)_i>\epsilon$ and $(\sigma,i-1)$ if $u(\sigma)_i\le\epsilon$, both exits.
No phantom configuration arises, since the neighbours are on opposite sides of $\epsilon$. So,
define the hand-off function
$$t^{[i]}_\sigma(u):=\begin{cases}t^{(i-1)}_\sigma(u)&\text{if }u(\sigma)_i>\epsilon\\
t^{(i)}_\sigma(u)&\text{if }u(\sigma)_i\le\epsilon\end{cases}\ \text{(entry case)},\quad
t^{[i]}_\sigma(u):=\begin{cases}t^{(i)}_\sigma(u)&\text{if }u(\sigma)_i>\epsilon\\
t^{(i-1)}_\sigma(u)&\text{if }u(\sigma)_i\le\epsilon\end{cases}\ \text{(exit case)}.$$
Both branches are $C^\infty$ on $B$, since their denominators satisfy
$|u(\sigma)_i-u(\sigma)_{i-1}|,\ |u(\sigma)_{i+1}-u(\sigma)_i|\ge\eta-2\delta>0$, and both equal $i$
on the wall $\{u(\sigma)_i=\epsilon\}$: there $t^{(i-1)}_\sigma(u)=(i-1)+1$ and $t^{(i)}_\sigma(u)=i+0$.
Enlarging $L_0$, each branch is $L_0$-Lipschitz on $B$. Since $B$ is convex and the two branches
agree on the wall, $t^{[i]}_\sigma$ is $L_0$-Lipschitz on $B$: for $u,u'\in B$ on opposite sides of
the wall, inserting the unique point $q\in[u,u']$ lying on it gives
$|t^{[i]}_\sigma(u)-t^{[i]}_\sigma(u')|\le L_0(\|u-q\|_\infty+\|q-u'\|_\infty)=L_0\|u-u'\|_\infty$.
We call $(\sigma,[i])$ a \emph{merged crossing}; it is an insertion in the entry case and a deletion
in the exit case.

\emph{(c) $\epsilon$-crossings with neighbours on the same side (transient pairs).} Let now
$v(\sigma)_i=\epsilon$ with $v(\sigma)_{i-1},v(\sigma)_{i+1}$ on the same side of $\epsilon$; again
$\sigma\in K^{(p)}$ by (F1). If both are $>\epsilon$ (\emph{type $+$}), then $\dim\sigma\ne p$
because $v\notin\mathcal N^{\epsilon+}_p$, hence $\dim\sigma=p+1$; if both are $<\epsilon$ (\emph{type
$-$}), then $\dim\sigma\ne p+1$ because $v\notin\mathcal N^{\epsilon-}_{p+1}$, hence $\dim\sigma=p$. For
$u\in B$ the values $u(\sigma)_{i\pm1}$ keep their signs, so: in type $+$, if $u(\sigma)_i<\epsilon$
then $\sigma$ has exactly the two genuine crossings $(\sigma,i-1)$ (exit) and $(\sigma,i)$ (entry),
with times $t^{(i-1)}_\sigma(u)\in(i-1,i)$ and $t^{(i)}_\sigma(u)\in(i,i+1)$; if
$u(\sigma)_i=\epsilon$ these two crossings are phantom and yield no event; if
$u(\sigma)_i>\epsilon$ neither is a crossing. In type $-$, if $u(\sigma)_i>\epsilon$ then $\sigma$
has exactly the two genuine crossings $(\sigma,i-1)$ (entry) and $(\sigma,i)$ (exit), with times in
$(i-1,i)$ and $(i,i+1)$; if $u(\sigma)_i\le\epsilon$ neither is a crossing. In both types the two
functions $t^{(i-1)}_\sigma,t^{(i)}_\sigma$ are defined by \eqref{eq1} on all of $B$, since their
denominators satisfy $|u(\sigma)_i-u(\sigma)_{i-1}|,\ |u(\sigma)_{i+1}-u(\sigma)_i|\ge\eta-2\delta>0$;
they are $C^\infty$, $L_0$-Lipschitz on $B$ (enlarging $L_0$), and both equal $i$ on the wall
$\{u(\sigma)_i=\epsilon\}$. We call $\beta:=(\sigma,i)$ a \emph{transient pair}, with crossings
$\beta^-$ (the crossing $(\sigma,i-1)$, with time function $\tau_{\beta^-}:=t^{(i-1)}_\sigma$) and
$\beta^+$ (the crossing $(\sigma,i)$, with $\tau_{\beta^+}:=t^{(i)}_\sigma$); $\beta^-$ is a deletion
and $\beta^+$ an insertion in type $+$, and conversely in type $-$. The \emph{active region} of
$\beta$ is the open half-ball $B_\beta:=\{u\in B:u(\sigma)_i<\epsilon\}$ in type $+$ and
$B_\beta:=\{u\in B:u(\sigma)_i>\epsilon\}$ in type $-$. On $B_\beta$ the two crossings of $\beta$ are
genuine crossings of $u$ and $\tau_{\beta^-}(u)<i<\tau_{\beta^+}(u)$; on the wall
$\partial B_\beta\cap B$ one has $\tau_{\beta^-}(u)=\tau_{\beta^+}(u)=i$.

\emph{(d) Tracked crossings and active sets.} Let $\mathcal E_0$ be the finite set consisting of the
ordinary crossings of $v$ from (a) and of the merged crossings from (b), let
$\mathcal E_1:=\{\beta^-,\beta^+:\beta\text{ a transient pair from (c)}\}$, and
$\mathcal E:=\mathcal E_0\sqcup\mathcal E_1$. For $\varepsilon\in\mathcal E$ write $|\varepsilon|\in K$ for
its simplex, $\mathrm{sgn}(\varepsilon)\in\{+,-\}$ for its type and $\tau_\varepsilon:B\to\mathbb R$ for its
time function; all $\tau_\varepsilon$ are $L_0$-Lipschitz on $B$. For $u\in B$ let
$$\mathcal A(u):=\mathcal E_0\ \cup\ \{\beta^-,\beta^+:\ u\in B_\beta\}$$
be the \emph{active set} at $u$. By (a), (b), (c), for every $u\in B$ the genuine crossings of $u$,
i.e.\ the events of $S(u)$ (Lemma~\ref{lem:S-valid}), are exactly the events
$\mathrm{sgn}(\varepsilon)|\varepsilon|$, $\varepsilon\in\mathcal A(u)$, at the times $\tau_\varepsilon(u)$:
a crossing of $u$ either has both values $\ne\epsilon$ at $v$, and is then in (a), or involves
exactly one coordinate with $v(\sigma)_i=\epsilon$ (two adjacent ones are excluded by
$\mathcal N^{\epsilon,\epsilon}_{p,p+1}$ and (F1)), and is then accounted for in (b) or (c). Note that
$|\mathcal A(u)|$ is not constant on $B$. Distinct elements of $\mathcal E$ acting on the same simplex
have distinct times at $v$, with one exception: the two crossings $\beta^-,\beta^+$ of a transient
pair, both of time $i$ at $v$ (ordinary crossings have non-integer times in disjoint intervals,
merged and transient crossings have integer times equal to their coordinate index, and a coordinate
is of at most one kind).
Finally, let $\rho:=\min\{|\tau_\varepsilon(v)-\tau_{\varepsilon'}(v)|:\varepsilon,\varepsilon'\in\mathcal E,\
\tau_\varepsilon(v)\ne\tau_{\varepsilon'}(v)\}>0$ and shrink $\delta$ so that $L_0\delta<\rho/4$. Then
$|\tau_\varepsilon(u)-\tau_\varepsilon(v)|<\rho/4$ for all $u\in B$ and all $\varepsilon\in\mathcal E$, hence
\begin{equation}\label{eq:order}
\tau_\varepsilon(v)<\tau_{\varepsilon'}(v)\ \Longrightarrow\ \tau_{\varepsilon'}(u)-\tau_\varepsilon(u)>\rho/2
\qquad\text{for all }u\in B .
\end{equation}
Consequently, elements of $\mathcal E$ whose times at some $u\in B$ differ by less than $\rho/2$
have the same time at $v$; in particular, crossings that are simultaneous at some $u\in B$ are
simultaneous at $v$. Note that two genuine crossings of the same simplex never share a common
time in $FF_N(K)$ (Remark~\ref{obs:phantom}).

\medskip\noindent\textbf{Step 1 (admissible orderings and their lifts).}
An \emph{ordering datum} is a pair $(\mathcal A,\pi)$ where $\mathcal A=\mathcal E_0\cup\{\beta^\pm:\beta\in J\}$
for some set $J$ of transient pairs and $\pi$ is a total order on $\mathcal A$. Listing the crossings
of $\mathcal A$ in the order $\pi$ gives the event sequence
$\mathbf e_{\mathcal A,\pi}:=(\mathrm{sgn}(\pi(1))|\pi(1)|,\dots,\mathrm{sgn}(\pi(|\mathcal A|))|\pi(|\mathcal A|)|)$,
which depends on $(\mathcal A,\pi)$ only, and for $u\in B$ the triple
$$f_{\mathcal A,\pi}(u):=\big(|\mathcal A|,\ (\tau_{\pi(1)}(u),\dots,\tau_{\pi(|\mathcal A|)}(u)),\ \mathbf e_{\mathcal A,\pi}\big).$$
The datum is \emph{admissible at $u$} if $\mathcal A=\mathcal A(u)$, (A) $\varepsilon\,\pi\,\varepsilon'$ implies
$\tau_\varepsilon(u)\le\tau_{\varepsilon'}(u)$, and (V) $\mathbf e_{\mathcal A,\pi}$ is valid. If $(\mathcal A,\pi)$
is admissible at $u$, then $f_{\mathcal A,\pi}(u)\in\mathrm{Rep}(S(u))$. Indeed, $f_{\mathcal A,\pi}(u)$ is
valid by (V), its times are non-decreasing by (A), and its events are the genuine crossings of $u$
with their times (Step~0(d)), i.e.\ the events of $S(u)$: hence its times are the stamps of $S(u)$,
the complex it reaches after the events of time $\le t$ is $K_t(u)$ (the reached complex depends only
on which simplices have been toggled an odd number of times, Appendix~\ref{app:rep}), so
$f_{\mathcal A,\pi}(u)\in\mathrm{Rep}^*(S(u))$, and it has no detour because two genuine crossings of
the same simplex never share a time (Remark~\ref{obs:phantom}); Remark~\ref{rem:rep}(ii) gives
$f_{\mathcal A,\pi}(u)\in\mathrm{Rep}(S(u))$. The representation of $S(u)$ in $\mathrm{SR}(K;<)$ corresponds
to the admissible datum $(\mathcal A(u),\pi_u)$, where $\pi_u$ sorts by $\tau_\cdot(u)$ and breaks ties
by $<$. Hence Lemma~\ref{lem:tie} gives
\begin{equation}\label{eq:lemma-applied}
d_B\big(\mathrm{PH}_p(f_{\mathcal A,\pi}(u)),\,\mathrm{PH}_p(S(u))\big)=0\qquad\text{whenever $(\mathcal A,\pi)$ is admissible at $u$.}
\end{equation}
Let $\Pi$ be the finite set of ordering data admissible at some point of $B$. For
$(\mathcal A,\pi)\in\Pi$ the sequence $\mathbf e_{\mathcal A,\pi}$ is a valid simplex-wise zigzag
filtration, so its pairing $P^{\mathcal A,\pi}_p:=P^{\mathbf e_{\mathcal A,\pi}}_p$ is defined, and we set
$$\Phi_{\mathcal A,\pi}:B\to\mathbb R^{2|P^{\mathcal A,\pi}_p|},\qquad
  \Phi_{\mathcal A,\pi}(u):=\widetilde{\mathrm{PH}}_{p,\mathbf e_{\mathcal A,\pi}}\big(\tau_{\pi(1)}(u),\dots,\tau_{\pi(|\mathcal A|)}(u)\big).$$
$\Phi_{\mathcal A,\pi}$ is defined on all of $B$, whether or not $(\mathcal A,\pi)$ is admissible at $u$
(all time functions are defined on $B$), and each of its coordinates is one of the functions
$\tau_\varepsilon$; hence $\Phi_{\mathcal A,\pi}$ is $L_0$-Lipschitz from $(B,\|\cdot\|_\infty)$ to
$(\mathbb R^{2|P^{\mathcal A,\pi}_p|},\|\cdot\|_\infty)$. If $(\mathcal A,\pi)$ is admissible at $u$,
Proposition~\ref{prop:permutation} applied to $f_{\mathcal A,\pi}(u)$ gives
$Q_{|P^{\mathcal A,\pi}_p|}(\Phi_{\mathcal A,\pi}(u))=\mathrm{PH}_p(f_{\mathcal A,\pi}(u))$, so by \eqref{eq:lemma-applied}
\begin{equation}\label{eq:lift-agrees}
d_B\big(Q_{|P^{\mathcal A,\pi}_p|}(\Phi_{\mathcal A,\pi}(u)),\,\mathrm{PH}_p(S(u))\big)=0\qquad\text{whenever $(\mathcal A,\pi)$ is admissible at $u$.}
\end{equation}

\medskip\noindent\textbf{Step 2 (Lipschitz bound along a segment).}
Let $w,w'\in B$ and $\gamma(s):=(1-s)w+sw'$, $s\in[0,1]$, contained in $B$ by its convexity.
Each coordinate $s\mapsto\gamma(s)(\sigma)_i$ is affine, so the segment meets each wall
$\{u(\sigma)_i=\epsilon\}$ (of a merged crossing or of a transient pair) in at most one point or lies
inside it; hence $[0,1]$ splits into finitely many closed sub-intervals on each of which every
$\tau_\varepsilon\circ\gamma$ is a rational function of $s$ (one branch of \eqref{eq1} or of a
hand-off). On such a sub-interval a difference $\tau_\varepsilon\circ\gamma-\tau_{\varepsilon'}\circ\gamma$
is rational, hence identically zero or with finitely many zeros. Let $0=s_0<s_1<\dots<s_r=1$
consist of $0$, $1$, the wall points and these zeros. On each open interval $(s_{j-1},s_j)$ the
active set $\mathcal A(\gamma(s))$ is constant (it changes only at the walls of transient pairs), call
it $\mathcal A_j$, and every difference $\tau_\varepsilon\circ\gamma-\tau_{\varepsilon'}\circ\gamma$ has
constant sign (positive, negative or identically zero), so the canonical order $\pi_{\gamma(s)}$ on
$\mathcal A_j$ is the same for all $s\in(s_{j-1},s_j)$; call it $\pi_j$. Then $(\mathcal A_j,\pi_j)\in\Pi$.

\emph{Claim: $d_B\big(Q(\Phi_{\mathcal A_j,\pi_j}(u_*)),\mathrm{PH}_p(S(u_*))\big)=0$ for
$u_*\in\{\gamma(s_{j-1}),\gamma(s_j)\}$.} We treat $u_*=\gamma(s_j)$, the other endpoint being
identical. Since active regions are open, $\mathcal A(u_*)\subseteq\mathcal A_j$; let $J_*$ be the set
of transient pairs $\beta$ with $\beta^\pm\in\mathcal A_j\setminus\mathcal A(u_*)$. For $\beta=(\sigma,i)\in J_*$,
$u_*$ lies on the wall $\{u(\sigma)_i=\epsilon\}$ (by continuity, since $\gamma(s)\in B_\beta$ for
$s<s_j$ close to $s_j$ and $u_*\notin B_\beta$), so $\tau_{\beta^-}(u_*)=\tau_{\beta^+}(u_*)=i$.

\emph{(i) The triple $g:=f_{\mathcal A_j,\pi_j}(u_*)$ belongs to $\mathrm{SR}^*(K)$ and, for every
$\beta=(\sigma,i)\in J_*$, the pair $(\beta^-,\beta^+)$ is a detour of $g$ whose intermediate events
act on simplices distinct from and face-unrelated to $\sigma$.} The event sequence
$\mathbf e_{\mathcal A_j,\pi_j}$ is valid by (V) at the points $\gamma(s)$, $s\in(s_{j-1},s_j)$, and
starts and ends at $\emptyset$; the times $\tau_\cdot(u_*)$ are non-decreasing along $\pi_j$ by
continuity; hence $g\in\mathrm{SR}^*(K)$ by Remark~\ref{rem:rep}(i). Since $\tau_{\beta^-}<\tau_{\beta^+}$
on $(s_{j-1},s_j)$, $\beta^-$ precedes $\beta^+$ in $\pi_j$, and every event between them has time
$i$ at $u_*$, hence (by \eqref{eq:order}) time $i$ at $v$; by Step~0(d) such an event is a merged
crossing $(\mu,[i])$ or a crossing of a transient pair $(\mu,i)$ with $v(\mu)_i=\epsilon$, so
$\mu\in K^{(p)}$ by (F1) and $\mu\ne\sigma$ (the coordinate $(\sigma,i)$ is of one kind only); in
particular no event of $\sigma$ lies between $\beta^-$ and $\beta^+$, so $(\beta^-,\beta^+)$ is a
detour. If $\beta$ is of type $+$ ($\dim\sigma=p+1$, $v(\sigma)_{i\pm1}>\epsilon$), a face
$\mu\subsetneq\sigma$ of this kind would have $\dim\mu=p$ and $v(\mu)_{i\pm1}\ge v(\sigma)_{i\pm1}>\epsilon$,
i.e.\ $v\in\mathcal N^{\epsilon+}_p$; and the cofaces of $\sigma$ have dimension $>p+1$, hence no
coordinate equal to $\epsilon$ by (F1). If $\beta$ is of type $-$ ($\dim\sigma=p$,
$v(\sigma)_{i\pm1}<\epsilon$), a coface $\mu\supsetneq\sigma$ of this kind would have $\dim\mu=p+1$
and $v(\mu)_{i\pm1}\le v(\sigma)_{i\pm1}<\epsilon$, i.e.\ $v\in\mathcal N^{\epsilon-}_{p+1}$; and the
faces of $\sigma$ have dimension $<p$. This proves (i).

\emph{(ii) $d_B(\mathrm{PH}_p(g),\mathrm{PH}_p(g^\circ))=0$, where $g^\circ:=f_{\mathcal A(u_*),\pi_j|_{\mathcal A(u_*)}}(u_*)$.}
By (i), Lemma~\ref{lem:detour} applies to $g$ and to the detour $(\beta^-,\beta^+)$ of any
$\beta\in J_*$ ($\beta^-=-\sigma$ with $\dim\sigma=p+1$ in type $+$, $\beta^-=+\sigma$ with
$\dim\sigma=p$ in type $-$), and removing the pair preserves the hypotheses for the remaining
pairs (their intermediate events form a subset of the previous ones). Removing the pairs of
$J_*$ one at a time yields $g^\circ$, which belongs to $\mathrm{SR}^*(K)$ and is valid.

\emph{(iii) $g^\circ\in\mathrm{Rep}(S(u_*))$, hence $d_B(\mathrm{PH}_p(g^\circ),\mathrm{PH}_p(S(u_*)))=0$.}
The triple $g^\circ$ is valid by (ii), its times are non-decreasing, and its events are those of
$\mathcal A(u_*)$ at the times $\tau_\cdot(u_*)$, i.e.\ the genuine crossings of $u_*$ with their
times (Step~0(d)). Exactly as in Step~1, this gives $g^\circ\in\mathrm{Rep}(S(u_*))$, and
Lemma~\ref{lem:tie} applied to $g^\circ$ and to the representation of $S(u_*)$ in $\mathrm{SR}(K;<)$
gives the claim.

By (iii) and (ii), $d_B(\mathrm{PH}_p(g),\mathrm{PH}_p(S(u_*)))=0$; by Proposition~\ref{prop:permutation}
applied to $g\in\mathrm{SR}^*(K)$ (Remark~\ref{rem:rep}(iv)), $Q(\Phi_{\mathcal A_j,\pi_j}(u_*))=\mathrm{PH}_p(g)$;
this proves the Claim.

By the Claim, the triangle inequality for the pseudo-metric $d_B$, the $1$-Lipschitz property of
$Q$ and the $L_0$-Lipschitz property of $\Phi_{\mathcal A_j,\pi_j}$,
\begin{equation*}
\begin{aligned}
d_B\big(
    \mathrm{PH}_p(S(\gamma(s_{j-1}))),
    \mathrm{PH}_p(S(\gamma(s_j)))
\big)
&\le
\big\|
    \Phi_{\mathcal A_j,\pi_j}(\gamma(s_{j-1}))
    -
    \Phi_{\mathcal A_j,\pi_j}(\gamma(s_j))
\big\|_\infty
\\
&\le
L_0\,\|\gamma(s_{j-1})-\gamma(s_j)\|_\infty .
\end{aligned}
\end{equation*}
Summing over $j$ and using that the points $\gamma(s_j)$ lie in order on the segment $[w,w']$, so
that $\sum_j\|\gamma(s_j)-\gamma(s_{j-1})\|_\infty=\|w-w'\|_\infty$,
$$d_B\big(\mathrm{PH}_p(S(w)),\mathrm{PH}_p(S(w'))\big)\le L_0\,\|w-w'\|_\infty\qquad\text{for all }w,w'\in B .$$
Hence $\mathrm{PH}_p\circ S$ is Lipschitz on $B$, which by the Reduction proves the theorem.
\end{proof}

\subsection{Region of potential non-local Lipschitzness}
\label{app:region_nonloc}

From Theorem \ref{thm:lipschitz_breve}, the map $\mathrm{PH}_p\circ S$ may fail to be locally Lipschitz, or even continuous, on $\mathcal N$ (Definition \ref{sec:lip_theorem}), which only bounds this region rather than characterizing it exactly: not every point of $\mathcal N$ is necessarily a failure point. We show two examples of these
discontinuities.

\begin{example}
\label{ex:split}
For simplicity, assume $K=\{\sigma\}$, $N=3$, and $p=0$.
Let $v(\sigma):=(0,a,\epsilon,a,0)$, with $a>\epsilon$, so that $v\in\mathcal{N}^{\epsilon+}_0$: $\sigma$ enters at $\epsilon/a$, exits at $4-\epsilon/a$, and its exit at time $2$ on $[1,2]$ cancels with its re-entry at time $2$ on $[2,3]$ (a phantom pair), so that $\mathrm{PH}_0(S(v))=\{(\epsilon/a,\,4-\epsilon/a)\}$.
Then, perturbing $v$ we get $u=(0,a,\epsilon-\delta,a,0)$, for which the phantom pair becomes a genuine exit at $t_1=2-\frac{\delta}{a-\epsilon+\delta}$ followed by a re-entry at $t_2=2+\frac{\delta}{a-\epsilon+\delta}$, and $\mathrm{PH}_0(S(u))=\{(\epsilon/a,\,t_1),\,(t_2,\,4-\epsilon/a)\}$. However,
$$d_B\big(\mathrm{PH}_0(S(u)),\mathrm{PH}_0(S(v))\big)=\min\Big\{4-\tfrac{\epsilon}{a}-t_1,\ t_2-\tfrac{\epsilon}{a},\ 2-\tfrac{\epsilon}{a}\Big\}\ \xrightarrow[\ \delta\to0^+\ ]{}\ 2-\tfrac{\epsilon}{a}>1,$$
while $\|u-v\|_\infty=\delta\to0$, so that $\mathrm{PH}_0\circ S$ is discontinuous at $v$ (from the other side, $u=(0,a,\epsilon+\delta,a,0)$ gives the same diagram as $v$: the jump is one-sided).
Qualitatively, $\mathcal{N}^{\epsilon+}_p$ and $\mathcal{N}^{\epsilon-}_{p+1}$ can lead to discontinuities from breaking intervals in the middle: a $p$-simplex disappearing ($\mathcal N^{\epsilon+}_p$) or $p+1$-simplex appearing ($\mathcal N^{\epsilon-}_{p+1}$) for an arbitrarily short time splits the interval of a class it carries, respectively kills, into two intervals of non-vanishing length.
See Figure \ref{fig:disc-split} for a visual representation of the example.
\end{example}

\begin{example}
\label{ex:create}
Assume $K=\{\sigma\}$, $N=2$, and $p=0$.
Let $v(\sigma):=(0,\epsilon,\epsilon,0)$, so that $v\in\mathcal{N}^{\epsilon,\epsilon}_0$: $\sigma$ never crosses $\epsilon$ and $\mathrm{PH}_0(S(v))=\emptyset$.
Then, perturbing $v$ we get $u=(0,\epsilon+\delta,\epsilon,0)$, and $\mathrm{PH}_0(S(u))=\{(\tfrac{\epsilon}{\epsilon+\delta},\,2)\}$. However, $d_B\big(\mathrm{PH}_0(S(u)),\mathrm{PH}_0(S(v))\big)=1-\tfrac{\epsilon}{2(\epsilon+\delta)}>\tfrac12$ for every $\delta>0$, so that $\mathrm{PH}_0\circ S$ is discontinuous at $v$.
Note that similarly happens taking $u=(0,\epsilon,\epsilon+\delta,0)$ or $u=(0,\epsilon+\delta,\epsilon+\delta,0)$.
Qualitatively, $\mathcal{N}^{\epsilon,\epsilon}_0$ can lead to discontinuities from the interpolation formula \eqref{eq1} becoming $0/0$: an interval of non-vanishing length appears or disappears.
See Figure \ref{fig:disc-create}.
\end{example}

Finally, we briefly discuss the size of the critical set $\mathcal N_p\subseteq FF_N(K)\subseteq (\mathbb R^{N|K|})^*$. The sets $\mathcal{N}^{\epsilon+}_p,\mathcal{N}^{\epsilon-}_{p+1}$ are contained in a finite union of affine hyperplanes of codimension $1$, while $\mathcal{N}^{\epsilon,\epsilon}_{p,p+1}$ is contained in a finite union of codimension $2$; in particular $\mathcal N_p$ has Lebesgue measure $0$ in $(\mathbb R^{N|K|})^*$.
Moreover, since it is easy to see that $\mathcal N_p$ is closed, the complementary set $FF_N(K)\setminus \mathcal N_p$ is open in $FF_N(K)$.

To give intuition on the shape of this set, consider the simple example $K=\{\sigma\}$, $N=3$,
$p=0$ (so $\dim\sigma=0$ and there is no simplex of dimension $1$). Here
$FF_3(K)=\{(0,x,y,z,0):x,y,z\in\mathbb R\}=(\mathbb R^{3|K|})^*$, and the critical set becomes $\mathcal{N}^{\epsilon,\epsilon}_{0,1}=\{(0,x,y,z,0)\in FF_3(K):x=y=\epsilon \text{ or } y=z=\epsilon\}$, $\mathcal{N}^{\epsilon+}_0=\{(0,x,y,z,0)\in FF_3(K):x,z>\epsilon\text{ and } y=\epsilon\}$ and $\mathcal{N}^{\epsilon-}_1=\emptyset$. Thus $\mathcal N_0=\mathcal N^{\epsilon+}_0\cup\mathcal N^{\epsilon,\epsilon}_{0,1}$ is entirely contained in the plane $\{y=\epsilon\}$, and removes from it only the open quadrant $x,z>\epsilon$ together with the two lines $\{x=y=\epsilon\}$ and $\{y=z=\epsilon\}$. Since the plane $\{y=\epsilon\}$ is not entirely removed, one can always continuously cross $\mathcal N_0$ without touching it; hence $FF_N(K)\setminus
\mathcal N_0$ is connected, as illustrated in Figure~\ref{fig:safe_set}.

\begin{figure}[h]
\centering
\begin{subfigure}[b]{0.48\textwidth}\centering
\begin{tikzpicture}[x=1.02cm, y=0.34cm, font=\small, >={Latex[length=4pt]}]
  \path[use as bounding box] (-0.85,-6.1) rectangle (4.85,3.9);
  \draw[zzthr] (-0.25,1) -- (4.45,1);
  \node[zztxt, gray!50!black] at (-0.55,1) {$\epsilon$};
  \draw[->] (-0.25,0) -- (4.62,0) node[right, zztxt] {$t$};
  \foreach \i in {0,1,2,3,4}
    \draw (\i,0.3) -- (\i,-0.3) node[below, zztxt] {$\i$};
  \draw[zzval] (0,0) -- (1,3) -- (2,1) -- (3,3) -- (4,0);
  \foreach \p/\q in {0/0, 1/3, 2/1, 3/3, 4/0} \fill[blue!60!black] (\p,\q) circle (1.5pt);
  \node[above, zztxt, blue!60!black] at (1,3) {$a$};
  \node[above, zztxt, blue!60!black] at (3,3) {$a$};
  \node[zzmark] at (2,1) {};
  \node[zzred, align=center] (phl) at (2,2.6) {phantom\\ pair};
  \draw[red!75!black, ->, thin] (phl) -- (2,1.3);
  \draw[zzdrop] (0.3333,0) -- (0.3333,-3.0);
  \draw[zzdrop] (3.6667,0) -- (3.6667,-3.0);
  \draw[zzbar]  (0.3333,-3.0) -- (3.6667,-3.0);
  \node[zztxt] at (2,-5.3) {$\mathrm{PH}_0(S(v))=\bigl\{(\tfrac{\epsilon}{a},\,4-\tfrac{\epsilon}{a})\bigr\}$};
\end{tikzpicture}
\caption{$v=(0,a,\epsilon,a,0)\in\mathcal N^{\epsilon+}_0$}
\end{subfigure}
\hfill
\begin{subfigure}[b]{0.48\textwidth}\centering
\begin{tikzpicture}[x=1.02cm, y=0.34cm, font=\small, >={Latex[length=4pt]}]
  \path[use as bounding box] (-0.85,-6.1) rectangle (4.85,3.9);
  \draw[zzthr] (-0.25,1) -- (4.45,1);
  \node[zztxt, gray!50!black] at (-0.55,1) {$\epsilon$};
  \draw[->] (-0.25,0) -- (4.62,0) node[right, zztxt] {$t$};
  \foreach \i in {0,1,2,3,4}
    \draw (\i,0.3) -- (\i,-0.3) node[below, zztxt] {$\i$};
  \draw[zzval] (0,0) -- (1,3) -- (2,0.5) -- (3,3) -- (4,0);
  \foreach \p/\q in {0/0, 1/3, 2/0.5, 3/3, 4/0} \fill[blue!60!black] (\p,\q) circle (1.5pt);
  \node[above, zztxt, blue!60!black] at (1,3) {$a$};
  \node[above, zztxt, blue!60!black] at (3,3) {$a$};
  \node[zzred] (dl) at (2,2.6) {$\epsilon-\delta$};
  \draw[red!75!black, ->, thin] (dl) -- (2,0.72);
  \fill[red!75!black] (1.8,1) circle (1.5pt);
  \fill[red!75!black] (2.2,1) circle (1.5pt);
  \draw[zzdrop] (0.3333,0) -- (0.3333,-3.0);
  \draw[zzdrop] (1.8,1)    -- (1.8,-3.0);
  \draw[zzdrop] (2.2,1)    -- (2.2,-3.0);
  \draw[zzdrop] (3.6667,0) -- (3.6667,-3.0);
  \draw[zzbar]  (0.3333,-3.0) -- (1.8,-3.0);
  \draw[zzbar]  (2.2,-3.0)    -- (3.6667,-3.0);
  \node[zzred, anchor=east] at (1.76,-1.95) {$t_1$};
  \node[zzred, anchor=west] at (2.24,-1.95) {$t_2$};
  \draw[<->, red!75!black, thin] (1.8,-3.8) -- (3.6667,-3.8);
  \node[zzred] at (2.73,-4.35) {$d_B\to 2-\epsilon/a$};
  \node[zztxt] at (2,-5.3) {$\mathrm{PH}_0(S(u))=\bigl\{(\tfrac{\epsilon}{a},t_1),(t_2,4-\tfrac{\epsilon}{a})\bigr\}$};
\end{tikzpicture}
\caption{$u=(0,a,\epsilon-\delta,a,0)$}
\end{subfigure}
\caption{\textbf{Example~\ref{ex:split}: an interval broken in the middle
($\mathcal N^{\epsilon+}_p$).} In each panel the upper plot is the piecewise-linear interpolation of the filtering values $v$ at $i=0,\dots,N+1$, with the threshold $\epsilon$ dashed, and the lower plot is the resulting $H_0$ barcode on the same time axis. In (a) the value at $i=2$ equals $\epsilon$ with both neighbours above it, so the exit and the re-entry occur at the same time and cancel as a phantom pair, leaving a single bar. In (b) an
arbitrarily small $\delta>0$ makes both crossings genuine, and the bar breaks in two.}
\label{fig:disc-split}
\end{figure}

\begin{figure}[h]
\centering
\begin{subfigure}[b]{0.32\textwidth}\centering
\begin{tikzpicture}[x=0.98cm, y=0.58cm, font=\small, >={Latex[length=4pt]}]
  \path[use as bounding box] (-0.8,-3.45) rectangle (3.75,2.35);
  \draw[zzthr] (-0.25,1) -- (3.45,1);
  \node[zztxt, gray!50!black] at (-0.55,1) {$\epsilon$};
  \draw[->] (-0.25,0) -- (3.62,0) node[right, zztxt] {$t$};
  \foreach \i in {0,1,2,3}
    \draw (\i,0.14) -- (\i,-0.14) node[below, zztxt] {$\i$};
  \draw[zzval] (0,0) -- (1,1) -- (2,1) -- (3,0);
  \foreach \p/\q in {0/0, 1/1, 2/1, 3/0} \fill[blue!60!black] (\p,\q) circle (1.5pt);
  \node[zzred] at (1.5,1.72) {never $>\epsilon$};
  \draw[gray!45, densely dotted] (0,-1.35) -- (3,-1.35);
  \node[zztxt, fill=white, inner sep=1pt] at (1.5,-1.35) {$\emptyset$};
  \node[zztxt] at (1.5,-2.6) {$\mathrm{PH}_0(S(v))=\emptyset$};
\end{tikzpicture}
\caption{$v=(0,\epsilon,\epsilon,0)\in \mathcal N^{\epsilon,\epsilon}_0$}
\end{subfigure}
\hfill
\begin{subfigure}[b]{0.32\textwidth}\centering
\begin{tikzpicture}[x=0.98cm, y=0.58cm, font=\small, >={Latex[length=4pt]}]
  \path[use as bounding box] (-0.8,-3.45) rectangle (3.75,2.35);
  \draw[zzthr] (-0.25,1) -- (3.45,1);
  \node[zztxt, gray!50!black] at (-0.55,1) {$\epsilon$};
  \draw[->] (-0.25,0) -- (3.62,0) node[right, zztxt] {$t$};
  \foreach \i in {0,1,2,3}
    \draw (\i,0.14) -- (\i,-0.14) node[below, zztxt] {$\i$};
  \draw[zzval] (0,0) -- (1,1.35) -- (2,1) -- (3,0);
  \foreach \p/\q in {0/0, 1/1.35, 2/1, 3/0} \fill[blue!60!black] (\p,\q) circle (1.5pt);
  \draw[<->, red!75!black, thin] (1,1) -- (1,1.35);
  \node[zzred, anchor=south] at (1,1.45) {$\delta$};
  \node[zzmark] at (2,1) {};
  \draw[zzdrop] (0.741,1) -- (0.741,-1.35);
  \draw[zzdrop] (2,1)     -- (2,-1.35);
  \draw[zzbar]  (0.741,-1.35) -- (2,-1.35);
  \node[zztxt] at (1.5,-2.6) {$\bigl\{(\tfrac{\epsilon}{\epsilon+\delta},\,2)\bigr\}$};
\end{tikzpicture}
\caption{$u=(0,\epsilon+\delta,\epsilon,0)$}
\end{subfigure}
\hfill
\begin{subfigure}[b]{0.32\textwidth}\centering
\begin{tikzpicture}[x=0.98cm, y=0.58cm, font=\small, >={Latex[length=4pt]}]
  \path[use as bounding box] (-0.8,-3.45) rectangle (3.75,2.35);
  \draw[zzthr] (-0.25,1) -- (3.45,1);
  \node[zztxt, gray!50!black] at (-0.55,1) {$\epsilon$};
  \draw[->] (-0.25,0) -- (3.62,0) node[right, zztxt] {$t$};
  \foreach \i in {0,1,2,3}
    \draw (\i,0.14) -- (\i,-0.14) node[below, zztxt] {$\i$};
  \draw[zzval] (0,0) -- (1,1.35) -- (2,1.35) -- (3,0);
  \foreach \p/\q in {0/0, 1/1.35, 2/1.35, 3/0} \fill[blue!60!black] (\p,\q) circle (1.5pt);
  \draw[<->, red!75!black, thin] (1.5,1) -- (1.5,1.35);
  \node[zzred, anchor=south] at (1.5,1.45) {$\delta$};
  \draw[zzdrop] (0.741,1) -- (0.741,-1.35);
  \draw[zzdrop] (2.259,1) -- (2.259,-1.35);
  \draw[zzbar]  (0.741,-1.35) -- (2.259,-1.35);
  \node[zztxt] at (1.5,-2.6) {$\bigl\{(\tfrac{\epsilon}{\epsilon+\delta},\,3-\tfrac{\epsilon}{\epsilon+\delta})\bigr\}$};
\end{tikzpicture}
\caption{$u=(0,\epsilon+\delta,\epsilon+\delta,0)$}
\end{subfigure}
\caption{\textbf{Example~\ref{ex:create}: an interval of non-vanishing length created from nothing ($\mathcal N^{\epsilon,\epsilon}_{p,p+1}$).} Same conventions as Figure~\ref{fig:disc-split}. In (a) two consecutive values equal $\epsilon$, so $\sigma$ never crosses the threshold and the diagram is empty. In (b) raising the first of them by $\delta>0$ leaves the value at $i=2$ exactly on the threshold, but its right neighbour is now below it, so the configuration is not phantom: the exit is genuine and occurs at $t=2$, producing the bar $\bigl(\tfrac{\epsilon}{\epsilon+\delta},\,2\bigr)$, whose length tends to $1$ and not to $0$. In (c) raising both values instead gives the bar $\bigl(\tfrac{\epsilon}{\epsilon+\delta},\,3-\tfrac{\epsilon}{\epsilon+\delta}\bigr)$ and the same conclusion holds.}
\label{fig:disc-create}
\end{figure}

\begin{figure}[h]%
    \centering
    \begin{minipage}[c]{0.48\textwidth}
        \centering
        \includegraphics[width=0.8\textwidth]{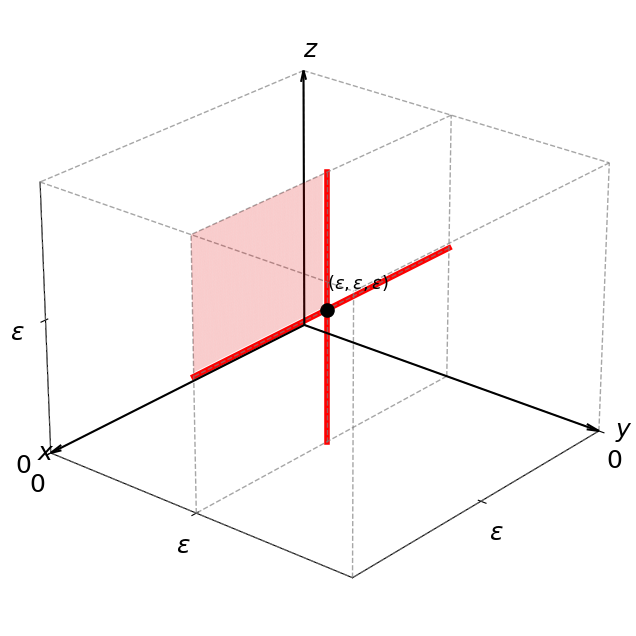}
    \end{minipage}
    \hfill
    \begin{minipage}[c]{0.48\textwidth}
        \centering
        \captionof{figure}{The set $\mathcal N_0$ is represented in red as a subset of
        $\mathbb R^{3|K|}\cong(\mathbb R^{3|K|})^*=FF_3(K)$. The highlighted red lines
        form $\mathcal{N}^{\epsilon,\epsilon}_{0,1}$, while the transparent red region
        represents $\mathcal{N}^{\epsilon+}_0$.
        The complementary set of $\mathcal N_0$ is open and connected.}
        \label{fig:safe_set}
    \end{minipage}
\vspace{-1.8em}
\end{figure}

\input{sections/app_exp_coverage}

\input{sections/app_exp_mooc}

\end{document}

%% file: sections/intro.tex
\section{Introduction}
\label{sec:intro}

Persistent homology has evolved in recent years to be embedded as a differentiable component of learning pipelines: a loss on a persistence diagram can be minimized by gradient descent, imposing a topological requirement on a model's output or learning the filtration itself \citep{leygonie2022}. That calculus takes as input a filtration of a fixed complex: one object, examined at all scales. Much of the data one would like to treat this way is instead evolving: point clouds in motion, graphs whose edges appear and disappear, fields that change in time. Their natural summary is zigzag persistent homology \citep{carlsson2010zigzag}, which follows a homological feature along a sequence of complexes connected by inclusions in either direction, tracking its identity across time rather than counting features at each instant. It is computationally practical \citep{dey2022fastzigzag}, but it cannot actively enter in any learning dynamic, since no gradient passes through a zigzag barcode. 

In standard persistent homology, the input is represented by a filtering function, a real vector that can vary continuously, and diagram endpoints are permutations of its components, the filtering values. In zigzag persistence, endpoints are instead a permutation of the sequence's temporal indices, which are fixed integers. The previous approach therefore cannot be extended to zigzag persistence, since barcode endpoints cannot inherit continuity from indices alone; a different framework is needed.
Moreover, stable optimization in standard persistence rests on the stability theorem \citep{cohensteiner2007stability}, which makes the barcode Lipschitz in the filtering function. For zigzag persistence, stability holds at module level~\citep{botnan2018}, bounding barcode distance by module distance rather than by filtration perturbations; \citep{carlsson2009realvalued} gives stability tied to an actual filtering function, but only for one specific construction (levelset zigzag). No such general theorem exists for zigzag persistence, where, moreover, no general notion of a filtering function has even been defined, making zigzag optimization harder to achieve even in principle.

We develop such a framework, first formalizing persistence computation on a sequence of complexes unrolled into a simplex-wise zigzag filtration (Section~\ref{sec:diff_pd}), and then obtaining such a sequence by thresholding time-dependent real filtering values defined on a fixed complex (Section~\ref{sec:lin_int}).
Rather than the fixed time index at which a threshold crossing is detected, we assign each barcode endpoint the real time at which the linearly interpolated filtering values cross the threshold (Eq.~\ref{eq1}). This yields smooth local lifts of the resulting map (Corollary~\ref{cor:condition2}), from which differentials follow via the chain rule.
In Section~\ref{sec:stability}, we show that a whole zigzag-persistence-based loss is differentiable almost everywhere for definable parametrizations of the filtering function (Proposition~\ref{prop:ae-differentiable}); we further prove that the map from filtering function to diagram is locally Lipschitz outside an explicit exclusion set (Theorem~\ref{thm:lipschitz_breve}), which nonetheless does not yield convergence guarantees for stochastic subgradient descent.
Even without such guarantees, we successfully test our optimization pipeline on two synthetic experiments, in Section \ref{sec:experiments}: 
minimizing coverage defect persistence in a sensor network, and classifying dynamic graphs via a learnable zigzag persistence feature map.
In Appendix~\ref{app:related}, we position these findings in detail against current literature.

%% file: sections/exp_coverage.tex
\section{Experiments}
\label{sec:experiments}

\begin{figure}[t]
\centering
\includegraphics[width=0.995\textwidth]{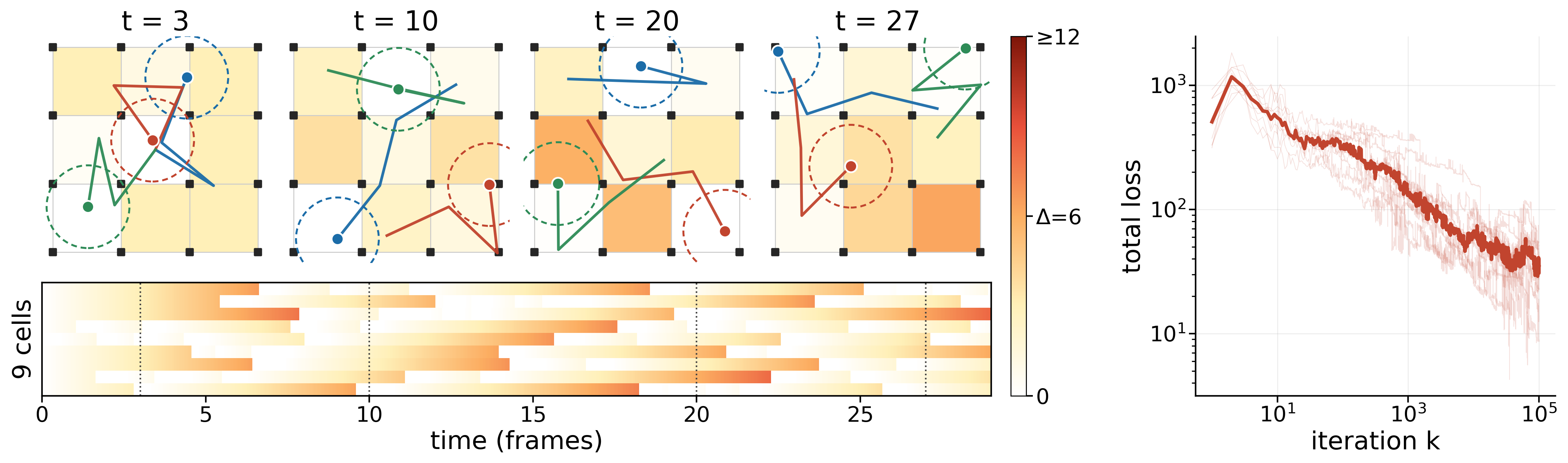}
\vspace{-0.7em}
\caption{{\bfseries The experiment.} \emph{Top:} four snapshots of the lifetime strategy's final
iterate on the median seed $18$; $t$ is elapsed time. Cells are shaded by how long they have
waited since last covered, pale to red past $\Delta=6$; black squares are the fixed sensors,
coloured dots the robots, with sensing range dashed and the last three frames of motion trailing.
\emph{Bottom left:} the same for all nine cells over the $30$ frames, snapshots dotted; the streaks keep
resetting because every cell keeps being revisited. \emph{Right:} the total objective
$\mathcal L$ along the iterate sequence, penalties included, held-out seeds $10$--$19$,
median bold and seeds faint, both axes logarithmic; the rise at $k=2$ is the first perturbation
breaking the unit-speed penalty.}
\label{fig:cov_topo}
\vspace{-0.4em}
\end{figure}

We test this construction employing zigzag calculus at two different levels: as a \emph{loss},
to minimize the persistence of a defect in a coverage network (Section~\ref{sec:exp_coverage}), and as a \emph{feature map}, to summarize a GNN transformation of dynamic graphs used for classification (Section~\ref{sec:exp_mooc}). Full specifications and results are in Appendices~\ref{app:exp_coverage} and~\ref{app:exp_mooc}.

\vspace{-0.3em}
\subsection{Optimizing how long a coverage defect lasts}
\vspace{-0.2em}
\label{sec:exp_coverage}

Sixteen static sensors on the lattice $\{0,1,2,3\}^2$ monitor $[0,3]^2$. At Vietoris-Rips radius $R=1.05$ the lattice sides are edges and the diagonals are not, so each of the nine unit cells is a $1$-cycle that does not bound: a class of $H_1$ representing a hole (the defect) in the network coverage. That
a homological condition of this kind certifies coverage is due to \citet{desilva2007,desilva2006},
and \citet{gamble2015} already follows such defects across a moving network with a zigzag barcode;
here instead we propose a setup where the barcode is differentiated and the trajectories
optimized against a functional of it (Appendix~\ref{app:coverage-related-work}). A group of $M$ robots are tasked to move for $30$ frames to cover the holes. We test from $M=2$ to $5$ robots. Since one robot can only cover one cell at once, we devise two distinct strategies to reduce these defects via a topological loss term $\mathcal{L}_{\text{topo}}$, computed on the persistence diagram $D$ in homological dimension $1$ for the evolving sensors-robots point cloud (via Rips extension at radius $R$).
The \emph{count} strategy sets $\mathcal{L}_{\text{topo}} = \sum_{(b,d)\in D}(d-b)$ to minimize total bar length, while the \emph{lifetime} strategy sets $\mathcal{L}_{\text{topo}} = \sum_{(b,d)\in D}\max(0,d-b-\Delta)^2$ for a fixed tolerance $\Delta$, ignoring short-lived gaps and penalizing only defects left uncovered for longer than $\Delta$ frames.
The total objective $\mathcal{L}$ combines $\mathcal{L}_{\text{topo}}$ with two penalty terms that keep the robots inside the lattice and enforce constant unit speed.
In what follows, we show results for $M=3$ robots and $\Delta=6$ frames. A sweep over $M$ from $2$ to $5$ and over $\Delta=6,8,10,12$ is in
Appendix~\ref{app:exp_coverage}, Tables~\ref{tab:cov_robots} and~\ref{tab:cov_delta}.

The parameters are the $174$ coordinates $y$ of the three robots' positions (two coordinates each) at frames $1,...,29$, since the shared frame-$0$ position is fixed at the lattice centre.
We minimize $\mathcal L$ with stochastic subgradient descent (SSD) (Definition \ref{def:clark}), as described in Section \ref{sec:stability}:
\begin{equation*}
y_{k+1}=y_k-\gamma_k\big(\nabla\mathcal L(y_k)+\xi_k\big),
\qquad \gamma_k=\frac{10^{-2}}{(1+k)^{0.51}},
\qquad \xi_k\sim\mathcal N(0,\sigma^2 I),\quad \sigma=30,
\end{equation*}
for $100\,000$ iterations, with a random displacement $\gamma_k\sigma$ of $0.3$ lattice units at the first iterate and $0.0008$ at the last. This Gaussian noise is introduced to insert stochasticity in the gradient, as a way to properly test SSD dynamics. We provide an analysis of turning this noise on and off in Appendix~\ref{app:exp_coverage}, Table~\ref{tab:cov_noise}.
The behaviour of these iterations is the experiment's primary object. Empirically, $\mathcal L$ steadily decreases and stabilizes along the iterates on every seed (see Fig.~\ref{fig:cov_topo} for the lifetime strategy; Appendix~\ref{app:coverage-loss} shows both strategies for all four tolerances).

Similarly, the topological loss itself consistently descends (see Appendix~\ref{app:coverage-loss}). The lifetime strategy achieves a topological loss of $33.8$, reducing the longest defect to a median of $9.4$ frames, while the count strategy reaches a total bar length of $211.0$ frames (Table~\ref{tab:cov_delta}).
As a control, we run the same optimization with $\mathcal{L}_{\text{topo}}$ removed from the loss, to check whether defect persistence is reduced relative to topologically unguided motion, and thus whether the topological gradient is responsible for it: without that term, both topological quantities instead stay high (lifetime loss $530.9$, count loss $246.7$), with a median longest defect of $21.3$ frames (Table~\ref{tab:cov_noise}). This gap confirms that the topological objective's descent is driven by the topological gradient itself instead of by the motion penalties.

%% file: sections/exp_mooc.tex
\vspace{-0.3em}
\subsection{Controlled demonstration of a learned zigzag feature map}
\vspace{-0.2em}
\label{sec:exp_mooc}

Can classification gradients passing through the interpolated endpoints of Eq.~\eqref{eq1} train the filtering function that generates the sequence of complexes, starting from one that is uninformative for the task? We build a task where the answer is checkable (Fig.~\ref{fig:mooc_setup}; Appendix~\ref{app:exp_mooc}). Each example is a graph of two figure-eights, loops sharing a node, whose nodes switch on and off over $32$ time steps so that one loop of each is closed at any time. The label is how fast the closed loop alternates in the \emph{signal} figure-eight, slowly or rapidly; a second figure-eight is a \emph{decoy} alternating at the opposite rate, and only an exclusive-or of two noisy node features tells which is which. Seen through which nodes are on, the classes have the same loop counts and lifetimes; with the decoy removed, the zigzag barcode separates them.

\begin{figure}[t]
\centering
\includegraphics[width=\textwidth]{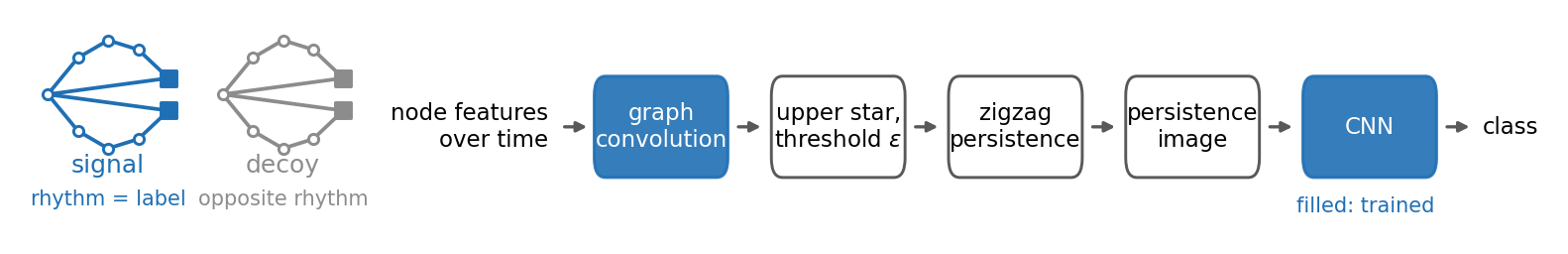}
\vspace{-1.2em}
\caption{{\bfseries The experiment.} Each example is two figure-eights whose square nodes switch so that one loop of each is closed at any time; how fast the signal alternates is the label, and the decoy alternates at the opposite rate. The node features go through the depicted pipeline, which predicts whether the signal alternates slowly or rapidly.
}
\label{fig:mooc_setup}
\vspace{-0.6em}
\end{figure}
The model is depicted in Figure \ref{fig:mooc_setup}.
At each frame a vertex carries five numbers, its activity and four features, which do not change over time. A graph convolution ($225$ parameters, applied to each frame alone with the same weights, and with no time index) maps them to a single value $x(\nu)_t$ per vertex and frame, as a residual on the activity field initialized at zero, so that training starts from uninformative $x(\nu)_t$ values.
The upper-star extension of Section~\ref{sec:applications} then turns this time-evolving vertex function into a zigzag barcode, with a fixed threshold $\epsilon$, following the paper's construction; the barcode is vectorized into a persistence image with a weighting that vanishes on the diagonal, which a $26$k-parameter convolutional reader maps to two logits.
Only the graph convolution and the reader are trained. The loss is cross-entropy. We minimize it with SSD (Definition~\ref{def:clark}), $\theta_{k+1}=\theta_k-\gamma_k\,\hat g_k,$ where $\gamma_k=0.02/(1+k)^{0.6}$
for $3\,200$ iterations, $\hat g_k$ being the gradient on $16$ examples drawn uniformly with replacement: the noise of Definition~\ref{def:clark} is the minibatch sampling alone. We evaluate on twenty splits, each a division of a dataset into training, validation and test examples (ten-fold cross-validation on two datasets that played no part in fixing any setting); on each split we train from four random seeds and report the final iterate of the seed with the lowest validation loss. Details of the experimental setup are in Appendix \ref{app:mooc_setup}.

With the uninformative filtering function frozen, the reader alone stays at chance, $0.540\pm0.086$, while trained end to end the pipeline reaches an accuracy of $0.995\pm0.016$. 
As a feasibility check, a filtering function that keeps only the signal figure-eight solves the task through the zigzag but not through per-frame persistence, so the label lies in how features are paired across time. Full results are in Appendix~\ref{app:mooc_results}.

%% file: sections/conclusions.tex
\section{Conclusions}

We present a differentiability framework for zigzag persistence generated by thresholding a time-dependent filtering function on a fixed simplicial complex. Representing bar endpoints as interpolated crossing times yields smooth local lifts with explicit differentials, from which we prove almost-everywhere differentiability under definable parametrizations of the filtering function. We also prove local Lipschitzness outside an explicit exclusion set, one hypothesis short of the convergence guarantee available for stochastic subgradient methods, a gap we do not close.
Our experiments use these gradients to show that optimization can be practically achieved even without such guarantees, in two controlled settings: a topological loss function and a learnable feature map, consistently driving the loss to low values. These results suggest that our framework can support zigzag-based optimization in different machine learning scenarios.

\subsection*{Code Availability}

The methods of this paper are implemented in \texttt{zz-top}, an open-source Python package (BSD-3-Clause) available at \url{https://github.com/RitAreaSciencePark/zz-top}: a pure-Python zigzag persistence engine and the module \texttt{zztop.zigzagdiff}, the differentiable read-out of Section~\ref{sec:diff_pd}. The code of the experiments, with the configurations and commands that regenerate every table and figure, will be added to the same repository.

\subsection*{AI use statement}

We used generative AI tools in this work, in the following ways.

\begin{description}\setlength{\itemsep}{0pt}
\item[Writing.] To aid and polish writing: We used AI tools to check for English language grammar, improve readability of sentences and structure paragraphs to be concise but clear.
\item[Retrieval and discovery.] To find related work and locate
  references: we used AI tools to survey prior works on differentiability in persistent homology and existing related works in the specific domain of zigzag persistence. We used AI tools to search for suitable datasets for the experiments.
\item[Research ideation and execution.] We used AI tools to structure the experiments in a comprehensive and extensive way, and we used it to write the necessary code to execute them.
\item[Drafting.] To draft sections of the paper: We used AI tools to think about how to structure a section, making a list of topics to cover. We also used AI tools to organize how the Appendix sections relative to the experiments should be structured.
\item[Synthetic data generation.] We used AI tools to write code producing the synthetic data used in the experiments.
\end{description}

All AI-assisted output was reviewed and verified by the authors, who take full responsibility for the final content of this paper, including its text, claims, proofs and artifacts.

\subsection*{Acknowledgements}

The authors thank Sven Heydenreich for useful comments on a draft. C. B., M.B. and E.M.F. are supported by the Ministero Università e Ricerca, project E-ARGO” (CUP J95F21002190001), with title “Fondo finalizzato al rilancio degli Investimenti delle Amministrazioni Centrali dello Stato e allo sviluppo del Paese, ex Art.1, comma 14, legge n. 160/2019".

%% file: sections/related.tex
\section{Related work}
\label{app:related}

This appendix collects the literature that Section~\ref{sec:intro} refers to, organized by the role it plays in this paper.

\paragraph{Differentiability of persistent homology.} Persistent homology \citep{edelsbrunner2010topological} entered learning pipelines first as a fixed featurization and then as a differentiable layer: \citet{brueelgabrielsson2020} implement persistence of point clouds and images as a layer with a backward pass, and \citet{hofer2020filtration} learn a vertex filtration function end to end for graph classification, in both cases with a monotone filtration of a fixed complex. Earlier, \citet{gameiro2016continuation} used the implicit function theorem to continue point clouds toward prescribed persistence diagrams,  and \citet{poulenard2018topological} optimized real-valued functions on shapes through the derivatives of persistence diagrams with respect to the function values. 
\citet{solomon2021fast} and \citet{nigmetov2024big} propose backpropagation schemes that spread the gradient of a persistence loss over more simplices than the critical pair. 
\citet{horn2022togl} learn graph filtrations end to end inside a graph neural network, for static graphs. The analytical basis for most of this practices is the differential calculus on barcodes of \citet{leygonie2022}, which identifies when a diagram-valued map admits a $C^r$ local lift through the quotient by the permutations of the bars, defines the differential relative to that lift, and gives the chain rule for composing a map into the space of diagrams with a map out of it; we recall the parts we use in Section~\ref{sec:diff}. \citet{carriere2021} pair that calculus with o-minimal geometry: for definable functions of a persistence diagram, stochastic subgradient descent converges to a critical point under the conditions of \citet{davis2020stochastic}, which we restate as Proposition~\ref{prop:clark} and whose hypotheses Theorem~\ref{thm:lipschitz_breve} does not fully supply. \citet{scoccola2024} extend differentiability and optimization to multiparameter persistence, decomposing the parameter space into regions on which the combinatorial structure of the invariant is constant. Our cell decomposition of $FF_N(K)$ plays the same role, with the difference that time is here a non-monotone evolution parameter rather than a second filtration axis.

\paragraph{Zigzag persistence and its computation.} Zigzag persistence was introduced by \citet{carlsson2010zigzag} and connected to the levelset topology of real-valued functions by \citet{carlsson2009realvalued}. It generalizes persistent homology to sequences of complexes connected by inclusions in either direction, so that a feature can be followed across a non-monotone sequence, at the cost of the decomposition being indexed by positions in the sequence. Computation is practical: \citet{dey2022fastzigzag} convert a zigzag filtration into an ordinary one and compute the barcode with standard persistence software, and \citet{deyhou2024vineyard} maintain a zigzag barcode along a one-parameter family of filtrations, including expansions and contractions. These algorithms fix the pairing; our construction reads gradients through the endpoints they return without modifying it.

\paragraph{Topological summaries of time-varying data.}
One family of methods computes standard persistence independently at each time step and follows the resulting summaries over time: CROCKER plots display Betti numbers as a function of time and scale for simulated biological aggregations \citep{topaz2015topological}, crocker stacks add a smoothing parameter to them \citep{xian2022capturing}, and \citet{rieck2020uncovering} compute sublevel-set cubical persistence of each fMRI volume of participants watching a movie, using summary statistics and persistence images of the per-time-point diagrams to predict age and to embed cohort brain-state trajectories. Such per-frame summaries do not match features across time. Vineyards \citep{cohensteiner2006vines} match them for a continuously varying filtration of a fixed complex, and zigzag persistence does so for sequences in which simplices are both inserted and deleted. Zigzag persistence is applied to fish swarms, where barcodes of time-varying point clouds are vectorized as persistence landscapes \citep{corcoran2017modelling};  to dynamic graphs and dynamic metric spaces, whose zigzag barcodes are stable under perturbation of the input \citep{kim2020analysis}; to dynamical systems, where it detects Hopf bifurcations from a sliding-window embedding \citep{tymochko2020hopf};  to temporal networks, where it summarizes the evolution of connectivity \citep{	myers2023temporal};  to time-series forecasting, where zigzag summaries are used as features inside a graph convolutional architecture \citep{chen2021zgcnets};  and to neural population activity recorded under time-varying stimuli \citep{gardinazzi2026neural}. \citet{deysamaga2025quasi} propose a variant designed to reduce the sensitivity of the barcode to the choice of the sequence. For temporal-graph classification, \citet{uddin2026t3former} use sliding-window topological and spectral descriptors as tokens of a transformer. An alternative to the zigzag route is to keep time as a genuine second parameter: \citet{kim2021spatiotemporal} treat time-varying point clouds as dynamic metric spaces and study a spatiotemporal multiparameter module, with stability in the interleaving distance, and \citet{chen2022tamp} summarize time-conditioned multipersistence as an Euler--Poincar\'e surface that is fed to a graph convolutional forecaster. In all of this work the filtration and the sequence are designed by hand, and where a topological summary enters a trained model \citep{chen2021zgcnets,chen2022tamp,uddin2026t3former}, it is computed from a fixed filtration and passed to the network as input; what is missing, and what this paper supplies, is a derivative of the barcode with respect to the values that generate the sequence.

\paragraph{Stability.} In one-parameter persistence, the bottleneck distance between diagrams is bounded by the supremum distance between the functions generating them \citep{cohensteiner2007stability}, so a barcode-valued map is $1$-Lipschitz and the failure of differentiability is confined to the strata where bars are exchanged. For zigzag modules the available results are algebraic: \citet{botnan2018} prove that block-decomposable zigzag modules satisfy an isometry theorem for the interleaving distance, and \citet{bjerkevik2021} give bounds for interval-decomposable modules. These bound distances between modules that are already given, and do not bound the change of the barcode under a perturbation of the filtering values, which in our setting can create or destroy a bar of positive length. Section~\ref{sec:stability} and Appendix~\ref{sec:lip_theorem} describe the resulting exclusion set and the discontinuities inside it.

\paragraph{Gradients through event times.} The device we use --- differentiating the instant at which an interpolated value crosses a threshold, rather than the value at a fixed instant --- is the one used for exact gradients in spiking neural networks, where the loss depends on the parameters through spike times defined by a threshold crossing of the membrane potential and the derivative of the time is obtained by implicit differentiation \citep{wunderlich2021eventprop}. \citet{bohte2002spikeprop} backpropagate errors through spike times in temporally encoded spiking networks, and \citet{chen2021event} differentiate through the termination times of neural ODEs defined implicitly by learned event functions.  Here the analogy is in the mechanism: there the event times enter a differential equation, here they are the endpoints of bars in a barcode.

\paragraph{Settings of the experiments.} The coverage experiment of Section~\ref{sec:exp_coverage} takes its criterion from the homological coverage literature, where a cycle in the nerve of the sensing network certifies a hole in the covered region without coordinates \citep{desilva2006}; repair strategies for such holes have been studied with spectral objectives \citep{yadokoro2023}. We propose as a novelty the object minimized: the lifetime of a defect across time, which is a function of the zigzag pairing. The classification experiment of Section~\ref{sec:exp_mooc} takes place on temporal graphs, for which a developed family of learned architectures exists, based on recurrent embedding trajectories \citep{kumar2019jodie,trivedi2019dyrep}, on temporal attention \citep{xu2020tgat}, and on memory modules \citep{rossi2020tgn}, with a unified benchmark \citep{yu2023dyglib}. We do not compare against these architectures. The question asked there is whether the classification gradient can train the filtration, not whether a topological pipeline is preferable to a temporal one.

%% file: sections/app_exp_coverage.tex
\section{Coverage experiment: details}
\label{app:exp_coverage}
\label{app:coverage-lifetimes}

In this section, we give more details on the Experiment presented in section \ref{sec:exp_coverage}. 

\subsection{Related work}
\label{app:coverage-related-work}

\paragraph{The coverage criterion.}
The homological coverage criterion used here is inspired by
\citep{desilva2007, desilva2006}, which introduced the idea that a homological
condition can certify coverage of a domain by a sensor network. The setup is coordinate-free, with unknown positions
and they compute the relative homology of a Rips pair against a fence with a controlled boundary. Our construction is different: we devise a local per-cell test on one
fixed $1$-cycle with known station positions, 
and we test the coverage gap with an external evaluator (explained in the next section).

\paragraph{Zigzag persistence for dynamic coverage.}
Zigzag persistence has already been applied to coverage in time-varying networks.
\citet{gamble2015} track coverage holes across a sequence of complexes built from
the communication graph, obtaining a barcode that records when each hole opens and
closes, together with representative cycles and per-hole size estimates. That is
the same object this appendix optimises over. The distinction is one of analysis
against synthesis: they \emph{measure} time-varying coverage with zigzag barcodes,
whereas the present work optimizes the trajectoris by \emph{differentiating through} the barcode against a functional of bar lifetimes. Their work establishes
zigzag-on-coverage as a formulation; the contribution here is the derivative.

\paragraph{Gradient descent on a spectral topological objective.}
The nearest prior method by \emph{mechanism} rather than by object is the weighted
combinatorial Laplacian of \citet{yadokoro2023}, who repair coverage in a mobile
sensor team by gradient ascent on the smallest non-zero eigenvalue of a weighted
Laplacian $\widetilde{L}_1$, with edge weights inverse to Euclidean length. The
mechanism is the one used here, i.e. gradient descent on a topological objective,
and the difference lies in the objective and in the complex it is computed on.
Their complex is a Delaunay triangulation, which is planar and contractible, so
$H_1$ vanishes and, by the discrete Hodge theorem, $\ker \widetilde{L}_1$ is
trivial: the quantity their eigenvalue responds to is not a homology class but an
\emph{almost}-hole, a subcomplex of low weights spanned by long edges. Their
functional is moreover a property of a single configuration, evaluated instant by
instant with no memory across frames, so like the count strategy we propose it cannot express how long a given defect has
persisted. The two lines therefore
address different problems: theirs is the static redistribution of an all-mobile
team, while the problem posed here is scheduling under scarcity, in which the
defects are permanent under the fixed infrastructure and only their lifetimes can
be controlled.

\subsection{Experimental setup}

\paragraph{Arena.}
Sixteen static sensors occupy the unit lattice $\{0,1,2,3\}^2$ and the monitored domain is $[0,3]^2$. The Rips radius is $R=1.05$. Since $R>1$, the lattice sides are edges of the complex, while the diagonals, of length $\sqrt2$, are not. Since $R<\sqrt5/2=1.118$, the distance from the midpoint of a shared side to the far corners of the two adjacent cells, one robot can fill at most one cell at a time; the nine unit cells are therefore $1$-cycles that do not bound, and each is a class of $H_1$. $M$ mobile robots are added to the sensors points, moving for $30$ frames, so that the sensors-robots point cloud evolves across frames only through robots movements. The physical sensing radius is $R_S=R/\sqrt3=0.606$ and is used by the evaluator only: for a fixed cell-boundary cycle the planar inclusion $\mathrm{Rips}(R)\subseteq\mathrm{\check Cech}(R/\sqrt3)$ says that a cycle bounding at radius $R$ is covered by sensing discs of radius $R_S$, which is a pointwise statement about one cycle at one instant and gives no bar-by-bar correspondence and no ordering of lifetimes. The barcode is the objective throughout and every coverage number is the evaluator's.

\begin{figure}[h]
\centering
\includegraphics[width=0.86\textwidth]{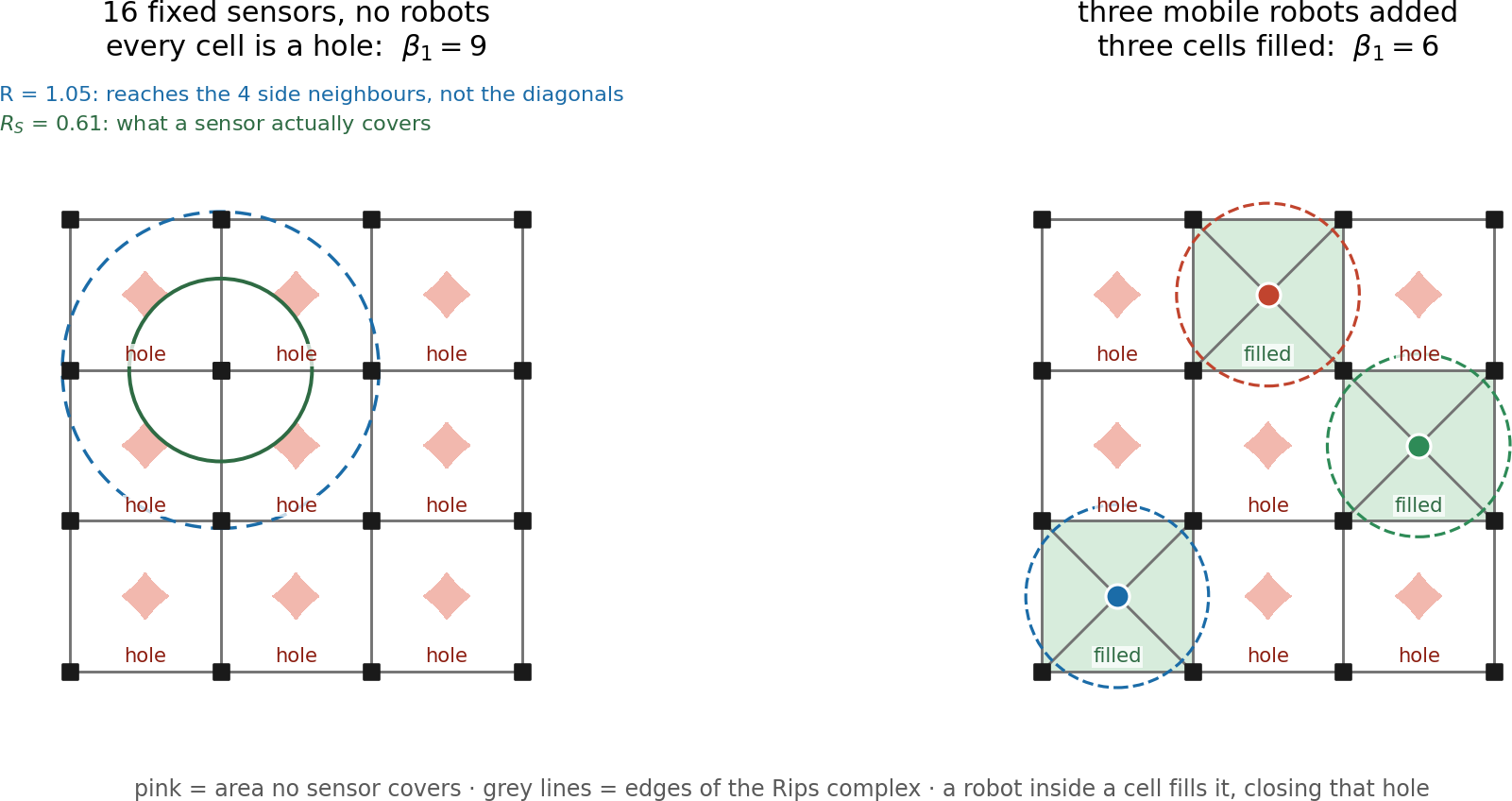}
\vspace{-0.4em}
\caption{The arena. Left: with no robots each of the nine cells is a hole, so the $1$-Betti number $\beta_1=9$; the two circles on one sensor are the Rips radius $R=1.05$, which reaches the four side neighbours and not the diagonals, and the sensing radius $R_S=0.61$; the pink diamonds are the area no sensor covers, the cell centres lying $0.707$ from their corners. Right: three robots fill three cells, $\beta_1=6$. Three robots can never close more than a third of the arena at once, which is why the total hole-time is close to a fixed budget and its distribution in time is the quantity at issue.}
\label{fig:cov_arena}
\end{figure}

\paragraph{How the robots move.}
The parameters are the robots' positions at each frame:
$$\theta = y = \bigl(y(j)_t\bigr)_{j = 1,\dots,M;\ t = 1,\dots,29} \in \mathbb R^{58M},$$
where $y(j)_t$ is the planar position of the $j$-th robot at frame $t$, and they start moving from the common initial point $y(j)_0=(1.5,1.5)$. 
Between consecutive frames a robot travels in a straight line at constant velocity equal to $1$, imposed at the loss-level.
To get a finer view of the robots' motion, each robots' path on frames $\{0,...,29\}$ is extended to a set of $88$ time stamps $\{0,\tfrac13,\tfrac23,\dots,29\}$ via piecewise-linear interpolation.
Formally, this can be seen as a parametrization map $G : \mathbb R^{58M} \to \mathbb R^{cbN} = \mathbb R^{2 \times (16+M) \times 88}$ where, if the point $\nu$ is the $j$-th robot:
\begin{align*}
  x(\nu)_{t+\frac m3}   &= \Bigl(1 - \tfrac{m}{3}\Bigr)\, y(j)_t + \tfrac{m}{3}\, y(j)_{t+1},
                 && t = 0,\dots,28,\,\, m = 0,1,2, \\
  x(\nu)_{29}  &= y(j)_{29},
\end{align*} 
while if $\nu$ is a sensor, with $s_\nu$ being its fixed position, $x(\nu)_i= s_\nu$ for every $i\in\{0,\tfrac13,\tfrac23,\dots,29\}$.
Thus, $x=G(y)$ is a vector representing the sensors-robots point cloud evolution along the $88$ time stamps $\{0,\tfrac13,\tfrac23,\dots,29\}$.

\paragraph{Objective.}
After point cloud motion formalization, we can apply the paper's pipeline with the indexing convention changed from $\{0,1,...,N,N+1\}$ to $\{-\tfrac13,0,\tfrac13,\dots,29,\tfrac{88}{3}\}$, as explained in Remark \ref{rmk:indexes}. Under this convention, the pipeline reads:
\[
  P_1 : \mathbb R^{58M}
  \xrightarrow{\ G\ } \mathbb R^{2 \times (16+M) \times 88}
  \xrightarrow{\ \mathrm{VR}\ } FF_{88}(K)
  \xrightarrow{S\ } \mathrm{SR}(K;<)
  \xrightarrow{\ \mathrm{PH}_1\ } \mathcal D,
\]
where $\mathrm{VR}$ is the Vietoris-Rips extension described in Section \ref{sec:applications}, with radius $R$.
Calling $D=P_1(y)$ the persistence diagram coming from this pipeline, the loss function to optimize is
\begin{equation*}
\mathcal L(y)= P_{\rm dom}(y)+ P_{\rm speed}(y)+\mathcal L_{\rm topo}(y),
\end{equation*}
with
\begin{equation*}
\mathcal L_{\rm topo}=\textstyle\sum_{(b,d)\in D}(d-b)
\ \ \text{or}\ \
\textstyle\sum_{(b,d)\in D}\max\!\big(0,d-b-\Delta\big)^2,
\end{equation*}
the first being the count objective, and the second the lifetime objective with $\Delta$ frames tolerance.
The two penalty terms are
\begin{align*}
P_{\rm dom}(y)&=40\sum_{i}\sum_{\nu \text{ robot}}\sum_{c\in\{1,2\}}
\Big[\max\big(0,-x(\nu)_{i,c}\big)^2+\max\big(0,x(\nu)_{i,c}-3\big)^2\Big],\\[-0.2em]
P_{\rm speed}(y)&=40\sum_{t}\sum_{j=1}^{M}\Big(\big\|y(j)_{t}-y(j)_{t-1}\big\|-1\Big)^2,
\end{align*}
the squared excursion of the interpolated path outside the monitored domain, summed over the $88$ time stamps and both coordinates, and the squared deviation of every move from unit length. The topological term enters at weight one, and the weight $40$ is the largest that does not diverge at the step size used ($400$ and $4000$ do).
Since $G$ is affine, $\mathrm{VR}$ is semialgebraic, and $\mathcal L$ can be seen as $\mathcal L=L\circ P_1$ with $L$ $1$-differentiable (Definition \ref{def:differentiability}), Proposition \ref{prop:ae-differentiable} applies and can be used to optimize the loss as described in Section \ref{sec:stability}.

\paragraph{Random walk as initial iterate.}
The optimization starts from a seeded random walk: from the centre, each robot makes $29$ unit moves in uniformly random directions, a move leaving $[0,3]^2$ being redrawn, so both penalties vanish at the first iterate. Within a seed, all strategies start from the same walk.

\subsection{Experiment evaluation and controls}

\paragraph{Evaluation.}
As an independent check of coverage, we compute intervals where the network has coverage gaps exactly in time at each of
$14\,641$ probes on a uniform grid of step $0.025$. A probe within $R_S$ of a sensor is never
blind; otherwise, for each robot and each motion segment $x(s)=a+s(b-a)$ the covered set is the
solution of a quadratic, and the intervals are merged across robots and segments and complemented
in the window. We computed the maximum over probes of the longest coverage gap as the primary
metric. Every barcode quantity reported, such as the longest $H_1$ bar and the topological loss functions (both from count and lifetime stratiegies),
is read from $D=P_1(y)$, the barcode the loss itself computes on the $88$ time stamps, at the final
iterate. 

Empirically the barcode and the evaluator agree closely in this arena: the longest bar is at least the maximum coverage gap in every run, exceeding it by a
median of $0.38$ frames (range $0.04$--$1.87$) on the lifetime objective and $0.21$ frames
(range $0.00$--$1.14$) on the count objective. That ordering is the one the pointwise filling statement
predicts: a cell whose boundary cycle bounds is covered by the sensing discs, so a probe may be
covered while its cell is still topologically open, so the barcodes report more failure than
the physical measurement does.

\paragraph{Hyperparameters and configuration search.}
The parametrization, $\gamma_0$, the noise level, the decay exponent, the budget and the penalty
weight were fixed jointly on development seeds $0$--$3$ over $30$ configurations, and every number
in Section~\ref{sec:exp_coverage} and below is from held-out seeds $10$--$19$, ten seeds per objective,
with the seed as the independent unit. We report medians across seeds, and $95\%$ bootstrap
intervals over the ten ($20\,000$ resamples) for paired within-seed differences, the pairing being
across objectives at a common seed. The \emph{final} iterate is the reported quantity for every objective.

Settings were fixed on the development seeds and on the lifetime objective only, scored by the median of
the primary metric, over $30$ configurations: $\gamma_0$ across
$\{3\cdot10^{-4},\dots,3\cdot10^{-2}\}$ and $\sigma$ across $\{0.5,2,5,30,60,100\}$, with both
budgets ($3\,000$ and $30\,000$) and both speed variants (fixed speed and bounded speed) also covered. The
decay exponent was held at $0.51$, the smallest value admitted by the step-size condition and hence
the one that keeps the step alive longest; faster decays were uniformly worse in preliminary
exploration and were not carried into this search. Four findings fixed the reported setting.

\begin{itemize}
\item \emph{The step size is bounded by the penalty, not by the topology.} Steps of $3\cdot10^{-2}$
and above diverge: the penalty weight $40$ gives a curvature of $80$ and hence a stability limit
near $0.025$. At the other end, steps of $3\cdot10^{-3}$ and below leave the longest coverage gap
above $12$ frames at any noise level or budget. 
\item \emph{The perturbation helps up to a point.} At $\gamma_0=10^{-2}$ and $30\,000$ steps the
development median improves monotonically with $\sigma$, giving $11.6$, $11.5$, $9.6$, $8.8$ frames at
$\sigma=0.5,2,5,30$. At t $\sigma=60$ and $100$ the trend turns back down ($9.3$ and $11.4$).
\item \emph{The speed bound was rejected.} Replacing the speed equality of $P_{\rm speed}$
by an inequality lets the robots slow down, travelling $68$ units instead
of $87$ with no gain in the loss, which destroys the equal-effort comparison between the objectives.
\item \emph{The budget was extended after selection.} $30\,000$ steps were selected on development
seeds; the budget was then extended to $100\,000$ on the held-out seeds, worth
$-0.20$ $[-0.71,+0.41]$ frames, lower in $8$ of $10$ seeds.
\end{itemize}

\subsection{Extended results}

All three tables \ref{tab:cov_delta}, \ref{tab:cov_robots} and \ref{tab:cov_noise} report held-out seeds $10$--$19$, final iterates after $100\,000$ steps, medians
over the ten seeds with the range across seeds in parentheses. The longest bar is that of $D$, in frames. The last column is the fraction of the monitored domain whose longest coverage gap
exceeds the tolerance \emph{of that row}, as measured by the independent evaluator.

Both count and lifetime strategies effectively reduces their topological loss compared to the random walk initialization (Table \ref{tab:cov_delta}).
In particular, for the lifetime strategy, the achieved longest bar rises with $\Delta$, from $9.4$ to $13.8$ frames as $\Delta$ goes from $6$ to $12$, staying between $1.8$ and $3.4$ frames
above its own target throughout. Thus, the lifetime topological objective is never driven to zero.

\begin{table}[h]
\centering\footnotesize\setlength{\tabcolsep}{6pt}
\caption{Topological results of the three-robots-experiment, for count and lifetime strategies, the last with four different tolerance $\Delta$ values.
The count objective carries no $\Delta$, so its waiting area is computed for $\Delta=6$.}
\label{tab:cov_delta}
\begin{tabular}{lccc}
\toprule
strategy & $\mathcal L_{\text{topo}}$: init $\to$ final & longest bar & area $>\Delta$ \\
\midrule
count (at $\Delta=6$)   & $244.8\to211.0$                           & $21.6$ \;($11.3$--$29.0$) & $2.7\%$ \;($2.3$--$3.2$) \\
\midrule
lifetime, $\Delta=6$    & $507.9\to33.8$                        & $\phantom{0}9.4$ \;($\phantom{0}8.0$--$10.9$) & $1.9\%$ \;($0.7$--$2.9$) \\
lifetime, $\Delta=8$    & $310.7\to17.2$                        & $11.1$ \;($\phantom{0}9.0$--$13.9$) & $0.7\%$ \;($0.4$--$1.9$) \\
lifetime, $\Delta=10$   & $180.1\to\phantom{0}6.7$              & $11.8$ \;($10.3$--$16.2$) & $0.5\%$ \;($0.1$--$1.2$) \\
lifetime, $\Delta=12$   & $\phantom{0}96.9\to\phantom{0}5.4$    & $13.8$ \;($12.3$--$21.0$) & $0.1\%$ \;($0.0$--$0.9$) \\
\bottomrule
\end{tabular}
\end{table}

When increasing the number of robots (shown in Table \ref{tab:cov_robots}) the result is closer to the topological objective as expected, since the task is easier to solve with more robots covering the holes.
\begin{table}[h]
\centering\footnotesize\setlength{\tabcolsep}{6pt}
\caption{Robot-count ladder, lifetime objective at $\Delta=6$ throughout.}
\label{tab:cov_robots}
\begin{tabular}{llccc}
\toprule
robots & strategy & $\mathcal L_{\text{topo}}$: init $\to$ final & longest bar & area $>\Delta$ \\
\midrule
$2$ & count    & $250.8\to227.9$                       & $25.4$ \;($19.9$--$29.0$) & $3.4\%$ \;($3.1$--$3.6$) \\
$2$ & lifetime & $996.2 \to 144.6$                       & $12.4$ \;($10.6$--$15.9$) & $3.3\%$ \;($2.7$--$3.7$) \\
\midrule
$3$ & count    & $244.8\to211.0$                       & $21.6$ \;($11.3$--$29.0$) & $2.7\%$ \;($2.3$--$3.2$) \\
$3$ & lifetime & $507.9\to 33.8$             & $\phantom{0}9.4$ \;($\phantom{0}8.0$--$10.9$) & $1.9\%$ \;($0.7$--$2.9$) \\
\midrule
$4$ & count    & $239.2\to197.0$                       & $17.1$ \;($\phantom{0}9.6$--$24.6$) & $1.8\%$ \;($1.3$--$2.4$) \\
$4$ & lifetime & $250.4 \to 14.8$             & $\phantom{0}8.9$ \;($\phantom{0}6.7$--$10.9$) & $0.5\%$ \;($0.0$--$1.3$) \\
\midrule
$5$ & count    & $232.4\to180.7$                       & $13.9$ \;($10.6$--$19.3$) & $1.5\%$ \;($0.9$--$2.4$) \\
$5$ & lifetime & $121.8 \to 5.0$             & $\phantom{0}7.9$ \;($\phantom{0}6.9$--$\phantom{0}9.6$) & $0.2\%$ \;($0.0$--$0.7$) \\
\bottomrule
\end{tabular}
\end{table}

We also test the role of the noise in the zigzag gradient (see Table \ref{tab:cov_noise}). Removing the perturbation moves the median longest bar only from $9.4$ to $10.6$ frames --- the paired difference is $+1.69$ $[-1.00,+7.36]$ frames, an interval containing zero --- but it
widens the spread across seeds sharply, one seed ending at the worst attainable $29$. Removing the topological term instead, so that $\mathcal L=P_{\rm dom}+P_{\rm speed}$ and both strategies reduce to the same motion-only run, yields no topological descent: the count term goes from $244.8$ to $246.7$, the lifetime term from $507.9$ to $530.9$, with the longest bar ending at $21.3$ frames. Against this reference the lifetime strategy shortens the longest bar in every seed (median paired difference
$-12.2$ frames), while the count strategy barely moves it ($-1.3$, lower in $7$ of $10$ seeds) and
leaves more of the domain waiting longer than $\Delta$ ($2.7\%$ against $2.2\%$). 
\begin{table}[H]
\centering\footnotesize\setlength{\tabcolsep}{6pt}
\caption{Three experiments with three robots and $\Delta=6$: stochastic optimization, optimization without noise (gradient only), stochastic optimization without topological term (motion only). In the motion only rows the topological term is removed from the loss, so both strategies give the same run: each row reads its own term along that path, evaluated but not optimized.}
\label{tab:cov_noise}
\begin{tabular}{llccc}
\toprule
strategy & condition & $\mathcal L_{\text{topo}}$: init $\to$ final & longest bar & area $>\Delta$ \\
\midrule
lifetime & gradient and noise (reported) & $507.9\to33.8$ & $\phantom{0}9.4$ \;($\phantom{0}8.0$--$10.9$) & $1.9\%$ \;($0.7$--$2.9$) \\
lifetime & gradient only, $\sigma=0$     & $507.9\to48.7$ & $10.6$ \;($\phantom{0}7.3$--$29.0$) & $2.1\%$ \;($0.7$--$3.4$) \\
lifetime & motion only, $\mathcal L_{\rm topo}$ removed & $507.9\to530.9$ & $21.3$ \;($18.9$--$29.0$) & $2.2\%$ \;($2.0$--$2.8$) \\
\midrule
count    & motion only, $\mathcal L_{\rm topo}$ removed & $244.8\to246.7$ & $21.3$ \;($18.9$--$29.0$) & $2.2\%$ \;($2.0$--$2.8$) \\
\bottomrule
\end{tabular}
\end{table}

\subsection{The loss along the iterate sequence}
\label{app:coverage-loss}

Figure~\ref{fig:cov_loss_count} shows the count objective's term (left) and its total loss $\mathcal L$
(right) along the iterate sequence. The term moves by $14\%$, from $244.8$ to $211.0$, with the ten
seeds with a low spread ending between $208.6$ and $212.8$; the total loss ends at $212.1$, the
penalties having decayed from their peak at the first iterates. That is a small movement, mainly caused by how the arena is built. With three robots at least six of the nine cells are open at every
instant, so $\mathcal L_{\rm topo}=\sum_{(b,d)\in D}(d-b)$ has a hard floor and little of it is reachable: the count objective starts
close to the best its own criterion admits (the random walk motion), and the $14\%$ it gains is the lowest total defect-time
reached anywhere in this experiment.

\begin{figure}[h]
\centering
\includegraphics[width=0.49\textwidth]{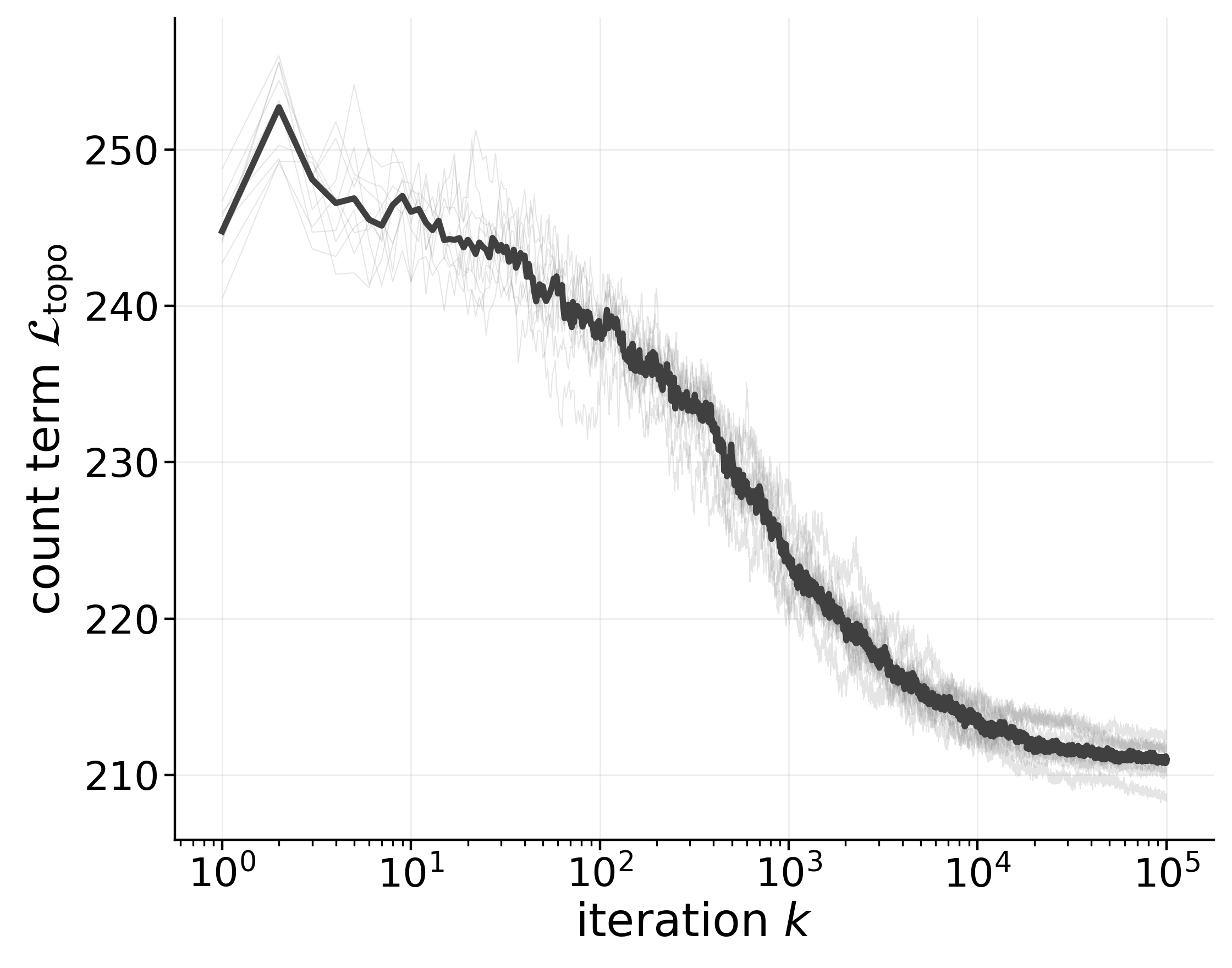}\hfill
\includegraphics[width=0.49\textwidth]{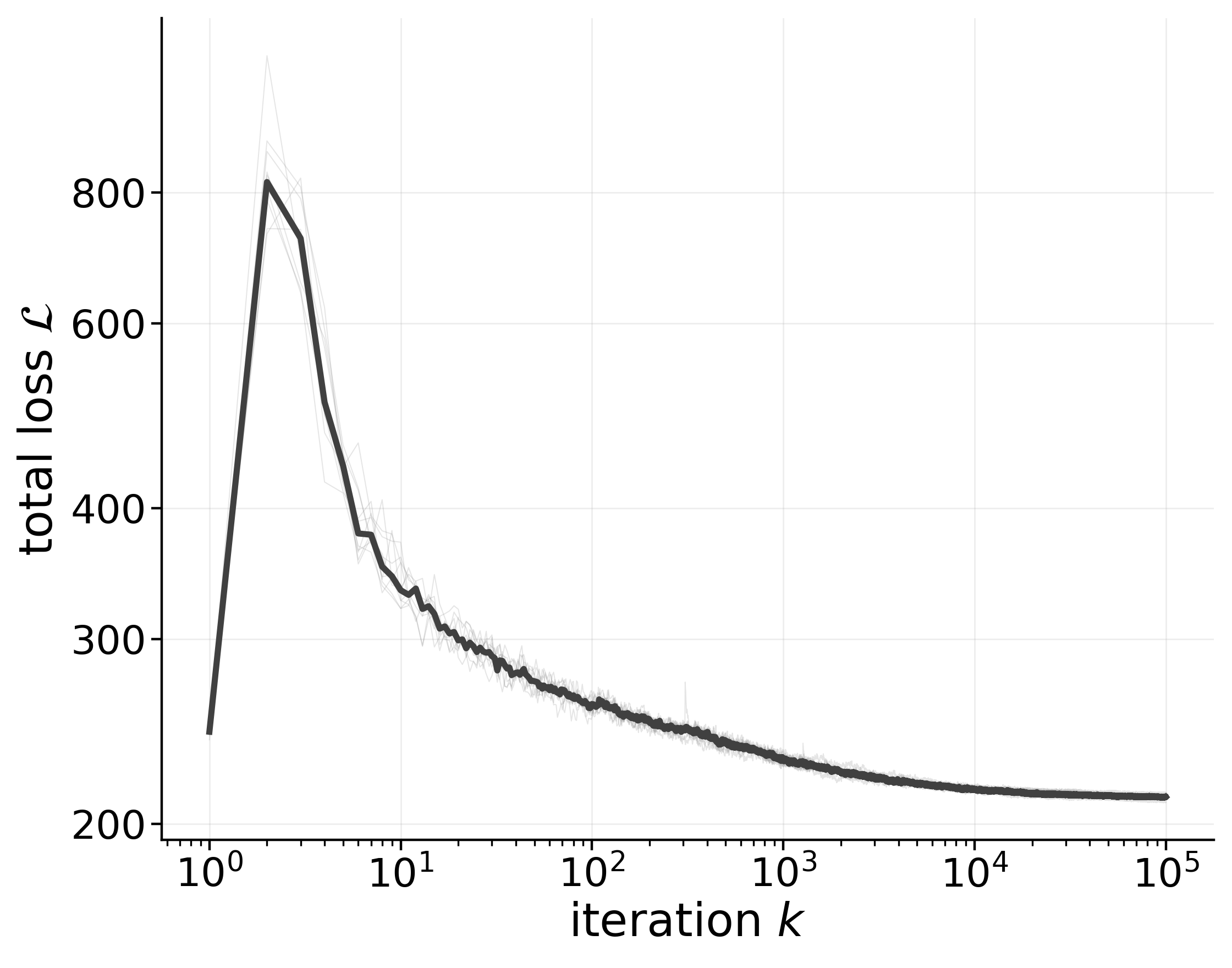}
\vspace{-0.4em}
\caption{The count strategy along the iterate sequence, held-out seeds $10$--$19$, median bold and seeds faint, logarithmic iteration axis. Left: the count term $\mathcal L_{\rm topo}=\sum_{(b,d)\in D}(d-b)$ on a linear value axis; it descends by $14\%$ and no more, because with three robots at least six of the nine cells are open at every instant and the loss has a hard floor. Right: the total loss $\mathcal L$, logarithmic value axis; the peak at $k=2$ is the first perturbation breaking the unit-speed penalty (which was satisfied by the random walk initialization), and the curve ends $1.1$ above $\mathcal L_{\rm topo}$. }
\label{fig:cov_loss_count}
\end{figure}

The lifetime objective has far more room, because it measures how the budget of defect-time is distributed in time. Over the same
$100\,000$ steps its term falls by a factor of $15$ at $\Delta=6$ ($507.9\to33.8$), $18$ at $\Delta=8$
($310.7\to17.2$), $27$ at $\Delta=10$ ($180.1\to6.7$) and $18$ at $\Delta=12$ ($96.9\to5.4$), as the left
panel of Figure~\ref{fig:cov_loss_vs_bar} shows.
Figure~\ref{fig:cov_loss_total} shows the total objective $\mathcal L$ for the same four conditions:
$507.9\to34.6$, $310.7\to17.8$, $180.1\to7.3$, $96.9\to6.1$. It differs from $\mathcal L_{\rm topo}$ only by
the two penalties, which are large only at the first iterates, where the first perturbation of
$0.3$ lattice units breaks the unit-speed condition, and amount to less than one unit at the final
iterate (median $0.6$--$0.7$).
Furthermore, the right panel of Figure~\ref{fig:cov_loss_vs_bar} shows the median length of the longest $H_1$ bar across iterations. Over the optimization runs, the longest bar starts at a median of $21.1$ frames (from the seeded random walk) and decreases to $9.4$, $11.1$, $11.8$, and $13.8$ across the four tolerances $\Delta$. Nevertheless, in all cases, the final defect length remains slightly above the target tolerance.

\subsection{Computational cost.} 

Every run uses a single CPU core of an AMD EPYC 9374F with no GPU,
and peaks at $350$\,MB of memory. One run of the reported configuration --- $\Delta=6$, three robots,
topology read at three sub-steps per move, $100\,000$ subgradient steps --- takes a median of
$1.5$ hours over the ten held-out seeds (range $1.35$--$1.63$), i.e.\ $57$\,ms per step, of which
$10$\,ms is the zigzag layer's forward pass (crossing times and pairing), $7$\,ms the evaluation of the
filtering values at the $88$ instants and $17$\,ms the backward pass; the remaining $23$\,ms rebuild the
ambient complex at each iterate, which is an artifact of our implementation rather than of the method,
since only $46$ distinct complexes occur across $300$ consecutive steps and a cache would remove most of
it. Each step handles $19$ points, $88$ instants, a complex of about $180$ simplices and about $300$
$H_1$ bars. The independent evaluator costs $1.6$\,s per call at $14\,641$ probes and is called twice per
run. The campaign behind Table~\ref{tab:cov_delta} totals $75$ core-hours.

\begin{figure}[h]
\centering
\includegraphics[width=0.995\textwidth]{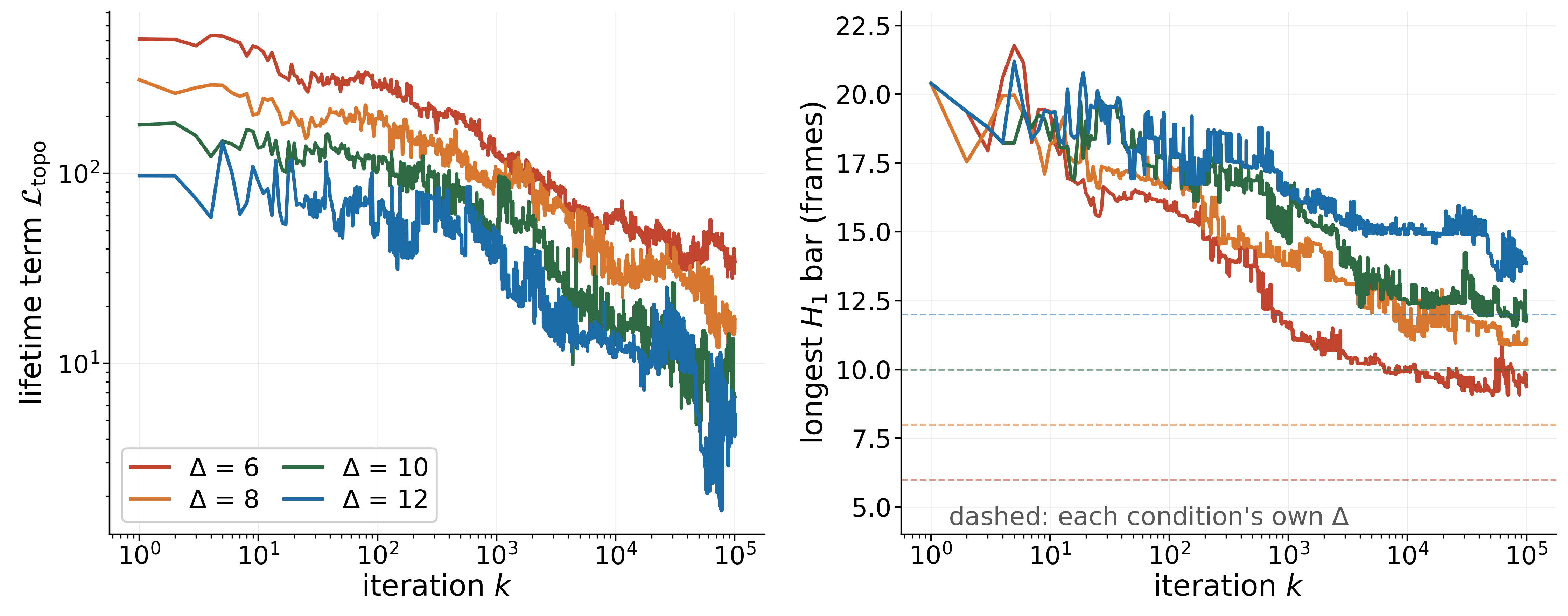}
\vspace{-0.4em}
\caption{Left: the topological term at the iterate for the lifetime objective at each tolerance,
medians over held-out seeds $10$--$19$, logarithmic axes. Right: the longest $H_1$ bar at the same iterates, each condition against its own $\Delta$ (dashed, matching colours). The objective goes on descending for the whole run at every tolerance.}
\label{fig:cov_loss_vs_bar}
\end{figure}

\begin{figure}[h]
\centering
\includegraphics[width=0.56\textwidth]{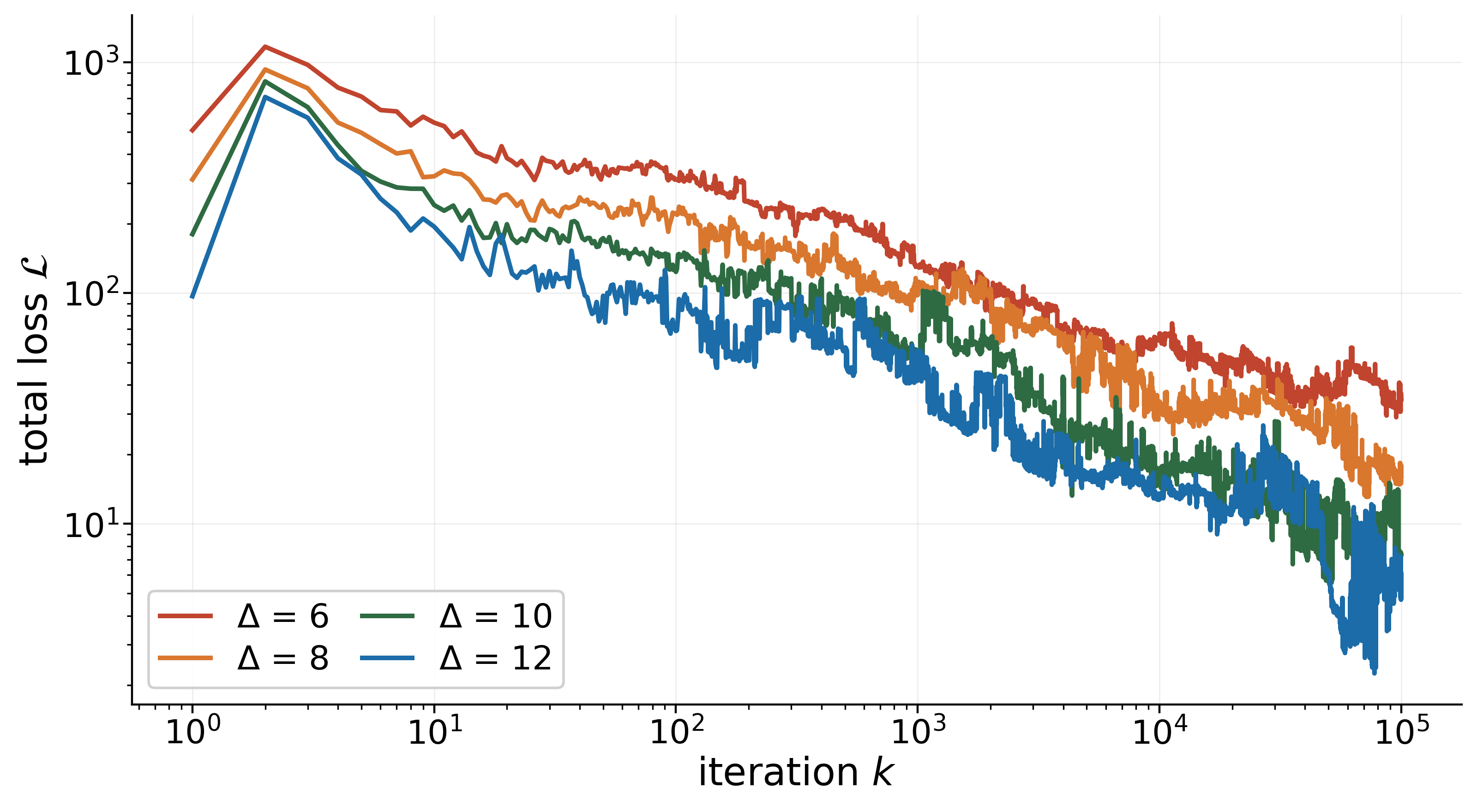}
\vspace{-0.4em}
\caption{The total objective $\mathcal L$ of the lifetime strategy at each tolerance, penalties
included, medians over held-out seeds $10$--$19$, logarithmic axes.}
\label{fig:cov_loss_total}
\end{figure}

%% file: sections/app_exp_mooc.tex
\section{Learned upper-star zigzag: details}
\label{app:exp_mooc}

In this section, we give more details on the experiment presented in Section~\ref{sec:exp_mooc}.

\subsection{Experimental Setup}
\label{app:mooc_setup}

\paragraph{Generator.}
A figure-eight is two loops sharing a node, each loop containing one switch vertex and closed only while that vertex is active, so at every frame exactly one of the two loops is closed and which one is decided by the switch pattern. Class~A alternates slowly, every $16$ frames, and class~B rapidly, every $2$. Every example carries a second, decoy figure-eight of identical structure running the other class's pattern with an independent phase, so each example contains one slowly and one rapidly alternating figure-eight whichever its label, and the two classes differ only in which of the two is the signal. That is written in four per-vertex features: the first two have opposite signs on the signal figure-eight's vertices and equal signs on the decoy's, with magnitudes drawn independently of the label, and the last two are noise. Neither feature alone nor any linear function of the two identifies the signal (logistic regression reaches $0.515$ and $0.517$ on each alone and $0.568$ on both, over the $3\,600$ vertices of a development draw), while their product does exactly ($1.000$): the network has to compute an exclusive-or. An example has $32$ frames, $18$ vertices in two figure-eights of nine, $20$ fixed edges and $38$ cells; a draw of the generator produces $300$ balanced examples, split into stratified ten folds with $20\%$ of each training pool held out for validation ($216$ training, $54$ validation and $30$ test examples per split).

\begin{figure}[h]
\centering
\includegraphics[width=0.94\textwidth]{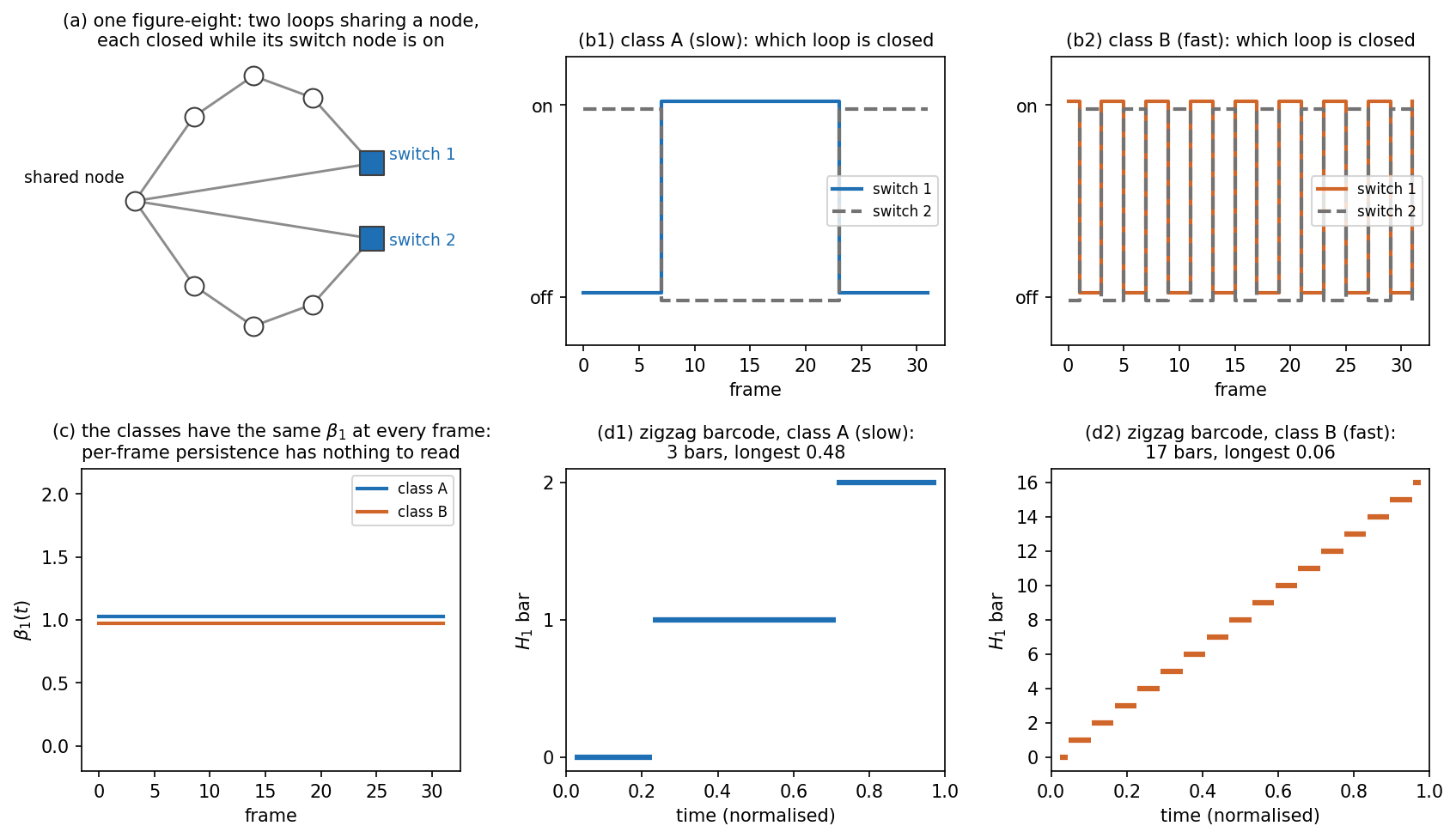}
\vspace{-0.4em}
\caption{The designed task. (a) One figure-eight: two loops sharing a node, each closed while its switch vertex is active. (b) The switch patterns defining the two classes, slow against fast. (c) The resulting Betti curves, identical between the classes at every frame, so a per-frame reading has nothing to separate. (d) The zigzag barcodes of one example of each class under the filtering function the generator knows, at the threshold of the experiment: three bars with the longest at $0.48$ against seventeen with the longest at $0.06$.}
\label{fig:mooc_task}
\end{figure}

\paragraph{Filtering function and read-out.}
Values on vertices are extended by the upper-star rule of Section~\ref{sec:applications}, $v(\sigma)$ is the minimum over the vertices of $\sigma$, producing a filtering function in $FF_{32}(K)$, where $K$ is the connected graph.
Since vertex values take values in $(0,1)$, we take a thresholded $\epsilon=0.5$, and apply the paper's construction pipeline to obtain a barcode from the resulting filtering function. A barcode becomes a $20\times20$ persistence image on (birth, persistence) under two weightings: the persistence itself and the persistence capped at one frame
 for $H_0$ and $H_1$ separately, so four image channels in all, and a convolutional reader of about $26$k parameters maps the image to two logits. Both weightings vanish on the diagonal.

\paragraph{Variants.}
Every variant produces a vertex field or reads one. The \emph{learned zigzag} uses $\sigma(\mathrm{standardise}(\mathrm{base}+r))$, where $\sigma(z)=1/(1+e^{-z})$ is the logistic sigmoid, $\mathrm{standardise}$ rescales a network's values to zero mean and unit standard deviation over all its vertices and frames, and $\mathrm{base}$ is the standardised activity field, $\mathrm{standardise}(a)$ with $a(\nu)_t=+1$ if vertex $\nu$ is active at frame $t$ and $-1$ otherwise; it trains the residual $r$, a two-layer graph convolution of $225$ parameters shared across frames, along with the reader; the \emph{fixed zigzag} uses $\sigma(\mathrm{base})$, the activity-only field, which keeps only each vertex's activity and ignores its four features, so that it cannot tell the signal figure-eight from the decoy, and trains the reader only;
the \emph{signal-only zigzag} uses the selection the generator knows, the signal figure-eight alone, and trains the reader only. The \emph{per-frame} variants apply standard persistence to each frame separately: at every frame, the level is lowered from $1$ to $0$ and a cell enters once all its vertices exceed it, so the complex at level $\epsilon$ is exactly the zigzag's complex at that frame, and the frame yields an ordinary barcode whose endpoints are field values rather than times. The $32$ barcodes become $32$ persistence images, stacked along a frame axis and passed to the same reader. Bars of different frames are never matched, so these variants see which loops exist at each frame but not whether a loop at one frame is the same loop as at the next, and since every level is read, they do not depend on the threshold. All variants share the reader, so differences between them are not about capacity. The residual's last layer is zero-initialized, so the learned variant starts as the fixed variant. The graph convolution's input at each frame is, for every vertex, its activity at that frame and its four features, which are the same at every frame; its output is a single value per vertex and frame, recomputed at every frame and at every iteration. It is applied to each frame independently, with no recurrence, no temporal convolution and no time index in its input. The threshold is fixed, and only the graph convolution and the reader are trained.

\paragraph{Loss and training.}
The loss is the standard cross-entropy loss $\mathcal L(\theta)=\mathrm{CE}(\text{logits},\text{label})+\lambda\|\theta\|^2$ with $\lambda=10^{-4}$. Stage~1 fits the reader alone on the images of the initial frozen field (Adam, weight decay $0.01$, $300$ epochs) and keeps the last iterate, with no checkpoint selection.
Stage~2 is the experiment, the iteration of Definition~\ref{def:clark} with the settings of Table~\ref{tab:mooc_settings}. Its stochasticity is the minibatch sampling, $16$ examples drawn uniformly with replacement, so that $\zeta_k=\hat g_k-\nabla\mathcal L(\theta_k)$.
The graph convolution has about a hundred times fewer parameters than the reader and, with a shared step size, would barely move; we therefore rescale it, writing its weights as $\mathrm{scale}\cdot w$ and training $w$, which leaves the model unchanged but makes its steps $\mathrm{scale}^2$ times larger.

\begin{table}[h]
\centering\footnotesize\setlength{\tabcolsep}{5pt}
\caption{Stage-2 settings and why they take these values. All were fixed on development draws of the generator, none of which the reported results come from.}
\label{tab:mooc_settings}
\begin{tabular}{llp{0.53\textwidth}}
\toprule
setting & value & why \\
\midrule
$\gamma_0$ & $0.02$ & scanned over $\{0.01,0.02,0.03,0.05,0.1\}$; the optimum is interior, larger steps moving the field further and classifying worse \\
$\alpha$ & $0.6$ & in $(1/2,1]$, matching the step-size condition of Proposition~\ref{prop:clark}, which on its own brings none of its other hypotheses; \\
iterations & $3\,200$ & at $800$ the loss is still falling when the schedule has decayed the step to nothing; four times the budget removes that truncation \\
minibatch & $16$, uniform with replacement & the sampling of Definition~\ref{def:clark} \\
ridge $\lambda$ & $10^{-4}$ & no measurable effect in development \\
amplification & $30$ & the rescaling of the graph convolution's weights; without it the graph convolution barely moves beside the reader (Appendix~\ref{app:mooc_setup}) \\
threshold $\epsilon$ & $0.5$ & strictly inside the gap $(0.056,0.587)$ of the initial field $\sigma(\mathrm{base})$, which takes only a low and a high value; \\
restarts & $4$ seeds & the final iterate of the seed with the lowest final \emph{validation} loss is reported \\
\bottomrule
\end{tabular}
\end{table}

\paragraph{Evaluation.} Settings were fixed on four development draws of the generator; every number in Section~\ref{sec:exp_mooc} and below is from two further draws, each divided by ten-fold cross-validation into training, validation and test examples; each of the resulting twenty divisions is a \emph{split}, and every variant is measured on the same twenty. The learned variants are trained from four seeds per split; a seed changes the initialization of the graph convolution's first layer and of the reader, and the order of the minibatches. The final iterate is the reported quantity for every variant, with no best-so-far selection inside a run, as in the coverage experiment; across seeds, the one with the lowest final validation loss is kept, test labels playing no part. A split is \emph{solved} when the test accuracy exceeds $0.9$; we report means and standard deviations over splits and, for single runs, over all $80$ runs. The reader-only variants are trained once per split, their filtering function being fixed.

\subsection{Extended results}
\label{app:mooc_results}

Table~\ref{tab:mooc_arms} reports every variant on the same twenty splits. The controls fail as the design intends: per-frame persistence is at chance at any fixed field and the activity-only filtering function is at chance in either read-out, while the selection the generator knows is solved exactly by the zigzag read-out. The signal-only rows are the only pair that isolates the pairing itself, sharing the filtering function, the complexes and the reader and differing only in whether features are matched across time; that pair separates completely and without variance, but on a task built so that it would. Per-frame persistence partly recovers once its field is also learned (Table~\ref{tab:mooc_arms}), because the network can write temporal structure into per-frame field values for the reader's convolution along the frame axis to pick up; the $1.000$ against $0.500$ comparison is therefore quoted at a shared filtering function. 
\begin{table}[h]
\centering\footnotesize\setlength{\tabcolsep}{4.5pt}
\caption{Temporal-graph classification on the same twenty held-out splits; final-iterate test accuracy, chance $=0.500$; \emph{solved} counts splits above $0.9$. Learned variants: best of four seeds by final validation loss, and all $80$ single runs.}
\label{tab:mooc_arms}
\begin{tabular}{lllcc}
\toprule
read-out & field & runs & solved & accuracy \\
\midrule
zigzag & learned & best of 4 & $\mathbf{20/20}$ & $\mathbf{0.995\pm0.016}$ \\
 & & single & $58/80$ & $0.893\pm0.178$ \\
per-frame persistence & learned & best of 4 & $0/20$ & $0.755\pm0.112$ \\
 & & single & $0/80$ & $0.635\pm0.136$ \\
\midrule
zigzag & signal-only & reader only & $20/20$ & $1.000\pm0.000$ \\
per-frame persistence & signal-only & reader only & $0/20$ & $0.500\pm0.000$ \\
zigzag & activity-only & reader only & $0/20$ & $0.540\pm0.086$ \\
per-frame persistence & activity-only & reader only & $0/20$ & $0.500\pm0.000$\\
\bottomrule
\end{tabular}
\end{table}

The successful runs keep descending steadily (Fig.~\ref{fig:mooc_selected}): averaged over the twenty splits, their validation loss falls from $0.55$ at the first evaluation to $0.078$ at iteration $1\,000$ and $0.024$ at the last, roughly as a power of $k$ and without a plateau, and the training loss follows it closely ($0.023$ at the last iterate), so the kept runs fit without overfitting.

\begin{figure}[h]
\centering
\includegraphics[width=0.7\textwidth]{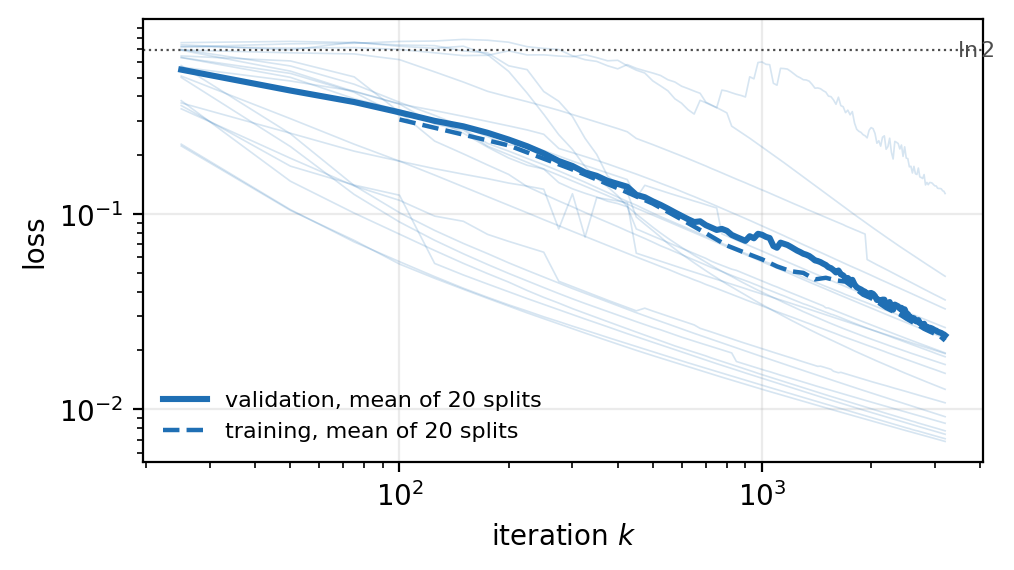}
\vspace{-0.4em}
\caption{Loss of the runs the protocol keeps: on each of the twenty held-out splits, the seed with the lowest final validation loss among four, which solves its split in every case. Faint lines are the individual validation curves; the solid line is their mean, the dashed line the mean training loss. Both axes logarithmic; the dotted line marks $\ln 2$, the loss at chance.}
\label{fig:mooc_selected}
\end{figure}
Within the 20 splits, 22 runs out of 80 fail to converge. These failed runs are an optimization effect: the signal-only filtering function solves every split, and within the admissible schedules $\gamma_k=\gamma_0(1+k)^{-\alpha}$ a faster decay keeps the solutions a run finds but leaves more runs at the uninformative start, so no fixed schedule removes both kinds of failure. To cure for this a step size that adapts along the trajectory would be needed, which falls outside the schedules covered by Proposition~\ref{prop:clark}.

\subsection{Computational cost}
\label{app:mooc_cost}

Every run uses four cores of an AMD EPYC 9374F with no GPU. Per network, a forward and backward pass of the learned zigzag variant takes $3.3$\,ms: $0.3$\,ms for the field, $1.7$\,ms for the zigzag layer (crossing times and pairing), $0.3$\,ms for the reader and $1.0$\,ms for the backward pass, or $0.05$\,s per minibatch step. A complete run --- stage~1, $3\,200$ subgradient steps, evaluation every $25$ steps and the two finite-difference checks --- takes a median of $3.1$ minutes, with a peak memory of $541$\,MB in the timing run. Per-frame persistence is the costly variant, $13.5$\,ms per network ($0.22$\,s per step) for $32$ separate persistence computations and a reader over $32$ images, and $71$ minutes per run. The runs behind Table~\ref{tab:mooc_arms} total $561$ core-hours, of which the per-frame variant with a learned field accounts for $381$.